%% file: paper_update2022.tex
\documentclass[11pt]{article}
\usepackage{graphicx,amsmath,amsfonts,amssymb,bm,hyperref,url,breakurl,epsfig,epsf,color,MnSymbol,mathbbol, fmtcount,algorithmic,algorithm, semtrans
} 
\usepackage{fullpage}

\date{July 15, 2017; Revised July 2022}

\usepackage{titlesec}

\usepackage{tikz}
\usepackage{pgfplots}

\usetikzlibrary{pgfplots.groupplots}

\titleformat{\paragraph}
{\normalfont\normalsize\bfseries}{\theparagraph}{1em}{}
\titlespacing*{\paragraph}
{0pt}{3.25ex plus 1ex minus .2ex}{1.5ex plus .2ex}

\usepackage{caption,subcaption}

\usepackage[bottom,hang,flushmargin]{footmisc} 

\usepackage{hyperref}
\definecolor{darkred}{RGB}{150,0,0}
\definecolor{darkgreen}{RGB}{0,150,0}
\definecolor{darkblue}{RGB}{0,0,200}
\hypersetup{colorlinks=true, linkcolor=darkred, citecolor=darkgreen, urlcolor=black}

\newtheorem{theorem}{Theorem}[section]
\newtheorem{lemma}[theorem]{Lemma}
\newtheorem{corollary}[theorem]{Corollary}
\newtheorem{proposition}[theorem]{Proposition}
\newtheorem{definition}[theorem]{Definition}

\newtheorem{assumption}[theorem]{Assumption}

\newtheorem{remark}[theorem]{Remark}
\newtheorem{example}[theorem]{Example}

\newcommand{\vb}{{\mtx{v}}}

\newcommand{\ub}{\mtx{u}}

\newcommand{\z}{{\mtx{z}}}
\newcommand{\A}{{\mtx{A}}}

\newcommand{\w}{\mtx{w}}

\newcommand{\opnorm}[1]{\left\|#1\right\|}
\newcommand{\fronorm}[1]{\left\|#1\right\|_{F}}

\newcommand{\twonorm}[1]{\left\|#1\right\|_{\ell_2}}

\newcommand{\infnorm}[1]{\left\|#1\right\|_{\ell_\infty}}
\newcommand{\nucnorm}[1]{\left\|#1\right\|_*}
\newcommand{\abs}[1]{\left|#1\right|}

\newcommand{\x}{\vct{x}}
\newcommand{\tx}{\vct{\tilde{x}}}

\newcommand{\W}{\mtx{W}}

\def\cL{\mathcal{L}}
\def\r{\vct{r}}
\def\z{\vct{z}}
\def\J{\mtx{J}}
\def\bJ{\mtx{\tilde{J}}}
\def\bJc{\mtx{\tilde{J}^c}}
\def\X{\mtx{X}}
\def\Dv{\mtx{D_v}}
\def\Gammab{\mtx{\Gamma}}
\def\normal{\mathcal{N}}

\def\diag{{\rm diag}}
\def\M{\mtx{M}}
\def\DM{\mtx{D_M}}
\def\A{\mtx{A}}
\def\B{\mtx{B}}
\def\DA{\mtx{D_A}}
\def\Q{\mtx{Q}}
\def\z{\vct{z}}

\def\de{{\rm d}}
\def\tJ{\tilde{\J}}
\def\mub{\vct{\mu}}
\def\wb{\vct{w}}
\def\mL{m_{\rm{L}}}
\def\mU{m_{\rm{U}}}
\def\tmU{\widetilde{m}_{\rm{U}}}

\def\tw{\widetilde{\wb}}
\def\s{\alpha}

\newcommand{\R}{\mathbb{R}}

\newcommand{\<}{\langle}
\renewcommand{\>}{\rangle}

\renewcommand{\P}{\operatorname{\mathbb{P}}}
\newcommand{\E}{\operatorname{\mathbb{E}}}

\newcommand{\vct}[1]{\bm{#1}}
\newcommand{\mtx}[1]{\bm{#1}}

\newcommand{\rank}{\operatorname{rank}}

\def\vec{{\rm vec}}
\def\tW{\widetilde{\mtx{W}}}
\def\tv{\widetilde{\vct{v}}}

\definecolor{ejc}{RGB}{0,0,255}

\numberwithin{equation}{section} 

\def \endprf{\hfill {\vrule height6pt width6pt depth0pt}\medskip}
\newenvironment{proof}{\noindent {\emph{Proof}} }{\endprf\par}

\newcommand*\samethanks[1][\value{footnote}]{\footnotemark[#1]}

\title{Theoretical insights into the optimization landscape of over-parameterized shallow neural networks}
\author{Mahdi Soltanolkotabi\thanks{Ming Hsieh Department of Electrical Engineering, University of Southern California, Los Angeles, CA} \quad  Adel Javanmard\thanks{Data Sciences and Operations Department, University of Southern California, Los Angeles, CA} \quad Jason D.~Lee\samethanks[2]}

\begin{document}
\maketitle
\begin{abstract}
In this paper we study the problem of learning a shallow artificial neural network that best fits a training data set. We study this problem in the over-parameterized regime where the number of observations are fewer than the number of parameters in the model. We show that with quadratic activations the optimization landscape of training such shallow neural networks has certain favorable characteristics that allow globally optimal models to be found efficiently using a variety of local search heuristics. This result holds for an arbitrary training data of input/output pairs. For differentiable activation functions we also show that gradient descent, when suitably initialized, converges at a linear rate to a globally optimal model. This result focuses on a realizable model where the inputs are chosen i.i.d.~from a Gaussian distribution and the labels are generated according to planted weight coefficients.
\end{abstract}
\vspace{0.3cm}
\begin{center}
\emph{\large{Dedicated to the memory of Maryam Mirzakhani.}}
\end{center}
\section{Introduction}
\subsection{Motivation}
Neural network architectures (a.k.a.~deep learning) have recently emerged as powerful tools for automatic knowledge extraction from raw data. These learning architectures have lead to major breakthroughs in applications such as visual object classification \cite{krizhevsky2012imagenet}, speech recognition \cite{mohamed2012acoustic} and natural language processing \cite{collobert2008unified}. Despite their wide empirical use the mathematical success of these architectures remains a mystery. Although the expressive ability of neural networks is relatively well-understood \cite{barron1994approximation}, computational tractability of training such networks remains a major challenge. In fact, training neural nets is known to be NP-hard even for very small networks \cite{blum1988training}. The main challenge is that training neural networks correspond to extremely high-dimensional and nonconvex optimization problems and it is not clear how to provably solve them to global optimality. Worst-case pessimism aside, these networks are trained successfully in practice via local search heuristics on real or randomly generated data. In particular, over-parameterized neural networks-where the number of parameters exceed the number of data samples-can be optimized to global optimality using local search heuristics such as gradient or stochastic gradient methods \cite{zhang2016understanding}. In this paper we wish to provide theoretical insights into this phenomenon by developing a better understanding of optimization landscape of such over-parameterized shallow neural networks.
\subsection{Problem formulation and models}
A fully connected artificial neural network is composed of computational units called neurons. The neurons are decomposed into layers consisting of one input layer, one output layer and a few hidden layers with the output of each layer fed in (as input) to the next layer. In this paper we shall focus on neural networks with only a single hidden layer with $d$ inputs, $k$ hidden neurons and a single output. The overall input-output relationship of the neural network in this case is a function $f:\R^d\rightarrow\R$ that maps the input vector $\vct{x}\in\R^d$ into a scalar output via the following equation
\begin{align}
\label{model}
\vct{x}\mapsto f_{\vct{v},\mtx{W}}(\vct{x}):=\sum_{\ell=1}^k\vct{v}_\ell\phi(\langle\mtx{w}_\ell,\vct{x}\rangle).
\end{align}
In the above the vectors $\vct{w}_\ell\in\R^d$ contains the weights of the edges connecting the input to the $\ell$th hidden node and $\vct{v}_\ell\in\R$ is the weight of the edge connecting the $\ell$th hidden node to the output. Finally, $\phi:\R\rightarrow\R$ denotes the activation function applied to each hidden node. For more compact notation we gathering the weights $\vct{w}_\ell/\vct{v}_\ell$ into larger matrices $\mtx{W}\in\R^{k\times d}$ and $\vct{v}\in\R^k$ of the form
\begin{align*}
\mtx{W}=\begin{bmatrix}\vct{w}_1^T\\\vct{w}_2^T\\\vdots\\\vct{w}_k^T\end{bmatrix}\quad\text{and}\quad\vct{v}=\begin{bmatrix}{v}_1\\{v}_2\\\vdots\\ {v}_k\end{bmatrix}.
\end{align*}
We can now rewrite our input-output model \eqref{model} in the more succinct form
\begin{align}
\label{model2}
\vct{x}\mapsto f_{\vct{v},\mtx{W}}(\vct{x}):=\vct{v}^T\phi(\mtx{W}\vct{x}).
\end{align}
Here, we have used the convention that when $\phi$ is applied to a vector is corresponds to applying $\phi$ to each entry of that vector. When training a neural network, one typically has access to a data set consisting of $n$ feature/label pairs $(\vct{x}_i,y_i)$ with $\vct{x}_i\in\R^d$ representing the feature and $y_i$ the associated label. We wish to infer the best weights $\vct{v},\mtx{W}$ such that the mapping $f_{\vct{v},\mtx{W}}$ best fits the training data. More specifically, we wish to minimize the misfit between $f_{\vct{v},\mtx{W}}(\vct{x}_i)$ and $y_i$ via an optimization problem of the form
\begin{align}
\underset{\vct{v}, \mtx{W}}{\min}\text{ }\mathcal{L}(\vct{v},\mtx{W})\,,
\end{align}
where
\begin{align}
\label{landscape}
\mathcal{L}(\vct{v},\mtx{W}):=\frac{1}{2n}\sum_{i=1}^n \left(y_i-f_{\vct{v},\mtx{W}}(\vct{x}_i)\right)^2=\frac{1}{2n}\sum_{i=1}^n\left(y_i-\vct{v}^T\phi(\mtx{W}\vct{x}_i)\right)^2.
\end{align}
In this paper we wish to understand the landscape of optimization problems of the form \eqref{landscape}. 
\section{Main results}
In this section we discuss the main results of this paper. The first set of results focus on understanding the global landscape of neural network optimization with one hidden layer with a particular activation function. We also discuss how this landscape characterization enables algorithms to find a global optima in polynomial time. The second set of results focuses on the local convergence of gradient descent but applies to a broad set of activation functions.
\subsection{Global landscape analysis with quadratic activations}\label{subsec:global}
We begin by discussing some global properties of the loss function of training neural networks.

\begin{theorem}\label{globthm}
Assume we have an arbitrary data set of input/label pairs $\vct{x}_i\in\R^d$ and $y_i$ for $i=1,2,\ldots,n$. Consider a neural network of the form
\begin{align*}
\vct{x}\mapsto \vct{v}^T\phi(\mtx{W}\vct{x}),
\end{align*}
 with $\phi(z)=z^2$ a quadratic activation and $\mtx{W}\in\R^{k\times d}$, $\vct{v}\in\R^k$ denoting the weights connecting input to hidden and hidden to output layers. We assume $k\ge 2d$ and set the weights of the output layer $\vct{v}$ so as to have at least $d$ positive entries and at least $d$ negative entries. Then, the training loss as a function of the weights $\mtx{W}$ of the hidden layer 
 \begin{align*}
\mathcal{L}(\mtx{W})=\frac{1}{2n}\sum_{i=1}^n\left(y_i-\vct{v}^T\phi(\mtx{W}\vct{x}_i)\right)^2,
 \end{align*}
 obeys the following two properties.
\begin{itemize}
\item There are no spurious local minima, i.e.~all local minima are global.
\item All saddle points have a direction of strictly negative curvature. That is, at a saddle point $\mtx{W}_s$ there is a direction $\mtx{U}\in\R^{k\times d}$ such that
\begin{align*}
\emph{vect}(\mtx{U})^T\nabla^2\mathcal{L}(\mtx{W}_s)\emph{vect}(\mtx{U})<0.
\end{align*}
\end{itemize} 
Furthermore, 
for almost every data inputs $\{\vct{x}_i\}_{i=1}^n$,
as long as
\begin{align*}
d\le n\le cd^2,
\end{align*}
the global optimum of $\mathcal{L}(\mtx{W})$ is zero. 
Here, $c>0$ is a fixed numerical constant.
\end{theorem}
The above result states that given an arbitrary data set, the optimization landscape of fitting neural networks have favorable properties that facilitate finding globally optimal models. In particular, by setting the weights of the last layer to have diverse signs all local minima are global minima and all saddles have a direction of negative curvature. This in turn implies that gradient descent on the input-to-hidden weights, when initialized at random, converges to a global optima. All of this holds as long as the neural network is sufficiently wide in the sense that the number of hidden units exceed the dimension of the inputs by a factor of two ($k\ge 2d$). 

An interesting and perhaps surprising aspect of  the first part of this theorem is its generality: it applies to any arbitrary data set of input/label pairs of any size! However, this result only guarantees convergence to a global optima but does not explain how good this global model is at fitting the data. Stated differently, it does not provide a bound on the optimal value. Such a bound may not be possible with adversarial data. However, at least intuitively, one expects the optimal value to be small when the input data $\vct{x}_i$ are sufficiently diverse and the neural network is sufficiently over-parameterized. The second part of Theorem \ref{globthm} confirms that this is indeed true. In particular assuming the input data $\vct{x}_i$ are distributed i.i.d.~$\mathcal{N}(0,\mtx{I}_d)$, and the total number of weights ($kd$) exceeds the number of data samples ($n$), the globally optimal model perfectly fits the labels (has optimal value of $0$). The interesting part of this result is that it still holds with arbitrary labels. Thus, this theorem shows that with random inputs, and a sufficiently diverse choice of the hidden-output weights, using gradient descent to iteratively update the input-hidden layer weights converges to a model that provably fits arbitrary labels! This result is also inline with recent numerical evidence in \cite{zhang2016understanding} that demonstrates that stochastic gradient descent learns deep, over-parametrized models with zero training error even with an arbitrary choice of labels.

While the above theorem shows that the saddles are strict and there is a direction of negative curvature at every saddle point, the margin of negativity is not quantified. Thus, the above theorem does not provide explicit convergence guarantees. In the theorem below we provide such a guarantee albeit for a more restrictive data model. 
\begin{theorem}\label{gthmrand}
Assume we have a data set of input/label pairs $\{(\vct{x}_i,y_i)\}_{i=1}^n$ with 
the labels $y_i\in\R$ generated according to a planted two layer neural network model of the form
\begin{align*}
y_i=\langle\vct{v}^*,\phi(\mtx{W}^*\vct{x}_i)\rangle,
\end{align*}
with $\phi(z)=z^2$ a quadratic activation and $\mtx{W}^*\in\R^{k\times d}$, $\vct{v}^*\in\R^k$ the weights of the input-hidden and hidden-output layer with $\sigma_{\min}(\mtx{W}^*)>0$.  Furthermore, assume $k\ge d$ and that all the non-zero entries of $\vct{v}^*$ have the same sign (positive or negative). We set the hidden-output weights to a vector $\vct{v}\in\R^k$ with non-zero entries also having the same sign with at least $d$ entries strictly nonzero (positive if the nonzero entries of $\vct{v}^*$ are positive and negative otherwise). Then, as long as
\begin{align*}
d\le n\le cd^2,
\end{align*}
with $c$ a fixed numerical constant, then the training loss as a function of the input-output weights $\mtx{W}$ of the hidden layer 
 \begin{align*}
\mathcal{L}(\mtx{W})=\frac{1}{2n}\sum_{i=1}^n\left(y_i-\vct{v}^T\phi(\mtx{W}\vct{x}_i)\right)^2,
 \end{align*}
 obeys the following two properties, for almost every input data $\{\vct{x}_i\}_{i=1}^n$.
 \begin{itemize}
\item There are no spurious local minima, i.e.~all local minima are global optima. 
\item All saddle points have a direction of negative curvature. That is, at a saddle point $\mtx{W}_s$ there is a direction $\mtx{U}\in\R^{k\times d}$ such that
\begin{align*}
\emph{vect}(\mtx{U})^T\nabla^2\mathcal{L}(\mtx{W}_s)\emph{vect}(\mtx{U})<0.
\end{align*}
\item 
The global optima has loss value equal to zero ($\mathcal{L}(\mtx{W})=0$).
\end{itemize}

Furthermore, suppose that the inputs $\vct{x}_i\in\R^d$ are distributed i.i.d.~$\mathcal{N}(\vct{0},\mtx{I}_d)$. Assume $k>d$ and $cd\log d\le n \le Cd^2$ for fixed constants $c$ and $C$. Also, assume we set all entries of $\vct{v}$ to $\nu$ (i.e.~$\vct{v}=\nu\vct{1}$) with $\nu$ having the same sign as the nonzero entries of $\vct{v}^*$.
\begin{itemize}
\item Then all points $\mtx{W}$ satisfying the approximate local minima condition
\begin{align}
\label{approxmin}
\fronorm{\nabla \mathcal{L}(\mtx{W})}\le \epsilon_g\quad\text{and}\quad\nabla^2\mathcal{L}(\mtx{W})\succeq -\epsilon_{H}\mtx{I},
\end{align}
obey
\begin{align*}
\mathcal{L}(\mtx{W})\le\frac{\epsilon_g}{\sqrt{\nu}}\max\left(\sqrt{1+14\nucnorm{{\mtx{W}^*}^T\mtx{D}_{\vct{v}^*}\mtx{W}^*}},4\frac{\epsilon_g}{\sqrt{\nu}}\right)+\frac{\epsilon_{H}}{2\nu}\nucnorm{{\mtx{W}^*}^T\mtx{D}_{\vct{v}^*}\mtx{W}^*},
\end{align*}
with probability at least $1-12de^{-\gamma d}-8/d$. Here $\gamma$ is a fixed numerical constant.
\end{itemize} 
\end{theorem}
The above result considers the setting where the input-output data set is generated according to a neural network model with Gaussian random input vectors. We show that if the data is generated according this model, then as long as the neural network is over-parameterized ($n\lesssim kd$) then all points obeying condition \eqref{approxmin} have small objective value. 

The reason this result is useful is that points obeying the approximate local minimum condition \eqref{approxmin} can be found in time depending polynomially on $\frac{1}{\epsilon_g}$ and $\frac{1}{\epsilon_{H}}$. Algorithms that provably work in this setting include cubic regularization \cite{nesterov2006cubic}, trust region methods \cite{curtis2014trust}, approximate cubic regularization schemes \cite{agarwal2016finding, carmon2016gradient, carmon2016accelerated}, randomly initialized and perturbed (stochastic) gradients \cite{jin2017escape, ge2015escaping, lee2016gradient, levy2016power, jin2017escape}. Therefore, the above theorem demonstrates that a weight matrix with small training error(i.e.~$\mathcal{L}(\mtx{W})\le\epsilon$) can be found in a computationally tractable manner (specifically with poly$(\frac{1}{\epsilon})$ computational effort).
\subsection{Local convergence analysis with general activations}
We now extend our analysis to general activations functions. However, our results in this section require a sufficiently close initialization to the global optimum. Starting from such a sufficiently close initialization we apply gradient descent updates based on the loss function \eqref{landscape}. 
Our results apply to any differentiable activation function with bounded first and second order derivatives. We believe our result will eventually extend also to non-differentiable activations by smoothing techniques but we do not pursue this direction in this paper. 

Before we state our theorem, we make a technical assumption regarding the activation $\phi$. 
\begin{assumption}\label{ass-phi}
For $\sigma\in \R_{\ge 0 }$, define $\mu(\sigma) = E[\phi'(\sigma g)]$ and $\gamma(\sigma) = E[\phi''(\sigma g)]$ where $g\sim\mathcal{N}(0,1)$. We consider activations $\phi$ such that $\mu(\sigma)$ is zero/nonzero everywhere. Likewise, we assume that one of the following holds true for $\gamma(\sigma)$:
\begin{itemize}
\item [$(a)$] Function $\gamma(\sigma)$ is nonzero everywhere.
\item [$(b)$] Function $\gamma(\sigma)$ is identical to zero. 
\end{itemize}
\end{assumption}
We note that $\mu(\sigma)$ and $\gamma(\sigma)$ can be thought of as the average slope and curvature of the activation function. Thus the assumption on $\mu(\sigma)$ can be interpreted as the activation should have a nonzero average slope for all $\sigma>0$ (i.e.~the mapping is somewhat correlated with the identity mapping $\phi(x)=x$, under any gaussian measure) or has average slope equal to zero for all $\sigma>0$ (i.e.~the mapping has no correlation with the identity mapping, under any gaussian measure).

\begin{example}\label{myex}
We provide several examples of an activation function that satisfy Assumption~\ref{ass-phi}.
\begin{enumerate}
\item (Softplus) $\phi_b(z) = \frac{1}{b} \log(1+e^{bz})$, for a fixed $b >0$.
\item (Sigmoid) $\phi_b(z) = \frac{1}{1+e^{-bz}}$, for a fixed $b>0$.
\item (Erf) $\phi(z) =\frac{2}{\sqrt{\pi}} \int_0^z e^{-t^2/2} \de t$.
\item (Hyperbolic Tangent) $\phi(z) = \tanh(z)$. 
\end{enumerate}
Note that all of these activations obey $\mu(\sigma) > 0$ as they are all strictly increasing. Furthermore, the Softplus activation satisfies Assumption~\ref{ass-phi} (a) because it is strictly convex, while the other three activations satisfy  Assumption~\ref{ass-phi} (b), because $\phi''$ is an odd function in these cases. These activations along with the popular ReLU activation ($\phi(z)=\max(0,z)$) are depicted in Figure \ref{fig:softplus-relu}. 
\end{example}
\begin{figure}
\centering
\begin{tikzpicture}[scale=1] 
 \begin{groupplot}[scale=1,group style={group size=2 by 1,horizontal sep=1cm,xlabels at=edge bottom, ylabels at=edge left,xticklabels at=edge bottom},xlabel=$z$,
        ylabel=$\phi(z)$, legend style={font=\tiny,at={(0.25,0.95)},anchor=north,legend cell align=left}]
    \nextgroupplot
     
     \addplot [blue,line width=2pt] table[x index=0,y index=1]{./Activations};
     \addlegendentry{ReLU}
     
     \addplot [magenta,line width=2pt] table[x index=0,y index=2]{./Activations};
     \addlegendentry{Softplus}
    
    \nextgroupplot 
     \addplot [red,line width=2pt] table[x index=0,y index=3]{./Activations};
     \addlegendentry{Sigmoid}
     
     \addplot [black,line width=2pt] table[x index=0,y index=4]{./Activations};
     \addlegendentry{Erf}
     
     \addplot [teal,line width=2pt] table[x index=0,y index=5]{./Activations};
     \addlegendentry{Hyperbolic tangent}   
\end{groupplot}
\end{tikzpicture}
\caption{Different activations from Example \ref{myex} with $b=4$ along with the ReLU activation.}
\label{fig:softplus-relu}
\end{figure}
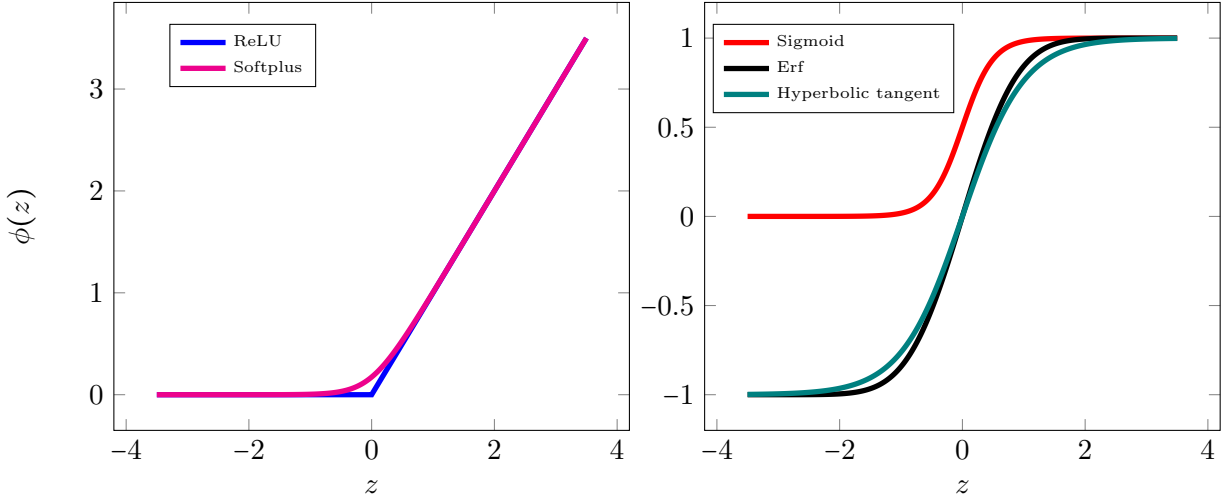
We now have all the elements to state our local convergence result.
\begin{theorem}\label{localthm}
Assume we have a data set of input/label pairs $\{(\vct{x}_i,y_i)\}_{i=1}^n$ with the inputs $\vct{x}_i\in\R^d$ distributed i.i.d.~$\mathcal{N}(\vct{0},\mtx{I}_d)$ and the labels $y_i\in\R$ generated according to a planted two layer neural network model of the form
\begin{align*}
y_i=\langle\vct{v}^*,\phi(\mtx{W}^*\vct{x}_i)\rangle,
\end{align*} 
Here, $\phi:\R\rightarrow \R$ is any activation function with bounded first and second derivatives, satisfying Assumption~\ref{ass-phi}($a$). Further, $\mtx{W}^*\in\R^{k\times d}$, $\vct{v}^*\in\R^k$ are the weights of the input-hidden and hidden-output layer. We also assume that the planted weight vector/matrix $\vct{v}^*/\mtx{W}^*$ obey
\begin{align}
\label{boundassthm}
0<v_{\min} \le |v^*_\ell| \le v_{\max}\quad0<w_{\min} \le \twonorm{\wb^*_\ell} \le w_{\max}\quad\text{for }\ell=1,2,\ldots,k\,,
\end{align}
for some fixed constants $v_{\min}, v_{\max}, w_{\min}$, and $w_{\max}$. Further, assume that $k\ge d$ and
\begin{align}
d \le n \le  c_0 \frac{\sigma_{\min}^4(\W^*)}{\sigma_{\max}^8(\W^*)} d^2 \label{samplesize}
\end{align}  
 for a sufficiently small constant $c_0>0$.

To estimate the weights $\vct{v}^*$ and $\mtx{W}^*$, we start from initial weights $\vct{v}_0$ and $\vct{W}_0$
\begin{align}
\fronorm{\W_0-\W^*} &\le C_0\frac{\sigma_{\min}^3(\W^*)}{\sigma_{\max}(\W^*)}\cdot \frac{d^{2.5}}{n^{1.5} k}  \,, \label{eq:initW0}\\ 
\opnorm{\vb_0-\vb^*}_\infty &\le C_0\frac{\sigma_{\min}^3(\W^*)}{\sigma_{\max}(\W^*) }\cdot \frac{d^{2.5}}{n^{1.5} k^{1.5}} \,, \label{eq:initv0}
\end{align}
and apply gradient descent updates of the form
\begin{align}
\label{gradd}
\vct{v}_{\tau+1} &= \vct{v}_{\tau}-\s\nabla_{\vct{v}} \mathcal{L}(\vct{v}_\tau,\mtx{W}_\tau),\nonumber\\
\vct{W}_{\tau+1} &=\vct{W}_{\tau}-\s\nabla_{\vct{W}} \mathcal{L}(\vct{v}_\tau,\mtx{W}_\tau),
\end{align}
with the learning rate obeying $\s\le 1/\beta$. Then there is an event of probability at least $1  - n^{-1} - 2ne^{-b \sqrt{d}}$, such that on this event starting from any initial point obeying \eqref{eq:initW0}-\eqref{eq:initv0} the gradient descent updates \eqref{gradd} satisfy
\begin{align*}
\cL(\vb_\tau,\W_\tau) \le \left(1- c \s \frac{d}{n} \right)^\tau \cL(\vb_0,\W_0)\,.
\end{align*} 
Here, $\beta$ is given by
\begin{align}
\beta: = C \left(\sigma_{\max}^2(\W^*) +  1\right) k\,,
\end{align}
and $b, C_0, c_0, C, c>0$ are fixed numerical constants.
\end{theorem}
\begin{remark} We note that the above theorem is still valid if Assumption~\ref{ass-phi} (b) holds in lieu of Assumption~\ref{ass-phi} (a). The only change is that the term $\sigma_{\min}(\W)$ should be now replaced by one in Equations~\eqref{samplesize}, \eqref{eq:initW0} and \eqref{eq:initv0}.
\end{remark}
The above result considers the setting where the input-output data set is generated according to a neural network model with a general activation and Gaussian random input vectors. We show that if the data is generated according this model, then as long as the neural network is over-parameterized ($n\lesssim kd$) then gradient descent converges to the planted model when initialized close to this planted model. This implies that a training error of zero can be achieved locally, using gradient descent. We would like to note that assumptions \eqref{boundassthm} on the weights of the planted neural networks are only made to avoid unnecessarily complicated expressions in the statement of the theorem. As it will become clear in the proofs (Section \ref{proofs}) our result continues to hold without these assumptions.

\input{Num2}
\section{Related Work}
\label{sec:related}
As mentioned earlier, neural networks have enjoyed great empirical success \cite{krizhevsky2012imagenet, mohamed2012acoustic, collobert2008unified}. To explain this success, many papers have studied the expressive ability of shallow neural networks dating back to the 80s (e.g.~see \cite{barron1994approximation}). More recently, interesting results have focused on the expressive ability of deeper and sparser architectures \cite{telgarsky2016benefits, bolcskei2017optimal, kuo2017cnn, wiatowski2017energy}. Computational tractability of training networks however is still a major challenge. In fact, training neural nets is known to be NP-hard even for very small networks \cite{blum1988training}. Despite this worst-case pessimism, local search heuristics such as gradient descent are surprisingly effective. For instance, \cite{zhang2016understanding} empirically demonstrated that sufficiently over-parametrized networks can be efficiently optimized to near global optimality with stochastic gradient descent.

For a neural network with zero hidden units and a single output with a monotonic activation function $\sigma$, numerous authors \cite{mei2016landscape,hazan2015beyond,kakade2011efficient,kalai2009isotron, soltanolkotabi2017learning} have shown that gradient-based methods converge to the global optimum under various assumptions and models. As a result of this literature a good understanding of the optimization landscape of learning single activations have emerged (at least when the data set is generic). Roughly stated, in this case the loss function only has a single local optima that coincides with the global optima and the gradient descent direction is sufficiently correlated with the direction pointing to the global optimum. Thus these results are able to explain the success of gradient-based methods. However, when there are hidden units, the loss function exhibits a completely different landscape, so that such analyses do not easily generalize. 

We now turn our attention to global convergence results known for deeper architectures. For a 2-layered network with leaky ReLU activation,~\cite{soudry2016no} showed that gradient descent on a modified loss function can obtain a global minimum of the modified loss function; however, this does not imply reaching a global minimum of the original loss function. Under the same setting, \cite{xie2016diversity} showed that critical points with large ``diversity" are near global optimality. \cite{choromanska2015loss} used several assumptions to simplify the loss function to a polynomial with i.i.d.~Gaussian coefficients. They then showed that every local minima of the simplified loss has objective value comparable to the global minima. \cite{kawaguchi2016deep} used similar assumptions to show that all local minimum are global minimum in a nonlinear network. However, \cite{choromanska2015loss,kawaguchi2016deep} require an \emph{independent activations} assumption meaning the activations of the hidden units are independent of the input and/or mutually independent, which is violated in practice. In comparison, our global convergence results hold for any arbitrary data set, without any additional assumptions. However, our global convergence results (see Section~\ref{subsec:global}) only focus on quadratic activations with a single hidden layer. 
 
We would like to mention a few interesting results regarding asymptotic characterizations of the dynamics of training shallow neural networks with Gaussian inputs using ideas from statistical physics \cite{saad1995line, biehl1996transient, vicente1997functional}. Saad and Solla \cite{saad1995line} study the online dynamics of learning a fully connected neural network with one-hidden layer in a Gaussian planted model (or Gaussian student-teacher model). In this model the output labels are generated according to planted weight vectors (a.k.a.~a teacher network) with Gaussian input data. Then a new ``student" network is trained to find the teacher weights. This result focuses on the case where the hidden-to-output weights are fixed to $+1$ (a.k.a.~soft committee machines) and the asymptotic regime where the number of inputs and the iterations tend to infinity ($d\rightarrow +\infty$ and $\tau\rightarrow \infty$). 
The authors provide certain dynamical equations that describes the asymptotic behavior of the training process for one-hidden layer neural networks using an ``online SGD" or ``resampled SGD" procedure where the SGD update is performed using a fresh and independent data point per new iteration. {These dynamical equations do not admit a closed form solution, but do enable the analysis of the convergence behavior close to stationarity.} Furthermore, by simulating these dynamical equations it is possible to gain some insights about the evolution of the online SGD updates. In particular, the authors use such simulations to gain insights into the convergence/divergence of such neural networks in the special case where diagonal input-hidden weights ($\mtx{W}$) are used. The paper \cite{biehl1996transient} also studies studies soft-committee machines under similar assumptions to those of \cite{saad1995line}. The authors of \cite{biehl1996transient} use tools from statistical physics (and in particular the aforementioned dynamical equations) to demonstrate that the dynamics of training such neural networks has several fixed points and find learning-rate dependent phenomena, such as splitting and disappearing of fixed points. This paper also differs from \cite{saad1995line} in that it also applies to transient iterations (finite $\tau$). Finally, the paper \cite{vicente1997functional} also studies soft committee machines but instead of online SGD, they present the analysis for a locally optimized online learning algorithm that focuses on extracting the largest possible amount of information from each new sample. The authors also demonstrate numerically that this choice leads to faster escape from the plateaux and hence faster decay of the generalization error. Our results are not directly comparable to this literature as it differs in terms of assumptions and conclusions in a variety of ways. In particular we focus on (1) analytic formulas, (2) gradient descent without resampling, (3) a non-asymptotic regime (finite $d$) and transient dynamics (finite $\tau$), (4) the over-parameterized case $n<kd$, and (5) general hidden-output weights. Some of the results presented in this paper also apply without assuming a planted model or random/Gaussian input data (e.g.~see Theorem \ref{globthm}).

We would like to note that there is also interesting growing literature on learning shallow neural networks with a single hidden layer with i.i.d.~inputs, and under a realizable model (i.e.~the labels are generated from a network with planted weights) \cite{tian2017analytical, brutzkus2017globally, zhang2017electron, li2017convergence, janzamin2015beating, zhong2017recovery}. For isotropic Gaussian inputs, \cite{tian2017analytical} shows that with two hidden unites ($k=2$) there are no critical points for configurations where both weight vectors fall into (or outside) the cone of ground truth weights. With the same assumptions, \cite{brutzkus2017globally} proves that for a single-hidden ReLU network with a single non-overlapping convolutional filter, all local minimizers of the population loss are global; they also give counter-examples in the overlapping case  and prove the problem is NP-hard when inputs are not Gaussian. \cite{zhang2017electron} show that for single-hidden layer networks with non-standard activation functions gradient descent converges to global minimizers. \cite{li2017convergence} focuses on a Recursive Neural Net (RNN) model where the hidden to output weights are close to the identity and shows that stochastic methods converge to the planted model over the population objective. \cite{zhong2017recovery} studies general single-hidden layer networks and shows that a version of gradient descent which uses a fresh batch of samples in each iteration converges to the planted model. This holds using an initialization obtained via a tensor decomposition method. Our approach and local convergence results differ from this literature in a variety of different ways. First, unlike some of these results such as \cite{brutzkus2017globally,li2017convergence}, we study the optimization properties of the empirical function, not its expected value. Second, we focus on the over-parametrized regime ($n<<kd$) which is the regime where most popular neural networks are trained. Mathematically, this is an important distinction as in this data-poor regime the empirical loss is no longer close to the population loss and therefore one can no longer infer the convergence behavior of gradient descent based on connecting the empirical loss to the population loss. Third, we optimize over both weights $\vct{v}$ and $\mtx{W}$, while most of this literature assumes $\vct{v}=\vct{v}^*=\vct{1}$ \cite{li2017convergence, zhong2017recovery}. Finally, our framework does not require a fresh batch of samples per new gradient iteration as in \cite{zhong2017recovery}.

Several publications study the effect of over-parametrization on the training of neural networks \cite{poston1991local,haeffele2015global, nguyen2017loss}. These results require a large amount of over-parametrization, mainly that the width of one of the hidden layers to be greater than the number of training samples, which is unrealistic for commonly used networks. For instance, for the case of a single hidden layer using the notation of this paper these publications require $n\lesssim k$. These results also require additional technical assumptions which we do not detail here. However, these results work for fairly general activations and also deep architectures. In comparison, our global optimality results only allow for quadratic activations and a single hidden layer. However, our results are completely deterministic and allow for modest over-parameterization ($n\lesssim kd$). Generalizing these result to other activations and deeper architectures is an interesting future direction. While discussing quadratic activations we would like to mention the interesting paper \cite{livni2014computational}. This paper is perhaps the first result to clearly state that over-parameterization is helpful for optimization purposes. The authors further showed that the Frank-Wolfe algorithm can optimize a neural net with quadratic activations of the form $\phi(\vct{x})=b+\langle \vct{w}_0, x\rangle + \sum_{\ell=1}^K\alpha_k (\langle w_\ell, x\rangle)^2$ under norm constraints of the form $\abs{\alpha_\ell}\le1$ and $\twonorm{\vct{w}_\ell}=1$. In comparison we focus on slightly less general activations (no linear term) and study landscape properties that are directly useful for analyzing gradient descent/local search methods in lieu of Franke-Wolfe type schemes. Finally, in the case of over-parameterized neural networks not all global optima may generalize to new data instances. Understanding the generalization capability of the solutions reached by (stochastic) gradient descent is an important research direction. See \cite{bartlett2017spectrally} for an interesting result in this direction under margin assumptions. 

Another line of research \cite{goel2016reliably, shalev2011learning, zhang2016l1} focuses on improper learning models using kernel-based approaches. These results hold under much more general data models than the realizable case. However, the practical success of deep learning is intricately tied to using gradient-based training procedures, and the learnability of these networks using improper learning does not explain the  success of gradient-based methods. Related, \cite{janzamin2015beating} proposes a method of moments estimator using tensor decomposition.

Finally, we would like to mention a few interesting recent developments that appeared after the initial arxiv submission of this paper. The authors of \cite{safran2017spurious} demonstrated analytically (with variable precision arithmetics) that without over-parameterization the landscape of one-hidden layer neural networks has bad local minima as soon as the number of hidden unites exceeds six ($k\ge 6$). This interesting result clearly demonstrates that the over-parametrization is necessary to ensure a favorable optimization landscape when training one-hidden layer neural networks. Another recent paper \cite{soudry2017exponentially} by Soudry and Hoffer suggests that ``most" local minima are global in the over-parameterized regime. This paper provides further evidence that over-parametrizing neural networks leads to more favorable optimization landscapes when training such neural nets. Finally, the recent paper \cite{li2017algorithmic} proves that gradient descent when initialized at zero can lead to training errors close to (but not equal to) zero. This result holds for one-hidden layer neural networks with quadratic activations and Gaussian inputs. Developing rigorous convergence guarantees to a global optima with training error equal to zero which also applies to other activations is an important future research direction.

\section{Preliminaries and notations}
Before we dive into the proofs in the next two sections we collect some useful results and notations in this section.
\subsection{Notations}
For two matrices $\mtx{A} = [\mtx{A}_1,\dotsc, \mtx{A}_p]\in\R^{m\times p}$ and $\mtx{B} =[\mtx{B}_1,\dotsc, \mtx{B}_p]\in \R^{n\times p}$, we define their Khatri-Rao product as $\mtx{A} * \mtx{B}  = [\mtx{A}_1\otimes \mtx{B}_1,\dotsc, \mtx{A}_p\otimes \mtx{B}_p]\in\R^{mn\times p}$, where $\otimes$ denotes the Kronecher product. For two matrices $\A, \mtx{B}$,
we denote their Hadamard (entrywise) product by $\A\circ\mtx{B}$.
For a matrix $\A$, we denote its maximum and minim singular values by $\sigma_{\max}(\A)$ and $\sigma_{\min}(\A)$, respectively. For a random variable $Y$, its sub-exponential norms is defined as
\[
\|Y\|_{\psi_1} = \inf\{C>0:\; \E \exp(|Y|/C) \le 2\}\,.
\]
Further, for a centered random vector $\vct{x}\in \R^d$, we define its sub-exponential norm as
\begin{align}
\|\vct{x}\|_{\psi_1}  = \sup_{\vct{y}\in S^{d-1}} \|\<\vct{x},\vct{y}\>\|_{\psi_1}\label{def2:subexp}
\end{align}
Throughout the paper we use $c, C$ to refer to constants whose values may change from line to line.
\subsection{Derivative calculations}
In this section we gather some derivative calculations that we will use throughout the proof. As a reminder the loss function is given by
\begin{align*}
\mathcal{L}(\vct{v},\mtx{W})=\frac{1}{2n}\sum_{i=1}^n\left(\vct{v}^T\phi(\mtx{W}\vct{x}_i)-y_i\right)^2.
\end{align*}
To continue let us define the residual vector $\vct{r}\in\R^n$ with the $i$th entry given by
\begin{align*}
r_i=\vct{v}^T\phi(\mtx{W}\vct{x}_i)-y_i.
\end{align*}
\subsubsection{First order derivatives}
We begin by calculating the gradient with respect to the weight matrix $\mtx{W}$. 
To this aim we begin by calculating the gradient with respect to the $q$th row of $\mtx{W}$ denoted by $\vct{w}_q$. This is equal to
\begin{align}
\label{gradw}
\nabla_{\vct{w}_q} \mathcal{L}(\vct{v},\mtx{W})=\frac{v_q}{n}\sum_{i=1}^n\left(\vct{v}^T\phi(\mtx{W}\vct{x}_i)-y_i\right)\phi'(\langle\vct{w}_q,\vct{x}_i\rangle)\vct{x}_i=\frac{v_q}{n}\sum_{i=1}^nr_i\phi'(\langle\vct{w}_q,\vct{x}_i\rangle)\vct{x}_i.
\end{align}
Aggregating these gradients as a row vector and setting $\mtx{D}_{\vct{v}}=\diag(v_1,\dotsc,v_k)$ we arrive at
\begin{align}
\label{gradmat}
\nabla_{\mtx{W}} \mathcal{L}(\vct{v},\mtx{W})=\mtx{D}_{\vct{v}}\left(\frac{1}{n}\sum_{i=1}^nr_i\phi'(\mtx{W}\vct{x}_i)\vct{x}_i^T\right).
\end{align}
We also define the Jacobian matrix $\mtx{J}=\begin{bmatrix}\mtx{J}_1 & \mtx{J}_2 & \ldots & \mtx{J}_n\end{bmatrix}\in\R^{kd\times n}$ with
\begin{align*}
\mtx{J}_i=\begin{bmatrix}  v_1 \phi'(\w_1^T\x_i) \x_i \\  v_2 \phi'(\w_2\x_i) \x_i \\ \vdots\\  v_k \phi'(\w_k^T\x_i) \x_i \end{bmatrix}.
\end{align*}
Let $\mtx{X}\in\R^{d\times n}$ be the data matrix with the $i$th column equal to $\vct{x}_i$. Using the Khatri-Rao product we can rewrite the Jacobian matrix in the form
\begin{align}
\label{myJ}
\mtx{J}=\mtx{D}_{\vct{v}}\phi'(\mtx{W}\mtx{X})\ast\mtx{X}.
\end{align}
 Using this Jacobian matrix we can write the vectorized version of the gradient, i.e.~gradient with respect to vect$(\mtx{W})$ as
\begin{align}
\label{gradvec}
\nabla_{\text{vect}(\mtx{W})} \mathcal{L}(\vct{v},\mtx{W})=\frac{1}{n} \J \r.
\end{align}
Taking the derivative with respect to $v_q$ we arrive at
\begin{align*}
\frac{\partial}{\partial v_q}\mathcal{L}(\vct{v},\mtx{W})=\frac{1}{n}\sum_{i=1}^n\left(\vct{v}^T\phi(\mtx{W}\vct{x}_i)-y_i\right)\phi(\langle \vct{w}_q,\vct{x}_i\rangle)=\frac{1}{n}\sum_{i=1}^nr_i\phi(\langle \vct{w}_q,\vct{x}_i\rangle).
\end{align*}
Thus the gradient with respect to $\vct{v}$ is equal to 
\begin{align}
\label{gradv}
\nabla_{\vct{v}}\mathcal{L}(\vct{v},\mtx{W})=\frac{1}{n}\phi(\mtx{W}\mtx{X})\vct{r}.
\end{align}
\subsubsection{Second order derivatives}
Using \eqref{gradw} we have
\begin{align*}
\frac{\partial^2}{\vct{w}_p^2} \mathcal{L}(\vct{v},\mtx{W})=\frac{v_p}{n}\sum_{i=1}^n\left(\vct{v}^T\phi(\mtx{W}\vct{x}_i)-y_i\right)\phi''(\langle\vct{w}_p,\vct{x}_i\rangle)\vct{x}_i\vct{x}_i^T+\frac{v_p^2}{n}\sum_{i=1}^n\left(\phi'(\langle\vct{w}_p,\vct{x}_i\rangle)\right)^2\vct{x}_i\vct{x}_i^T.
\end{align*}
Also using \eqref{gradw}, for $p\neq q$
\begin{align*}
\frac{\partial^2}{\partial\vct{w}_p\vct{w}_q} \mathcal{L}(\vct{v},\mtx{W})=\frac{v_pv_q}{n}\sum_{i=1}^n\phi'(\langle\vct{w}_p,\vct{x}_i\rangle)\phi'(\langle\vct{w}_q,\vct{x}_i\rangle)\vct{x}_i\vct{x}_i^T.
\end{align*}
Thus
\begin{align}
\label{partHess}
\left(\text{vect}(\mtx{U})\right)^T&\nabla_{\mtx{W}}^2\mathcal{L}(\vct{v},\mtx{W}) \left(\text{vect}(\mtx{U})\right)\nonumber\\&=\frac{1}{n}\sum_{i=1}^n\left(\vct{v}^T\phi(\mtx{W}\vct{x}_i)-y_i\right)\left(\sum_{p=1}^kv_p\phi''(\langle\vct{w}_p,\vct{x}_i\rangle)(\langle\vct{u}_p,\vct{x}_i\rangle)^2\right)\nonumber\\
&+\frac{1}{n}\sum_{i=1}^n \left(\sum_{p=1}^kv_p\phi'(\langle\vct{w}_q,\vct{x}_i\rangle)\langle\vct{u}_p,\vct{x}_i\rangle\right)^2\nonumber\\
=&\frac{1}{n}\sum_{i=1}^nr_i\left(\vct{x}_i^T\mtx{U}^T\mtx{D}_{\vct{v}}\mtx{D}_{\phi''(\mtx{W}\vct{x}_i)}\mtx{U}\vct{x}_i\right)+\frac{1}{n}\sum_{i=1}^n \left( \vct{v}^T\mtx{D}_{\phi'(\mtx{W}\vct{x}_i)}\mtx{U}\vct{x}_i\right)^2.
\end{align}
Using the expression for the Jacobian matrix the latter can also be written in the form
\begin{align}
\label{partHess2}
\left(\text{vect}(\mtx{U})\right)^T&\nabla_{\mtx{W}}^2\mathcal{L}(\vct{v},\mtx{W}) \left(\text{vect}(\mtx{U})\right)\nonumber\\
&=\frac{1}{n}\sum_{i=1}^n\langle \mtx{U}^T\mtx{D}_{\vct{v}}\mtx{D}_{\phi''(\mtx{W}\vct{x}_i)}\mtx{U}, r_i\vct{x}_i\vct{x}_i^T\rangle+\frac{1}{n}\text{vect}(\mtx{U})^T\mtx{J}\mtx{J}^T\text{vect}(\mtx{U}).
\end{align}
Finally, note that
\begin{align}
\label{Hessianv}
\nabla_{\vct{v}}^2\mathcal{L}(\vct{v},\mtx{W})=\phi\left(\mtx{W}\mtx{X}\right) \phi\left(\mtx{W}\mtx{X}\right)^T.
\end{align}

\subsection{Useful identities involving matrix products}
In this section we gather a few preliminary results regarding matrix products that will be useful throughout our proofs. Most of these identities are trivial and we thus skip their proof. The only exception is Lemma \ref{Pre4} which is proved in Appendix \ref{pfprem}.

\begin{lemma}\label{Pre1}
For two vectors $\ub\in\R^{m\times 1}$ and $\vb\in \R^{n\times 1}$, and a matrix $\M\in\R^{k\times n}$, we have
\[\ub\otimes \M\vb = \DM (\ub\otimes \vb)\,,\]
where $\DM\in \R^{mk\times mn}$ is the block diagonal matrix with $m$ copies of $\M$ in the diagonal blocks.
\end{lemma}

\begin{lemma}\label{Pre2}
For vectors $\vct{a}\in\R^{m\times 1}$ and $\vct{b}\in \R^{n\times 1}$, and a matrix $\M\in\R^{n\times m}$, we have 
\[(\vct{a}\otimes \vct{b})^T \vec(\M) = \vct{b}^T \M \vct{a}\,,\]
where $\vec(M)$ denotes the vectorization of the matrix $\M$ formed by stacking the columns of $\M$ into a single column vector.
\end{lemma}
\begin{lemma}\label{Pre4}
Consider an arbitrary matrix $\A\in \R^{k\times d}$ and an arbitrary vector $\vb \in \R^k$ and define $\Dv:=\diag(\vb)$. Then the following identity holds
\begin{align}
\|\vb\|_\infty \sigma_{\min}(\A) \le \sigma_{\max}(\Dv\A)\,.
\end{align}
\end{lemma}

\begin{lemma}\label{Pre5}
For any two matrices $\A$ and $\B$, we have
\begin{align}
(\A\ast \B)^T (\A\ast \B) = (\A^T \A) \circ (\B^T\B)\,.
\end{align}
\end{lemma}

\subsection{Useful probabilistic identities involving random vectors and matrices}
In this Section we gather some useful probabilistic identities that will be used throughout our proofs. We defer the proof of these results to Appendix \ref{pfprem}. The first result, proven in Appendix \ref{prfPre3}, relates the tail of a random variable to a proper Orlicz norm.
\begin{lemma}\label{Pre3}
Suppose that a random variable $Y$ satisfies the following tail bound
\begin{align}
\P(|Y|\ge t) \le 2e^{-c\min(t^2/A, t/B)}\,.
\end{align}
Then, $\|Y\|_{\psi_1} \le 9\max(\sqrt{A/(2c)},B/c)$.
\end{lemma}
The second result concerns the minimum eigenvalue of the Khatrio-Rao product of a generic matrix with itself.
\begin{lemma}\label{XkrX} 
For almost every data input matrix $\mtx{X}\in \R^{d\times n}$, as long as 
$d\le n\le d(d+1)/2$, the following holds
\begin{align*}
\sigma_{\min}(\mtx{X}\ast\mtx{X})>0.
\end{align*}
\end{lemma}
We refer to Appendix~\ref{app:XkrX} for the proof of Lemma~\ref{XkrX}.

Finally, we state a Lemma based on \cite{WF, soltanolkotabi2014algorithms, soltanolkotabi2017structured} that allows us to lower bound the loss function $\mathcal{L}(\vct{v},\mtx{W})$ in terms of how close $(\vct{v},\mtx{W})$ is to $(\vct{v}^*,\mtx{W}^*)$. We prove this Lemma in Appendix \ref{PRlemmapf}.
\begin{lemma}\label{PRlemma} Let $\mtx{A}\in\R^{d\times d}$ be a symmetric positive semidefinite matrix. Also for $i=1,2,\ldots,n$, let $\vct{x}_i$ be i.i.d.~random vectors distributed as $\mathcal{N}(\vct{0},\mtx{I}_d)$. Furthermore, assume 
\begin{align*}
n\ge c(\delta)d \log d,
\end{align*}
with $c(\delta)$ a constant only depending on $\delta$. Then for all matrices $\mtx{M}\in\R^{d\times d}$ we have
\begin{align*}
\nucnorm{\frac{1}{n}\sum_{i=1}^n \left(\vct{x}_i^T\mtx{A}\vct{x}_i\right)\vct{x}_i\vct{x}_i^T-\left(2\mtx{A}+\emph{trace}(\mtx{A})\mtx{I}\right)}\le\delta\cdot\nucnorm{\mtx{A}},
\end{align*}
holds with probability at least $1-10de^{-\gamma d}-8/d$. Here, $\nucnorm{\cdot}$ denotes the nuclear norm i.e.~sum of singular values of the input matrix.
\end{lemma}

\section{Proof of global landscape results}
\label{proofs}
In this section we prove the two global theorems. 
\subsection{Derivative calculations for quadratic activations}
We begin by gathering the derivatives of the loss function for quadratic activations. We note that for quadratic activations the loss function \eqref{landscape} as a function of $\mtx{W}$ takes the form
\begin{align*}
\mathcal{L}(\mtx{W})=&\frac{1}{2n}\sum_{i=1}^n\left(y_i-\sum_{\ell=1}^k\vct{v}_\ell\abs{\langle\mtx{w}_\ell,\vct{x}_i\rangle}^2\right)^2,\\
=&\frac{1}{2n}\sum_{i=1}^n\left(\vct{x}_i^T\mtx{W}^T\diag(\vct{v})\mtx{W}\vct{x}_i-y_i\right)^2,\\
:=&\frac{1}{2n}\sum_{i=1}^nr_i^2.
\end{align*}
Using \eqref{gradmat} we have
\begin{align}
\label{gradmat2}
\nabla_{\mtx{W}} \mathcal{L}(\vct{v},\mtx{W})=2\mtx{D}_{\vct{v}}\mtx{W}\left(\frac{1}{n}\sum_{i=1}^nr_i\vct{x}_i\vct{x}_i^T\right).
\end{align}
Using \eqref{myJ} the Jacobian matrix is given by
\begin{align}
\label{myJ2}
\mtx{J}=2\mtx{D}_{\vct{v}}(\mtx{W}\mtx{X})\ast\mtx{X}.
\end{align}
Using this Jacobian matrix we can write the vectorized version of the gradient, i.e.~gradient with respect to vect$(\mtx{W})$ as
\begin{align}
\label{gradvec2}
\nabla_{\text{vect}(\mtx{W})} \mathcal{L}(\vct{v},\mtx{W})=\frac{1}{n} \J \r.
\end{align}
Also \eqref{partHess2} specialized to quadratic activations results in the partial Hessian
\begin{align}
\label{partHess3}
\left(\text{vect}(\mtx{U})\right)^T\nabla_{\mtx{W}}^2\mathcal{L}(\vct{v},\mtx{W}) \left(\text{vect}(\mtx{U})\right)=\frac{2}{n}\sum_{i=1}^nr_i\left(\vct{x}_i^T\mtx{U}^T\mtx{D}_{\vct{v}}\mtx{U}\vct{x}_i\right)+\frac{2}{n}\sum_{i=1}^n \left( \vct{x}_i^T\mtx{W}^T\mtx{D}_{\vct{v}}\mtx{U}\vct{x}_i\right)^2.
\end{align}
\subsection{Proof of Theorem \ref{globthm}}
We begin by proving a simple lemma.
\begin{lemma}\label{globalMlem}
Any point $\widetilde{\mtx{W}}\in\R^{k\times d}$ obeying
\begin{align}
\label{gradcondW}
\frac{1}{n}\sum_{i=1}^n\left(\vct{x}_i^T\widetilde{\mtx{W}}^T\diag(\vct{v})\widetilde{\mtx{W}}\vct{x}_i-y_i\right)\vct{x}_i\vct{x}_i^T=0,
\end{align}
is a global optimum of the loss function
\begin{align*}
\mathcal{L}(\mtx{W})=\frac{1}{2n}\sum_{i=1}^n\left(\vct{x}_i^T\mtx{W}^T\diag(\vct{v})\mtx{W}\vct{x}_i-y_i\right)^2.
\end{align*}
\end{lemma}
\begin{proof}
Consider the optimization problem
\begin{align*}
\underset{\mtx{M}\in\R^{d\times d}}{\min}\text{ }f(\mtx{M}):=\frac{1}{2n}\sum_{i=1}^n\left(\vct{x}_i^T\mtx{M}\vct{x}_i-y_i\right)^2.
\end{align*}
The objective value is convex in terms of $\mtx{M}$ (in fact its a quadratic). Therefore, at any point where the gradient with respect to $\widetilde{\mtx{M}}$ is zero i.e.~$\nabla f(\widetilde{\mtx{M}})=0$, that point must be a global optimum of $f$. More specifically, any point $\widetilde{\mtx{M}}$ obeying
\begin{align}
\label{gradM}
\frac{1}{n}\sum_{i=1}^n\left(\vct{x}_i^T\widetilde{\mtx{M}}\vct{x}_i-y_i\right)\vct{x}_i\vct{x}_i^T=0,
\end{align}
is a global optimum of $f(\mtx{M})$. That is, for any arbitrary matrix $\mtx{M}\in\R^{d\times d}$ and any point $\widetilde{\mtx{M}}\in\R^{d\times d}$ obeying \eqref{gradM}  we have
\begin{align}
\label{globM}
f(\mtx{M})\ge f(\widetilde{\mtx{M}}).
\end{align}
Now note that if \eqref{gradcondW} holds for some $\widetilde{\mtx{W}}$, then \eqref{gradM} holds with $\widetilde{\mtx{M}}=\widetilde{\mtx{W}}^T$diag($\vct{v})\widetilde{\mtx{W}}$. Thus for any $\mtx{W}\in\R^{k\times d}$, using \eqref{globM} with $\mtx{M}=\mtx{W}^T$diag$(\vct{v})\mtx{W}$ we have
\begin{align*}
\mathcal{L}(\mtx{W})=f\left(\mtx{W}^T\diag(\vct{v})\mtx{W}\right)\ge f(\widetilde{\mtx{M}})=\mathcal{L}(\widetilde{\mtx{W}}).
\end{align*}
Thus any $\widetilde{\mtx{W}}$ obeying \eqref{gradcondW} is a global optima of $\mathcal{L}(\mtx{W})$, concluding the proof.
\end{proof}
\subsubsection{Proof of no spurious local minima and strict saddle property}
We will prove the first two conclusions of the theorem (i.e.~no spurious local optima and saddles have a direction of negative curvature) simultaneously. To show this note that since the function $\mathcal{L}$ is twice differentiable all local optima or saddles that do not have a direction of negative curvature obey
\begin{align}
\label{secord}
\nabla \mathcal{L}(\mtx{W})=\mtx{0}\quad\text{and}\quad \nabla^2\mathcal{L}(\mtx{W})\succeq \mtx{0}.
\end{align}
We will prove that all points obeying \eqref{secord} are a global optima of $\mathcal{L}(\mtx{W})$. To this aim let us define the sets
\begin{align*}
\mathcal{S}_{+}=\{i: \vct{v}_i>0\}\quad\text{and}\quad\mathcal{S}_{-}=\{i: \vct{v}_i<0\},
\end{align*}
i.e.~the indices of the positive and negative entries of $\vct{v}$. Now note that two cases are possible
\begin{itemize}
\item \textbf{Case I:} At least one of the sub-matrices $\left(\mtx{D}_{\vct{v}}\mtx{W}\right)_{\mathcal{S}_{+}}$ and  $\left(\mtx{D}_{\vct{v}}\mtx{W}\right)_{\mathcal{S}_{-}}$ is not rank deficient. That is,
\begin{align}
\label{fullrank}
\text{rank}\left(\mtx{D}_{\vct{v}}\mtx{W}\right)_{\mathcal{S}_{+}}=d\quad\text{or}\quad \text{rank}\left(\mtx{D}_{\vct{v}}\mtx{W}\right)_{\mathcal{S}_{-}}=d.
\end{align}
\item \textbf{Case II:} Both of the sub-matrices $\left(\mtx{D}_{\vct{v}}\mtx{W}\right)_{\mathcal{S}_{+}}$ and  $\left(\mtx{D}_{\vct{v}}\mtx{W}\right)_{\mathcal{S}_{-}}$ are rank deficient. That is,
\begin{align}
\label{cond2}
\text{rank}\left(\left(\mtx{D}_{\vct{v}}\mtx{W}\right)_{\mathcal{S}_{+}}\right)<d\quad\text{or}\quad \text{rank}\left(\left(\mtx{D}_{\vct{v}}\mtx{W}\right)_{\mathcal{S}_{-}}\right)<d.
\end{align}
\end{itemize}
In both cases we will show that for any $\mtx{W}$ obeying \eqref{secord}, we have
\begin{align*}
\sum_{i=1}^n r_i\vct{x}_i\vct{x}_i^T=\frac{1}{n}\sum_{i=1}^n\left(\vct{x}_i^T\mtx{W}^T\diag(\vct{v})\mtx{W}\vct{x}_i-y_i\right)\vct{x}_i\vct{x}_i^T=0.
\end{align*}
The latter together with Lemma \ref{globalMlem} immediately implies that any $\mtx{W}$ obeying \eqref{secord} is a global optimum, proving that there are no spurious local minima and no saddles which do not have a direction of strict negative curvature. We now proceed to show that indeed $\sum_{i=1}^n r_i\vct{x}_i\vct{x}_i^T=0$ holds in both cases mentioned above.

\noindent\textbf{Case I: } Note that \eqref{fullrank} together with the fact that $k\ge d$ implies that the matrix $\mtx{D}_{\vct{v}}\mtx{W}$ has a left inverse i.e.~there exists a matrix $\mtx{M}\in\R^{d\times k}$ such that $\mtx{M}\mtx{D}_{\vct{v}}\mtx{W}=\mtx{I}$. Furthermore, note that by \eqref{gradmat2}, $\nabla \mathcal{L}(\mtx{W})=\mtx{0}$ is equivalent to
\begin{align*}
\mtx{D}_{\vct{v}}\mtx{W}\left(\frac{1}{n}\sum_{i=1}^nr_i\vct{x}_i\vct{x}_i^T\right)=0.
\end{align*}
Multiplying both sides of the above equality on the left by $\mtx{M}$ we conclude that
\begin{align*}
\frac{1}{n}\sum_{i=1}^nr_i\vct{x}_i\vct{x}_i^T=0,
\end{align*}
concluding the proof in case I.

\noindent\textbf{Case II: } First, note that for any $\mtx{W}$ obeying $\nabla^2\mathcal{L}(\mtx{W})\succeq \mtx{0}$ and any $\mtx{U}\in\R^{k\times d}$ by \eqref{partHess2} we have
\begin{align}
\label{69}
0\le&\frac{1}{2} \left(\text{vect}(\mtx{U})\right)^T\nabla_{\mtx{W}}^2\mathcal{L}(\vct{v},\mtx{W}) \left(\text{vect}(\mtx{U})\right),\nonumber\\
=&\frac{1}{n}\sum_{i=1}^nr_i\left(\vct{x}_i^T\mtx{U}^T\mtx{D}_{\vct{v}}\mtx{U}\vct{x}_i\right)+\frac{1}{n}\sum_{i=1}^n \left( \vct{x}_i^T\mtx{W}^T\mtx{D}_{\vct{v}}\mtx{U}\vct{x}_i\right)^2.
\end{align}
Let us choose $\mtx{U}$ of the form $\mtx{U}=\vct{a}\vct{b}^T$ with $\vct{a}\in\R^k$ obeying $\mtx{W}^T\mtx{D}_{\vct{v}}\vct{a}=0$ (i.e.~$\mtx{D}_{\vct{v}}\vct{a}\in$Null($\mtx{W}^T$)) and $\vct{b}\in\R^d$ an arbitrary vector. Plugging such a $\mtx{U}$ into \eqref{69} we conclude that
\begin{align}
\label{ineq}
0\le&\frac{1}{n}\sum_{i=1}^nr_i\left(\vct{x}_i^T\mtx{U}^T\mtx{D}_{\vct{v}}\mtx{U}\vct{x}_i\right)+\frac{1}{n}\sum_{i=1}^n \left( \vct{x}_i^T\mtx{W}^T\mtx{D}_{\vct{v}}\mtx{U}\vct{x}_i\right)^2,\nonumber\\
=&(\vct{a}^T\mtx{D}_{\vct{v}}\vct{a})\vct{b}^T\left(\sum_{i=1}^nr_i\vct{x}_i\vct{x}_i^T\right)\vct{b}+\frac{1}{n}\sum_{i=1}^n \left( \vct{x}_i^T\mtx{W}^T\mtx{D}_{\vct{v}}\vct{a}\vct{b}^T\vct{x}_i\right)^2,\nonumber\\
=&\left(\vct{a}^T\mtx{D}_{\vct{v}}\vct{a}\right)\vct{b}^T\left(\sum_{i=1}^nr_i\vct{x}_i\vct{x}_i^T\right)\vct{b},
\end{align}
holds for all $\vct{b}\in\R^d$ and all $\vct{a}\in\R^k$ obeying $\mtx{W}^T\mtx{D}_{\vct{v}}\vct{a}=0$.
 Next, we note that by the assumptions of Theorem \ref{globthm} $\abs{\mathcal{S}_{+}}\ge d$ and $\abs{\mathcal{S}_{+}}\ge d$. This together with \eqref{cond2} implies that there exists non-zero vectors $\vct{u},\vct{w}\in\R^k$ with non-zero entries belonging to $\mathcal{S}_{+}$ and $\mathcal{S}_{-}$ such that
\begin{align*}
\mtx{W}^T\mtx{D}_{\vct{v}}\vct{u}=\mtx{W}^T\mtx{D}_{\vct{v}}\vct{w}=0.
\end{align*}
Furthermore, 
\begin{align*}
\vct{u}^T\mtx{D}_{\vct{v}}\vct{u}=\sum_{i\in\mathcal{S}_{+}} v_iu_i^2>0.
\end{align*}
Thus using \eqref{ineq} with $\vct{a}=\vct{u}$ we conclude that for all $\vct{b}\in\R^d$ we have
\begin{align}
\label{myeq1}
\vct{b}^T\left(\sum_{i=1}^nr_i\vct{x}_i\vct{x}_i^T\right)\vct{b}\ge 0.
\end{align}
Similarly, 
\begin{align*}
\vct{w}^T\mtx{D}_{\vct{v}}\vct{w}=\sum_{i\in\mathcal{S}_{-}} v_iw_i^2< 0.
\end{align*}
Thus using \eqref{ineq} with $\vct{a}=\vct{w}$ we conclude that for all $\vct{b}\in\R^d$ we have
\begin{align}
\label{myeq2}
\vct{b}^T\left(\sum_{i=1}^nr_i\vct{x}_i\vct{x}_i^T\right)\vct{b}\le 0.
\end{align}
Equations \eqref{myeq1} and \eqref{myeq2} together imply that 
\begin{align*}
\sum_{i=1}^nr_i\vct{x}_i\vct{x}_i^T=0,
\end{align*}
concluding the proof in case II.
\subsubsection{Proof of zero training error}
Note that in the previous section we proved that all global optima of $\mathcal{L}(\mtx{W})$ obey
\begin{align*}
\sum_{i=1}^nr_i\vct{x}_i\vct{x}_i^T=0.
\end{align*}
The latter identity can be rewritten in the form
\begin{align}
\label{tempglob}
\left(\mtx{X}\ast\mtx{X}\right)\vct{r}=0.
\end{align}
Using Lemma \ref{XkrX}, $\sigma_{\min}(\mtx{X}\ast\mtx{X})>0$ 
for almost every data input matrix $\mtx{X}\in \R^{d\times n}$,
as long as $n\le c d^2$ with $c$ a fixed numerical constant. This combined with \eqref{tempglob} implies that $\vct{r}=\vct{0}$, completing the proof for zero training error.


\subsection{Proof of Theorem \ref{gthmrand}}
First note that since we assume both $\vct{v}$ and $\vct{v}^*$ have the same sign pattern we can without loss of generality assume both $\vct{v}$ and $\vct{v}^*$ have non-negative entries. We will first prove that there are no spurious local minima, and all saddles have a direction of negative curvature, and all global optima have zero loss value. We then proceed to show all approximate local minima have small objective value in Section \ref{pfneg}.
\subsubsection{Proof of no spurious local minima, zero training error, and strict saddle}
We begin by noting that the function $\mathcal{L}$ is twice differentiable and thus all local optima or saddle points that do not have a direction of negative curvature obey
\begin{align}
\label{secordp}
\nabla \mathcal{L}(\mtx{W})=\mtx{0}\quad\text{and}\quad \nabla^2\mathcal{L}(\mtx{W})\succeq\mtx{0}.
\end{align}
Note that $\nabla \mathcal{L}(\mtx{W})=\mtx{0}$ implies that
\begin{align}
\label{tmppp2}
\mtx{D}_{\vct{v}}\mtx{W}\left(\sum_{i=1}^n r_i\vct{x}_i\vct{x}_i^T\right)=\mtx{0}.
\end{align}
Define the set
\begin{align*}
\mathcal{S}=\{i: v_i>0\}.
\end{align*}
Since we assume $k\ge d$, two cases can occur.
\begin{itemize}
\item \textbf{Case I:} rank($\left(\mtx{D}_{\vct{v}}\mtx{W}\right)_{\mathcal{S}})=d$.
\item \textbf{Case II:} rank($\left(\mtx{D}_{\vct{v}}\mtx{W}\right)_{\mathcal{S}})<d$.
\end{itemize}
In both cases we will show that any $\mtx{W}$ obeying \eqref{secord}, must be a global optimum. 

\noindent\textbf{Case I: }rank($\left(\mtx{D}_{\vct{v}}\mtx{W}\right)_{\mathcal{S}})=d$.\\
In this case the matrix $\mtx{D}_{\vct{v}}\mtx{W}$ has a left inverse, i.e.~there exists a $\mtx{M}\in\R^{d\times k}$ such that $\mtx{M}\mtx{D}_{\vct{v}}\mtx{W}=\mtx{I}$. Multiplying both sides of \eqref{tmppp2} on the left by $\mtx{M}$ yields 
\begin{align}
\label{tmpthm22}
\sum_{i=1}^n r_i\vct{x}_i\vct{x}_i^T=\left(\mtx{X}\ast\mtx{X}\right)\vct{r}=\frac{1}{n}\sum_{i=1}^n\left(\vct{x}_i^T\mtx{W}^T\diag(\vct{v})\mtx{W}\vct{x}_i-y_i\right)\vct{x}_i\vct{x}_i^T=0.
\end{align}
The latter together with Lemma \ref{globalMlem} immediately implies that any $\mtx{W}$ obeying \eqref{secordp} and the condition of Case I is a global optimum, proving that there are no spurious local minima. Furthermore using Lemma \ref{XkrX}, $\sigma_{\min}(\mtx{X}\ast\mtx{X})>0$ 
as long as $n\le c d^2$ with 
$c$ a fixed numerical constant. This combined with \eqref{tmpthm22} implies that $\vct{r}=\vct{0}$, completing the proof for zero training error in this case. 

\noindent\textbf{Case II: }rank($\left(\mtx{D}_{\vct{v}}\mtx{W}\right)_{\mathcal{S}})<d$.\\
First, note that for any $\mtx{W}$ obeying $\nabla^2\mathcal{L}(\mtx{W})\succeq \mtx{0}$ and any $\mtx{U}\in\R^{k\times d}$ by \eqref{partHess2} we have
\begin{align}
\label{692}
0\le&\frac{1}{2} \left(\text{vect}(\mtx{U})\right)^T\nabla_{\mtx{W}}^2\mathcal{L}(\vct{v},\mtx{W}) \left(\text{vect}(\mtx{U})\right),\nonumber\\
=&\frac{1}{n}\sum_{i=1}^nr_i\left(\vct{x}_i^T\mtx{U}^T\mtx{D}_{\vct{v}}\mtx{U}\vct{x}_i\right)+\frac{1}{n}\sum_{i=1}^n \left( \vct{x}_i^T\mtx{W}^T\mtx{D}_{\vct{v}}\mtx{U}\vct{x}_i\right)^2.
\end{align}
In this case since $\left(\mathcal{D}_{\vct{v}}\right)_{\mathcal{S}}$ is invertible and $\left(\mtx{D}_{\vct{v}}\mtx{W}\right)_{\mathcal{S}}$ is rank deficient there exists a nonzero vector $\vct{a}\in\R^k$ supported on $\mathcal{S}$ obeying $\mtx{W}^T\mtx{D}_{\vct{v}}\vct{a}=0$ (i.e.~$\mtx{D}_{\vct{v}}\vct{a}\in$Null($\mtx{W}^T$)). Thus plugging $\mtx{U}$ of the form $\mtx{U}=\vct{a}\vct{b}^T$ with $\vct{b}\in\R^d$ in \eqref{692} we conclude that
\begin{align}
\label{ineqthm22}
0\le&\frac{1}{n}\sum_{i=1}^nr_i\left(\vct{x}_i^T\mtx{U}^T\mtx{D}_{\vct{v}}\mtx{U}\vct{x}_i\right)+\frac{1}{n}\sum_{i=1}^n \left( \vct{x}_i^T\mtx{W}^T\mtx{D}_{\vct{v}}\mtx{U}\vct{x}_i\right)^2,\nonumber\\
=&(\vct{a}^T\mtx{D}_{\vct{v}}\vct{a})\vct{b}^T\left(\sum_{i=1}^nr_i\vct{x}_i\vct{x}_i^T\right)\vct{b}+\frac{1}{n}\sum_{i=1}^n \left( \vct{x}_i^T\mtx{W}^T\mtx{D}_{\vct{v}}\vct{a}\vct{b}^T\vct{x}_i\right)^2,\nonumber\\
=&\left(\vct{a}^T\mtx{D}_{\vct{v}}\vct{a}\right)\vct{b}^T\left(\sum_{i=1}^nr_i\vct{x}_i\vct{x}_i^T\right)\vct{b},
\end{align}
holds for all $\vct{b}\in\R^d$. Also note that
\begin{align*}
\vct{a}^T\mtx{D}_{\vct{v}}\vct{a}=\sum_{i\in\mathcal{S}} v_ia_i^2>0.
\end{align*}
The latter together with \eqref{ineqthm22} implies that 
\begin{align}
\label{rxxpos}
\sum_{i=1}^n r_i\vct{x}_i\vct{x}_i^T\succeq \mtx{0}.
\end{align}
Thus for any $\mtx{W}$ obeying \eqref{secordp} and the condition of Case II we have
\begin{align*}
\mathcal{L}(\mtx{W})&=\frac{1}{2n}\langle \mtx{W}^T\mtx{D}_{\vct{v}}\mtx{W}-{\mtx{W}^*}^T\mtx{D}_{\vct{v}^*}\mtx{W}^*,\sum_{i=1}^n r_i\vct{x}_i\vct{x}_i^T\rangle,\\
&\overset{(a)}{=}-\frac{1}{2n}\langle{\mtx{W}^*}^T\mtx{D}_{\vct{v}^*}\mtx{W}^*,\sum_{i=1}^n r_i\vct{x}_i\vct{x}_i^T\rangle,\\
&\overset{(b)}{\le} 0.
\end{align*}
where (a) follows from \eqref{tmppp2} and (b) follows from \eqref{rxxpos} combined with the fact that $\vct{v}^*$ has nonnegative entries. Thus, $\mathcal{L}(\mtx{W})=0$ and therefore such a $\mtx{W}$ must be a global optimum. This completes the proof in Case II.

\subsubsection{Bounding the objective value for approximate local minima}
\label{pfneg}
Assume $\mtx{W}$ is an approximate local minima. That is it obeys
\begin{align}
\label{secordp3}
\fronorm{\nabla \mathcal{L}(\mtx{W})}\le\epsilon_g\quad\text{and}\quad \nabla^2\mathcal{L}(\mtx{W}_s)\succeq -\epsilon_{H}\mtx{I}.
\end{align}
For such a point by Cauchy-Schwarz we have
\begin{align}
\label{trivial}
\langle\mtx{W},\nabla \mathcal{L}(\mtx{W})\rangle\le \epsilon_g\fronorm{\mtx{W}}.
\end{align}
Furthermore,  for such a point using \eqref{partHess2}
\begin{align}
\label{6923}
-\epsilon_{H}\fronorm{\mtx{U}}^2\le&\frac{1}{2} \left(\text{vect}(\mtx{U})\right)^T\nabla_{\mtx{W}}^2\mathcal{L}(\vct{v},\mtx{W}) \left(\text{vect}(\mtx{U})\right),\nonumber\\
=&\frac{1}{n}\sum_{i=1}^nr_i\left(\vct{x}_i^T\mtx{U}^T\mtx{D}_{\vct{v}}\mtx{U}\vct{x}_i\right)+\frac{1}{n}\sum_{i=1}^n \left( \vct{x}_i^T\mtx{W}^T\mtx{D}_{\vct{v}}\mtx{U}\vct{x}_i\right)^2,
\end{align}
holds for all $\mtx{U}\in\R^{k\times d}$. Since $\vct{v}=\nu\vct{1}$ there exists a unit norm, non-zero vector $\vct{a}\in\R^k$ supported on $\mathcal{S}$ such that $\mtx{W}^T\mtx{D}_{\vct{v}}\vct{a}=0$. We now set $\mtx{U}=\vct{a}\vct{b}^T$ in \eqref{6923} with $\vct{b}\in\R^d$. We conclude that 
\begin{align}
\label{tmpsecthm}
\left(\vct{a}^T\mtx{D}_{\vct{v}}\vct{a}\right)\left(\vct{b}^T\left(\frac{1}{n}\sum_{i=1}^n r_i\vct{x}_i\vct{x}_i^T\right)\vct{b}\right)\ge -\epsilon_{H}\twonorm{\vct{b}}^2.
\end{align}
holds for all $\vct{b}\in\R^d$. Also note that
\begin{align*}
\vct{a}^T\mtx{D}_{\vct{v}}\vct{a}=\nu\twonorm{\vct{a}}^2=\nu.
\end{align*}
Thus, \eqref{tmpsecthm} implies that
\begin{align}
\label{condrxx}
\frac{1}{n}\sum_{i=1}^n r_i\vct{x}_i\vct{x}_i^T\succeq -\frac{\epsilon_{H}}{\nu}\mtx{I}.
\end{align}
Now note that the gradient is equal to
\begin{align}
\label{gradthm22}
\nabla\mathcal{L}(\mtx{W})=\mtx{D}_{\vct{v}}\mtx{W}\left(\sum_{i=1}^n r_i\vct{x}_i\vct{x}._i^T\right).
\end{align}
Thus using \eqref{gradthm22}, \eqref{trivial}, and \eqref{condrxx} we conclude that for any point $\mtx{W}$ obeying \eqref{secordp3} we have
\begin{align}
\label{tmpbndLW}
\mathcal{L}(\mtx{W})&=\frac{1}{2n}\langle \mtx{W}^T\mtx{D}_{\vct{v}}\mtx{W}-{\mtx{W}^*}^T\mtx{D}_{\vct{v}^*}\mtx{W}^*,\sum_{i=1}^n r_i\vct{x}_i\vct{x}_i^T\rangle,\nonumber\\
&\overset{(a)}{=}\langle\mtx{W},\nabla\mathcal{L}(\mtx{W})\rangle-\frac{1}{2n}\langle{\mtx{W}^*}^T\mtx{D}_{\vct{v}^*}\mtx{W}^*,\sum_{i=1}^n r_i\vct{x}_i\vct{x}_i^T\rangle,\nonumber\\
&\overset{(b)}{\le} \epsilon_g\fronorm{\mtx{W}}-\frac{1}{2}\langle{\mtx{W}^*}^T\mtx{D}_{\vct{v}^*}\mtx{W}^*,\frac{1}{n}\sum_{i=1}^n r_i\vct{x}_i\vct{x}_i^T\rangle,\nonumber\\
&\overset{(c)}{\le}\epsilon_g\fronorm{\mtx{W}}+\frac{\epsilon_{H}}{2\nu}\nucnorm{{\mtx{W}^*}^T\mtx{D}_{\vct{v}^*}\mtx{W}^*}.
\end{align}
Here, (a) follows from \eqref{gradthm22}, (b) from \eqref{trivial} and (c) from Holder's inequality combined with \eqref{condrxx}.

We proceed by showing that for a point $\mtx{W}$ obeying $\fronorm{\nabla \mathcal{L}(\mtx{W})}\le\epsilon_g$ it's Frobenius norm norm is also sufficiently small. To this aim note that for such a point by Cauchy-Schwarz have
\begin{align}
\label{simple}
\langle\nabla \mathcal{L}(\mtx{W}),\mtx{W}\rangle\le \Delta\fronorm{\mtx{W}}.
\end{align}
Also
\begin{align}
\label{myineqgg}
\langle\nabla \mathcal{L}(\mtx{W}),\mtx{W}\rangle=&\frac{1}{n}\langle\mtx{D}_{\vct{v}}\mtx{W}\sum_i \left(\vct{x}_i^T\left(\mtx{W}^T\mtx{D}_{\vct{v}}\mtx{W}-{\mtx{W}^*}^T\mtx{D}_{\vct{v}^*}\mtx{W}^*\right)\vct{x}_i\right)\vct{x}_i\vct{x}_i^T,\mtx{W}\rangle\nonumber\\
=&\frac{1}{n}\sum_{i=1}^n \left(\vct{x}_i^T\mtx{W}^T\mtx{D}_{\vct{v}}\mtx{W}\vct{x}_i\right)^2-\frac{1}{n}\sum_{i=1}^n\left(\vct{x}_i^T\mtx{W}^T\mtx{D}_{\vct{v}}\mtx{W}\vct{x}_i\right)\left(\vct{x}_i^T{\mtx{W}^*}^T\mtx{D}_{\vct{v}^*}\mtx{W}^*\vct{x}_i\right)\nonumber\\
\ge&\left(\frac{1}{n}\sum_{i=1}^n\vct{x}_i^T\mtx{W}^T\mtx{D}_{\vct{v}}\mtx{W}\vct{x}_i\right)^2-\frac{1}{n}\sum_{i=1}^n\left(\vct{x}_i^T\mtx{W}^T\mtx{D}_{\vct{v}}\mtx{W}\vct{x}_i\right)\left(\vct{x}_i^T{\mtx{W}^*}^T\mtx{D}_{\vct{v}^*}\mtx{W}^*\vct{x}_i\right).
\end{align}
Using Holder's inequality for matrices we have
\begin{align*}
\frac{1}{n}\sum_{i=1}^n\vct{x}_i^T\mtx{W}^T\mtx{D}_{\vct{v}}\mtx{W}\vct{x}_i-\text{trace}\left(\mtx{W}^T\mtx{D}_{\vct{v}}\mtx{W}\right)=&\langle \mtx{W}^T\mtx{D}_{\vct{v}}\mtx{W},\frac{1}{n}\sum_{i=1}^n\vct{x}_i\vct{x}_i^T-\mtx{I}\rangle\\
\le&\text{trace}(\mtx{W}^T\mtx{D}_{\vct{v}}\mtx{W})\opnorm{\frac{1}{n}\sum_{i=1}^n\vct{x}_i\vct{x}_i^T-\mtx{I}}\\
=&\text{trace}(\mtx{W}^T\mtx{D}_{\vct{v}}\mtx{W})\opnorm{\frac{1}{n}\sum_{i=1}^n\vct{x}_i\vct{x}_i^T-\mtx{I}}.
\end{align*}
Thus by standard concentration of sample covariance matrices we conclude that for any $\delta>0$, as long as $n\ge \frac{c}{\delta^2} d$ for a fixed numerical constant $c$, then for all $\mtx{W}$
\begin{align*}
\abs{\frac{1}{n}\sum_{i=1}^n\vct{x}_i^T\mtx{W}^T\mtx{D}_{\vct{v}}\mtx{W}\vct{x}_i-\text{trace}\left(\mtx{W}^T\mtx{D}_{\vct{v}}\mtx{W}\right)}\le \delta\cdot\text{trace}(\mtx{W}^T\mtx{D}_{\vct{v}}\mtx{W}),
\end{align*}
holds with probability at least $1-2e^{-b d}$, for a fixed constant $b > 0$. Thus with high probability for all $\mtx{W}$ we have
\begin{align*}
\frac{1}{n}\sum_{i=1}^n\vct{x}_i^T\mtx{W}^T\mtx{D}_{\vct{v}}\mtx{W}\vct{x}_i\ge (1-\delta)\text{trace}(\mtx{W}^T\mtx{D}_{\vct{v}}\mtx{W}).
\end{align*}
Plugging the latter into \eqref{myineqgg} we conclude that
\begin{align}
\label{myineq2gg}
\langle\nabla \mathcal{L}(\mtx{W}),\mtx{W} \rangle \ge (1-\delta)^2\left(\text{trace}(\mtx{W}^T\mtx{D}_{\vct{v}}\mtx{W})\right)^2-\frac{1}{n}\sum_{i=1}^n\left(\vct{x}_i^T\mtx{W}^T\mtx{D}_{\vct{v}}\mtx{W}\vct{x}_i\right)\left(\vct{x}_i^T{\mtx{W}^*}^T\mtx{D}_{\vct{v}^*}\mtx{W}^*\vct{x}_i\right).
\end{align}
We now apply Lemma \ref{PRlemma} to conclude that as long as $n\ge c(\delta) d\log d$ 
\begin{multline*}
\opnorm{\frac{1}{n}\sum_{i=1}^n\left(\vct{x}_i^T{\mtx{W}^*}^T\mtx{D}_{\vct{v}^*}\mtx{W}^*\vct{x}_i\right)\vct{x}_i\vct{x}_i^T-\left(2{\mtx{W}^*}^T\mtx{D}_{\vct{v}^*}\mtx{W}^*+\text{trace}\left({\mtx{W}^*}^T\mtx{D}_{\vct{v}^*}\mtx{W}^*\right)\mtx{I}\right)}\\\le\delta\cdot\text{trace}\left({\mtx{W}^*}^T\mtx{D}_{\vct{v}^*}\mtx{W}^*\right),
\end{multline*}
holds with probability at least $1-10de^{-b d}-8/d$. The latter implies that with high probability for all $\mtx{W}$ we have
\begin{align*}
\Bigg|\frac{1}{n}\sum_{i=1}^n\left(\vct{x}_i^T\mtx{W}^T\mtx{D}_{\vct{v}}\mtx{W}\vct{x}_i\right)&\left(\vct{x}_i^T{\mtx{W}^*}^T\mtx{D}_{\vct{v}^*}\mtx{W}^*\vct{x}_i\right)\\
&-\left(2\langle{\mtx{W}}^T\mtx{D}_{\vct{v}}\mtx{W}  ,{\mtx{W}^*}^T\mtx{D}_{\vct{v}^*}\mtx{W}^* \rangle+\text{trace}\left({\mtx{W}^*}^T\mtx{D}_{\vct{v}^*}\mtx{W}^*\right)\text{trace}\left({\mtx{W}}^T\mtx{D}_{\vct{v}}\mtx{W}\right)\right)\bigg|\\
\le& \delta\cdot\text{trace}\left({\mtx{W}^*}^T\mtx{D}_{\vct{v}^*}\mtx{W}^*\right)\text{trace}\left({\mtx{W}}^T\mtx{D}_{\vct{v}}\mtx{W}\right).
\end{align*}
Plugging the latter into \eqref{myineq2gg} and we conclude that with high probability for all $\mtx{W}$ we have
\begin{align*}
\langle\nabla \mathcal{L}(\mtx{W}),\mtx{W} \rangle \ge& (1-\delta)^2\left(\text{trace}(\mtx{W}^T\mtx{D}_{\vct{v}}\mtx{W})\right)^2-\frac{1}{n}\sum_{i=1}^n\left(\vct{x}_i^T\mtx{W}^T\mtx{D}_{\vct{v}}\mtx{W}\vct{x}_i\right)\left(\vct{x}_i^T{\mtx{W}^*}^T\mtx{D}_{\vct{v}^*}\mtx{W}^*\vct{x}_i\right).\\
\ge&(1-\delta)^2\left(\text{trace}(\mtx{W}^T\mtx{D}_{\vct{v}}\mtx{W})\right)^2-\delta\cdot\text{trace}\left({\mtx{W}^*}^T\mtx{D}_{\vct{v}^*}\mtx{W}^*\right)\text{trace}\left({\mtx{W}}^T\mtx{D}_{\vct{v}}\mtx{W}\right)\\
&-\left(2\langle{\mtx{W}}^T\mtx{D}_{\vct{v}}\mtx{W}  ,{\mtx{W}^*}^T\mtx{D}_{\vct{v}^*}\mtx{W}^* \rangle+\text{trace}\left({\mtx{W}^*}^T\mtx{D}_{\vct{v}^*}\mtx{W}^*\right)\text{trace}\left({\mtx{W}}^T\mtx{D}_{\vct{v}}\mtx{W}\right)\right)\\
\ge&(1-\delta)^2\left(\text{trace}(\mtx{W}^T\mtx{D}_{\vct{v}}\mtx{W})\right)^2-(3+\delta)\cdot\text{trace}\left({\mtx{W}^*}^T\mtx{D}_{\vct{v}^*}\mtx{W}^*\right)\text{trace}\left({\mtx{W}}^T\mtx{D}_{\vct{v}}\mtx{W}\right).
\end{align*}
Plugging the latter into \eqref{simple} implies that
\begin{align*}
(1-\delta)^2\left(\text{trace}(\mtx{W}^T\mtx{D}_{\vct{v}}\mtx{W})\right)^2-(3+\delta)\cdot\text{trace}\left({\mtx{W}^*}^T\mtx{D}_{\vct{v}^*}\mtx{W}^*\right)\text{trace}\left({\mtx{W}}^T\mtx{D}_{\vct{v}}\mtx{W}\right)\le \epsilon_g\fronorm{\mtx{W}}.
\end{align*}
Using the assumption that $\vct{v}=\nu\vct{1}$, the latter is equivalent to
\begin{align*}
&\fronorm{\mtx{W}}\left((1-\delta)^2\nu^2\fronorm{\mtx{W}}^2-\nu(3+\delta)\text{trace}\left({\mtx{W}^*}^T\mtx{D}_{\vct{v}^*}\mtx{W}^*\right)\right)\\
&=(1-\delta)^2\nu^2\fronorm{\mtx{W}}^3-(3+\delta)\nu\fronorm{\mtx{W}}\text{trace}\left({\mtx{W}^*}^T\mtx{D}_{\vct{v}^*}\mtx{W}^*\right)\le \epsilon_g.
\end{align*}
From the latter we can conclude that for all $a>0$ we have
\begin{align*}
\fronorm{\mtx{W}}\le& \max\left(\sqrt{\frac{\epsilon_g}{a\nu^2(1-\delta)^2}+\frac{3+\delta}{\nu(1-\delta)^2}\text{trace}\left({\mtx{W}^*}^T\mtx{D}_{\vct{v}^*}\mtx{W}^*\right)}, a\right)
\end{align*}
Setting $a=\frac{\epsilon_g}{\nu(1-\delta)^2}$ and $\delta=1/2$, the latter inequality implies that
\begin{align*}
\fronorm{\mtx{W}}\le \max\left(\frac{1}{\sqrt{\nu}}\sqrt{1+14\text{trace}\left({\mtx{W}^*}^T\mtx{D}_{\vct{v}^*}\mtx{W}^*\right)},\frac{4}{\nu}\epsilon_g\right).
\end{align*}
Plugging the latter into \eqref{tmpbndLW} we conclude that for all $\mtx{W}$ obeying \eqref{secordp3}
\begin{align*}
\mathcal{L}(\mtx{W})\le&\epsilon_g\fronorm{\mtx{W}}+\frac{\epsilon_{H}}{2\nu}\nucnorm{{\mtx{W}^*}^T\mtx{D}_{\vct{v}^*}\mtx{W}^*},\\
\le&\frac{\epsilon_g}{\sqrt{\nu}}\max\left(\sqrt{1+14\nucnorm{{\mtx{W}^*}^T\mtx{D}_{\vct{v}^*}\mtx{W}^*}},4\frac{\epsilon_g}{\sqrt{\nu}}\right)+\frac{\epsilon_{H}}{2\nu}\nucnorm{{\mtx{W}^*}^T\mtx{D}_{\vct{v}^*}\mtx{W}^*},
\end{align*}
holds with high probability, concluding the proof.

\section{Proof of local convergence results}
In this section we prove Theorem \ref{localthm}. We begin by explaining a crucial component of our proof which involves bounding the spectrum of the Jacobian matrix. These results are the subject of the next section with the corresponding proofs appearing in Section \ref{proofJ}. We then utilize these intermediate results to finalize the proof of Theorem \ref{localthm} in Section \ref{pfthmloc}.

\subsection{Key lemmas on spectrum of the Jacobian Matrix}\label{sec:key-lemma}
We start with a few definitions. As defined previously in Assumption \ref{ass-phi} for $\sigma\in \R$ and $g\sim\normal(0,\sigma^2)$ we define
$\mu_\phi(\sigma):= \E[\phi'(g)]$, 
and let 
\begin{align}\label{mub}
\mub=\left(\mu_\phi(\twonorm{\vct{w}_1}),\mu_\phi(\twonorm{\vct{w}_2}),\ldots,\mu_\phi(\twonorm{\vct{w}_k})\right)\,.
\end{align}
Similarly, 
\[\gamma_\phi(\sigma):= \frac{1}{\sigma^2} \E[\phi'(g)g] \,, \]
and note that by Stein's lemma we have
\begin{align*}
\gamma_\phi(\sigma)=\frac{1}{\sigma^2}\E[\phi'(g)g]=\E[\phi''(g)]\,.
\end{align*}
Also define
\begin{align}\label{Gammab}
\Gammab=\text{diag}\left(\gamma_\phi(\twonorm{\vct{w}_1}),\gamma_\phi(\twonorm{\vct{w}_2}),\ldots,\gamma_\phi(\twonorm{\vct{w}_k})\right)\,.
\end{align}
Recall that the Jacobian matrix is given by $\J =  \Dv \phi'(\W\X) \ast \X$. To bound the spectrum of the Jacobian we first introduce a new matrix obtained by centering the $\phi'(\W\X)$ component. Precisely, we let
\begin{align}\label{mycenter}
\tJ =  \X \ast \Dv(\phi'(\W\X) - \E[\phi'(\W\X)])\,.
\end{align}

Our first result, proven in in Appendix~\ref{app:J-Jt}, relates the spectrum of $\J$ to that of $\tJ$. Note that in case  of $\mu(\sigma) = 0$ everywhere (See Assumption~\ref{ass-phi}), then $\E[\phi'(\W\X)] = 0$ and $\J = \bJ$, up to a permutation of the rows. Therefore, this step becomes superfluous in this case.
\begin{proposition}\label{pro:J-Jt}
Let $\phi:\R\rightarrow \R$ be a nonlinear activation obeying $\abs{\phi''}\le L$. Assume the inputs $\vct{x}_i\in\R^d$ are distributed i.i.d.~$\mathcal{N}(\vct{0},\mtx{I}_d)$ for $i=1,2,\ldots,n$. 
\begin{itemize}
\item There exists an absolute constant $c>0$ such that for  $n\le  cd^2(L\infnorm{\vct{v}}\opnorm{\mtx{W}})^4$, we have 
\begin{align}
\sigma_{\min}^2\left(\mtx{J}\right)&\ge \sigma_{\min}^2\left(\widetilde{\mtx{J}}\right) - \left(4\sqrt{3} n+ Cd^{\frac{3}{2}} + C\sqrt{dn}\right)L^2\infnorm{\vct{v}}^2\opnorm{\mtx{W}}^2\,,\
\end{align}
with probability at least $1-2e^{-c'\sqrt{d}}$, where $c', C>0$ are constants (depending on $c$).  
\item In addition, with probability at least $1- 2 e^{-n/2}$, we have
\begin{align}
\sigma_{\max}(\J) &\le \sigma_{\max}(\tJ) + 3\sqrt{n}\twonorm{\Dv\mub}\,,\label{eq:J-JtB}
\end{align}
\end{itemize}
\end{proposition}
Now that we have established a connection between the spectrum of $\mtx{J}$ and $\tilde{\mtx{J}}$, we focus on bounding $\sigma_{\min}(\tJ)$ and $\sigma_{\max}(\tJ)$. The following proposition is at the core of our analysis and may be of independent interest. We defer the proof of this result to Section~\ref{proofJ}. 
%
\begin{proposition}\label{pro:J-eig} 
Let $\phi:\R\rightarrow \R$ be a general activation obeying $\abs{\phi''}\le L$. Assume the inputs $\vct{x}_i\in\R^d$ are distributed i.i.d.~$\mathcal{N}(\vct{0},\mtx{I}_d)$ for $i \in [n]$. 
Suppose that $n\ge d$ and $k\ge d$. Furthermore, assume
\begin{align}
&\sqrt{n} \le c\frac{\sigma_{\min}^2\left(\Gammab\W\right)}{L^4 \sigma^4_{\max}\left(\W\right)+1} d \label{ass1-1}\,,
\end{align}  
holds for a sufficiently small constant $c>0$.
Then, there exist constants $C>0$, such that,
\begin{align}
\sigma_{\min}(\tJ) &\ge \frac{d}{2} \, \sigma_{\min} (\Dv\Gammab \W)\,,\label{eq:sigminJ1}\\
 \sigma_{\max}(\tJ) &\le C\left(d+ \sqrt{nd}\, \sigma_{\max}(\Dv\Gammab \W)\right) \,,\label{eq:sigmaxJ1}
\end{align}
holds with probability at least $1 - ne^{-b_1\sqrt{n}} - n^{-1} - 2ne^{-b_2 d}$ with $b_1,b_2>0$ fixed numerical constants.
\end{proposition}
\begin{remark}\label{rem:assb}
As discussed in the proof of Proposition~\ref{pro:J-eig}, when Assumption~\ref{ass-phi} (b) holds in lieu of Assumption~\ref{ass-phi} (a), then $\Gammab = \boldsymbol{0}$. In this case the statement of the above proposition must be modified as follows.
Equation~\eqref{ass1-1} should be replaced with
 \begin{align}
&\sqrt{n} \le c\frac{1}{L^4 \sigma^4_{\max}\left(\W\right)+1} d\,.
\end{align}  
Further, we have the following bounds in lieu of equations~\eqref{eq:sigminJ1} and~\eqref{eq:sigmaxJ1}.
\begin{align*}
\sigma_{\min}(\bJ) &\ge \frac{d}{2} \,.\\
\sigma_{\max}(\bJ) &\le C d\,. 
\end{align*}
(See Equations~\eqref{assb-min} and \eqref{assb-max}).
\end{remark}
We now combine Propositions~\ref{pro:J-Jt} and~\ref{pro:J-eig} to characterize the spectrum of $\J$. We defer the proof of this result to Appendix \ref{app:log-eig}.
\begin{proposition}\label{cor:log-eig} 
Let $\phi:\R\rightarrow \R$ be a general activation obeying Assumption~\ref{ass-phi} $(a)$ and $ |\phi'' |\le L$ .
Assume the inputs $\vct{x}_i\in\R^d$ are distributed i.i.d.~$\mathcal{N}(\vct{0},\mtx{I}_d)$ for $i\in [n]$. 
Suppose that
\begin{align}
\label{boundass}
0<v_{\min} \le |v^*_\ell| \le v_{\max}\quad0<w_{\min} \le \twonorm{\wb^*_\ell} \le w_{\max}\quad\text{for }\ell=1,2,\ldots,k,
\end{align}
for some fixed constants $v_{\min}, v_{\max}, w_{\min}, w_{\max}$. Further, assume that $k\ge d$ and
\begin{align}
d \le n \le  \frac{c_0 \sigma_{\min}^4(\W)}{\sigma_{\max}^8(\W)} d^2\label{ass1-rep}
\end{align}  
 for a sufficiently small constant $c_0>0$. Then, there exist constants $C\ge c>0$, such that,
\begin{align}
\sigma_{\min}(\J)  &\ge c \, \sigma_{\min}(\W)\,d  \,,\label{eq:sigminJ} \\
\sigma_{\max}(\J) &\le  C\, \sigma_{\max}( \W)\sqrt{nk} \,,\label{eq:sigmaxJ}
\end{align}
holds with probability at least $1  - n^{-1} - 2ne^{-b \sqrt{d}}$ with $b>0$ a fixed numerical constant.

Also when Assumption~\ref{ass-phi} $(b)$ holds in lieu of Assumption~\ref{ass-phi} $(a)$, the term $\sigma_{\min}(\W)$ should be replaced by one in Equations~\eqref{ass1-rep} and \eqref{eq:sigminJ}.
\end{proposition}
A noteworthy case that will be used in our analysis for both local and global convergence results is $\phi(z) =z^2/2$. 
Note that for this choice of $\phi$, we have $\E[\phi'(\mtx{WX})] = 0$ and therefore the centering step~\eqref{cor:log-eig} becomes superfluous as $\tJ = \J$. Further, $\Gammab  = \mtx{I}$. 
Applying Proposition~\ref{cor:log-eig} to this case with $\mtx{W} = \mtx{I}$ and $\vct{v}= (1,1,\dotsc, 1,1)$, we obtain the following result.
\begin{corollary}\label{cor:ldentity} 
Let $\X\in\R^{d\times n}$ be a matrix with i.i.d. $\mathcal{N}(0,1)$ entries. 
Suppose that $d\le n \le c_1 d^2$, for a sufficiently small constant $c_1>0$.
Then, there exist constants $C\ge c>0$, such that,
\begin{align}
\sigma_{\min}(\X \ast \X)  &\ge cd  \,,\label{eq:sigminJ-I} \\
\sigma_{\max}(\X \ast \X) &\le  C\sqrt{nd} \,,\label{eq:sigmaxJ-I}
\end{align}
holds with probability at least $1 - ne^{-b_1\sqrt{n}} - n^{-1} - 2ne^{-b_2 d}$ with $b_1,b_2>0$ fixed numerical constants.
\end{corollary}
\subsection{Proof of Proposition~\ref{pro:J-eig}}\label{proofJ}
Bounding $\sigma_{\min}(\tJ)$ is involved and requires several technical lemmas that are of independent interest. We first present an outline of our approach and then discuss the steps in details. We then focus on bounding $\sigma_{\max}(\tJ)$, which follows along the same lines. 

\noindent Our proof strategy consists of four main steps:
\begin{enumerate}
\item \emph{(Whitening).}  
We let $\M\in \R^{d\times k}$ be the left inverse of $\Dv\Gammab\W$ and construct a block-diagonal matrix $\DM$ with $d$ copies of $\M$ on its diagonal. We construct the  whitened matrix $\DM\bJ\in \R^{d^2\times n}$. The reason this operation is useful is that it acts as a partial centering of the entries of $\bJ$ so that almost all entries of the resulting matrix have zero mean.

\item \emph{(dropping rows)}  Let us index the rows of $\DM\bJ\in \R^{d^2\times n}$ by $(i,j)$ for $1\le i\le j\le d$. We next construct the matrix $\bJc$ which is obtained from $\DM\bJ$ by dropping rows indexed by $(i,i)$. The reason this operation is useful is that it removes the entries of $\bJ$ that are not centered. Indeed, the notation $\bJc$ is used to cue that the columns of $\bJc$ are centered (have  zero mean). 
\item \emph{(Bounding sub-exponential norms).} In this step we show that the columns of $\bJc$ have bounded sub-exponential norm.  
\item \emph{(Bounding singular values).} We prove bounds on singular values of a matrix with independent columns and bounded sub-exponential norms.
\end{enumerate}
Steps one and two collectively act as a nontrivial ``centering" of the columns of $\bJ$. The reason this centering is required is that the columns of the matrix $\bJ$ have non-zero mean which lead to rather large sub-exponential norms. By centering the columns we are able to reduce the sub-exponential norm. The reader may be puzzled as to why we do not use the trivial centering of subtracting the mean from each column of $\bJ$. The reason we do not pursue this path is that we can not directly relate the minimum eigenvalues of the resulting matrix to that of $\bJ$ in a useful way. Steps one and two allow us to center the columns of $\bJ$, while being able to relate the minimum eigenvalue of $\bJ$ to that of the centered matrix $\bJc$. We note that under Assumption~\ref{ass-phi} (b), the whitening and droppings step are superfluous because the matrix $\bJ$ is centered. We are now ready to discuss each of these steps in greater detail.
\subsubsection{Whitening $\tJ$}
In this step we whiten the matrix $\bJ$ in such a way as most of the entries of the corresponding matrix have zero mean. To explain our whitening procedure we begin this section by computing the expectation of $\bJ$. 
\begin{lemma}\label{lem:center1}
Let $\Gammab$ be given by~\eqref{Gammab} and set $\A:= \Dv\Gammab \W$. Construct the block diagonal matrix $\DA\in \R^{kd\times d^2}$ with $d$ copies of $\A$ on its diagonal. Further, consider $\Q\in \R^{d^2\times n}$, where its rows are indexed by $(i,j)$, for $1\le i\le j \le d$. The rows $(i,i)$ of $\Q$ are all-one, while the other rows are all-zero. Then,
\begin{align}
\E[\bJ] = \DA\Q \,,
\end{align}
where the expectation is with respect to the randomness in the inputs $\x_i$.
\end{lemma}
Lemma~\ref{lem:center1} above is proven in Appendix~\ref{app:center1}.  To center, most of the entries of $\bJ$ we use the structure of $\Q$ and whiten it from the left-hand side. We will also show that this whitening procedure will not significantly decrease the minimum eigenvalue.

Note that we can assume that $\sigma_{\min}(\Dv\Gammab\mtx{W})>0$ otherwise, Claim~\eqref{eq:sigminJ1} becomes trivial. Now let $\M\in R^{d\times k}$ be the left inverse of $\Dv\Gammab\W$ and construct a block-diagonal matrix $\DM$ with $d$ copies of $\M$ on its diagonal.The lemma below, relates the minimum singular value of $\bJ$ to that of the whitened matrix $\mtx{D}_{\mtx{M}}\bJ$.

\begin{lemma}\label{lem:center2}
We have
\[\sigma_{\min}(\tJ) \ge  \sigma_{\min}(\mtx{D}_{\mtx{M}}\tJ) \sigma_{\min}(\Dv\Gammab\W)\,.\]
\end{lemma}
\begin{proof}
We prove this lemma by contradiction i.e.~assume
\[\sigma_{\min}(\tJ) <  \sigma_{\min}(\mtx{D}_{\mtx{M}}\tJ) \sigma_{\min}(\Dv\Gammab\W)\,.\]
Now let $\vb$ be the bottom singular vector of $\tJ$. Then,
\[
\twonorm{\tJ\vb} = \sigma_{\min}(\tJ) \twonorm{\vb} < \sigma_{\min}(\Dv\Gammab\W) \sigma_{\min}(\DM\tJ)\twonorm{\vb}  \le \frac{\twonorm{\DM\tJ\vb}}{\sigma_{\max}(\M)} \le\twonorm{\tJ\vb}\,,
\]
which is a contradiction. Note that in the penultimate inequality we used the fact that $\M\Dv\Gammab\W = \mtx{I}$ and the last inequality holds because $\|\DM\|=\sigma_{\max}(\M)$.
\end{proof}
\subsubsection{Dropping rows}
By Lemma~\ref{lem:center1}, we have
\[\E[\DM\bJ] = \DM\DA\Q = \Q\,. \]
Therefore, the whitened matrix $\DM\bJ$ is almost centered. To reach a completely centered matrix we drop the rows corresponding to the nonzero rows of $\mtx{Q}$. Specifically, let $\bJc$ be the matrix obtained after dropping rows $(i,i)$ from $\DM\bJ$, for $1\le i\le d$. By dropping these rows, $\Q$ vanishes and hence, $\E[\bJc] = \vct{0}$. Note that $\bJc$ is obtained by dropping $d$ of the rows from $\DM\tJ$. Hence, $\sigma_{\min}(\DM\tJ) \ge \sigma_{\min}(\bJc)$. Combining the latter with Lemma \ref{lem:center2} we arrive at
\begin{align}
\label{mahdicentered}
\sigma_{\min}(\tJ) \ge  \sigma_{\min}(\bJc) \sigma_{\min}(\Dv\Gammab\W).
\end{align}
Thus, in the remainder we shall focus on lower bounding $\sigma_{\min}(\bJc)$. 

\subsubsection{Bounding sub-exponential norms}
Let $\bJ_{\x}$ be the column of $\bJ$ corresponding to data point $\x$.
By Lemma~\ref{Pre1}, we can write
\[
\DM\bJ_{\x} = \x \otimes \M\Dv (\phi'(\W\x) - \E[\phi'(\W\x)])\,.
\]
We recall that $\bJc$ is obtained by dropping $d$ of the rows from $\DM\bJ$. By the structure of $\Q = \E[\DM\bJ]$, we can alternatively obtain $\bJc$ by dropping the same set of rows from $\DM\bJ - \E[\DM\bJ]$. Note that dropping entries from a vector can only reduce the Orlicz norm. Therefore,
\begin{align}\label{eq:subexp1}
\|\bJc_{\x}\|_{\psi_1} \le \Big\|\DM\bJ_{\x} - \E[\DM\bJ_{\x}] \Big\|_{\psi_1} = \Big\|\x\otimes \M\Dv \z - \E[\x\otimes \M\Dv \z ]\Big\|_{\psi_1}.
\end{align}
In order to bound the right-hand-side of~\eqref{eq:subexp1}, we need to study the tail of the random variable
\begin{align*}
\bigg\langle\x\otimes \M\Dv \z - \E[\x\otimes \M\Dv \z],\vct{u}\bigg\rangle
\end{align*}
for any unit-norm vector $\vct{u}\in\R^{kd}$. To simplify this random variable define $\widetilde{\mtx{U}}\in\R^{d\times d}$ by arranging every $k$ entries of $\vct{u}$ as the rows of $\widetilde{\mtx{U}}$. Note that $\vct{u}=$vect$(\widetilde{\mtx{U}})$. Thus Lemma~\ref{Pre2} allows us to rewrite this random variable in the simplified form
\begin{align}
\label{tempsubexp}
\langle\x\otimes \M\Dv \z - \E[\x\otimes \M\Dv \z],\vct{u}\rangle=\vct{x}^T\mtx{\widetilde{U}} \M \Dv\vct{z}-\E[\vct{x}^T\mtx{\widetilde{U}} \M \Dv\vct{z}],
\end{align}
with $\vct{z}\in\R^k$ defined as $\z:= \phi'(\W\x) - \E[\phi'(\W\x)]$. To characterize the tails of the latter random variable we use the following lemma proven in Appendix~\ref{proof:subexp2}.
\begin{proposition}\label{pro:subexp2}
Let $\vct{x}\in\R^d$ be a random Gaussian vector distributed as $\mathcal{N}(\vct{0},\mtx{I}_d)$. Also define  $\z:= \phi'(\W\x) - \E[\phi'(\W\x)]$. Then,
\begin{align} 
\label{HW} 
\mathbb{P}\bigg\{\abs{\vct{x}^T\mtx{U}\z-\E[\vct{x}^T\mtx{U}\z]}\ge t\bigg\}\le 2\exp\left(-\frac{1}{C}\min\left(\frac{t^2}{K^4\fronorm{\mtx{U}}^2},\frac{t}{K^2\opnorm{\mtx{U}}}\right)\right), 
\end{align}
holds with $C$ a fixed numerical constant and $K= \sqrt{L^2\sigma_{\max}^2(\mtx{W})+1}$.
\end{proposition}
We apply Proposition~\ref{pro:subexp2}, with $\mtx{U} = \mtx{\widetilde{U}} \M \Dv$ and use Lemma~\ref{Pre3}, to conclude that
\begin{align}
\Big\|\x^T \mtx{\widetilde{U}} \M \Dv \z - \E[\x^T \mtx{\widetilde{U}} \M \Dv \z ]\Big\|_{\psi_1}&\le
 C(L^2\sigma^2_{\max}(\W)+1) \fronorm{\mtx{\widetilde{U}} \M \Dv},\nonumber\\
&\le C(L^2\sigma_{\max}^2(\W)+1)\,  \sigma_{\max}(\M \Dv) \fronorm{\mtx{\widetilde{U}}} ,\nonumber\\
&\le \frac{C}{\sigma_{\min}(\Gammab\W)} (L^2\sigma_{\max}^2(\W)+1)\,.\label{eq:subexp3}
\end{align}
In the last inequality we used the fact that $\fronorm{\widetilde{\mtx{U}}}=\twonorm{\vct{u}}=1$.
Now since $\vct{u}$ was arbitrary, recalling Definition~\ref{def2:subexp} and combining equations \eqref{eq:subexp1}, \eqref{tempsubexp} and \eqref{eq:subexp3} we conclude that
\begin{align}\label{eq:subexp-impt}
\|\bJc_{\x}\|_{\psi_1} \le \frac{C}{\sigma_{\min}(\Gammab\W)} (L^2\sigma_{\max}^2(\W)+1)\,.
\end{align}
\subsubsection{Bounding minimum singular value}
Now that we have a bound on the sub-exponential norm of the columns of $\bJc$, we are now ready to bound $\sigma_{\min}(\bJc)$. To this aim we first state a lemma on the spectrum of matrices with independent sub-exponential columns. The Lemma is similar to~\cite[Theorem 3.2]{adamczak2011restricted} and we give a proof in Appendix~\ref{app:mineig}.
\begin{proposition}\label{pro:mineig}
Let $\vct{u_1}, \vct{u_2}, \dotsc, \vct{u_n}$ be independent sub-exponential random vectors and also let $\psi = \max_{1\le i\le n}\, \|\vct{u_i}\|_{\psi_1}$.  
Furthermore, let $\mtx{U}$ be a matrix with $\vct{u_1}, \dotsc, \mtx{u_n}$ as columns. Define $\eta_{\min} = \min_{1\le i\le n} \|\vct{u_i}\|$ and $\eta_{\max} = \max_{1\le i\le n} \|\vct{u_i}\|$.
Further, set $\xi = \psi K+K'$, where
$K,K'\ge 1$ are arbitrary but fixed. Also define
\begin{align*}
\Delta = C\xi^2 \eta_{\min} \sqrt{n} \log \Big({\frac{2\eta_{\min}}{\sqrt{n}}}\Big).
\end{align*}
Then
\begin{align*}
\sigma_{\min}^2(\mtx{U}) \ge \eta_{\min}^2 - \Delta\quad\text{and}\quad\sigma_{\max}^2(\mtx{U}) \le \eta_{\max}^2 + \Delta,
\end{align*}
hold with probability larger than
\[
1 - C\exp\Big(-cK\sqrt{n}\log \Big({\frac{2\eta_{\min}}{\sqrt{n}}} \Big) \Big) -\P\Big(\eta_{\max} \ge K'\eta_{\min} \Big)\,.
\]
Here, $c,C>0$ are universal constants.
\end{proposition} 
To apply Proposition~\ref{pro:mineig}, we only need to control norms of columns $\bJc_{\x}$  (and thus the parameters $\eta_{\min}$ and $\eta_{\max}$ in the proposition statement). To lighten the notation, we let $\z = \M \Dv (\phi'(\W\x) - \E[\phi'(\W\x)])$. The column of $\DM\bJ$ corresponding to data point $\x$ is given by $\x\otimes \z$ and the column $\bJc_{\x}$ is obtained after dropping the entries $x_i z_i$, for $1\le i\le d$. Hence,
\begin{align}
\label{temp1}
\twonorm{\bJc_{\x}}^2= \sum_{i=1}^d z_i^2 (\sum_{j\neq i} x_j^2) \,.
\end{align}
We continue by stating a lemma that bounds the Euclidean norm of the random vector $\vct{z}$. We defer the proof of this lemma to Appendix~\ref{app:twonormconc}.
\begin{lemma}\label{twonormconc} Assume $\vct{x}$ is a Gaussian random vector distributed as $\mathcal{N}(\vct{0},\mtx{I})$. Furthuremore, assume $\phi:\R\to \R$ is an activation function with bounded second derivative, i.e.~$\abs{\phi''}\le L$. Define
\begin{align*}
\rho(\mtx{W}):=L\frac{\sigma_{\max}\left(\mtx{W}\right)}{\sigma_{\min}\left(\Gammab\W\right)}.
\end{align*}
Then for $\z = \M \Dv(\phi'(\W\x) -\E[\phi'(\W\x)])$, 
\begin{align}
\label{lbnd}
\sqrt{d}-t\rho(\mtx{W}) \le \twonorm{\z} \le (\sqrt{d} +t )\rho(\mtx{W})\,, 
\end{align}
holds with probability at least $1-2e^{-\frac{t^2}{2}}$. 
\end{lemma}

Further, by concentration of $\chi^2$ random variables, with probability at least $1-2e^{-(d-1)/32}$, the following bounds hold. For any fixed $1\le i\le n$,
\begin{align}\label{eq:chi2}
 \frac{d-1}{2}\le \Big|\sum_{j\neq i} x_j^2 \Big| \le  \frac{3(d-1)}{2}\,.
\end{align}

Plugging bound \eqref{lbnd}, with $t=2\sqrt{\log n}$, and bound~\eqref{eq:chi2} into \eqref{temp1}, the followings inequalities hold with probability at least $1-2n^{-2} - 2e^{-(d-1)/32}$,
\begin{align}
\twonorm{\bJc_{\x}} &\ge \sqrt{\frac{d-1}{2}} (\sqrt{d}-2\rho(\W) \sqrt{\log n})\ge \frac{1}{3}\sqrt{d(d-1)}\,,\label{tmp0}\\
\twonorm{\bJc_{\x}} &\le \sqrt{\frac{3(d-1)}{2}} (\sqrt{d}+2\sqrt{\log n})\rho(\W) \le 3\sqrt{{d(d-1)}}\rho(\W)\,.
\end{align}
Note that the second inequality in~\eqref{tmp0} holds because by our assumption~\eqref{ass1-1},
\begin{align}
\label{cond}
\rho^2(\W) <\frac{L^2\sigma^2_{\max}\left(\mtx{W}\right)+1}{\sigma^2_{\min}\left(\Gammab\mtx{W}\right)}\le c_1\frac{d}{\sqrt{n}} < \frac{d}{16\log n}\,.
\end{align}

Therefore, by union bounding over $n$ columns with probability at least $1-2n^{-1}-2n e^{-(d-1)/32}$ we have 
\begin{align}
\eta_{\min} &\ge \frac{1}{3}\sqrt{d(d-1)}\,,\label{eq:eta-min}\\
\eta_{\max}&\le  3\sqrt{d(d-1)}\rho(\W)\,. \label{eq:eta-max}
\end{align}
Also, by~\eqref{eq:subexp-impt}, the maximum sub-exponential norms of the columns is bounded by 
\begin{align}\label{eq:subexp-impt2}
\psi \le \frac{C}{\sigma_{\min}(\Gammab\W)} (L^2\sigma_{\max}^2(\W)+1)\,.
\end{align} 
We now have all the elements to apply Proposition~\ref{pro:mineig}. In particular we use $K = 1$ and $K' = 10 \rho(\W)$, to conclude that
\begin{align}
\sigma_{\min}^2(\bJc) \ge \eta_{\min}^2 \left(1 - C (\psi + \rho(\W))^2 \frac{\sqrt{n}}{\eta_{\min}} \log\left(\frac{2\eta_{\min}}{\sqrt{n}} \right) \right)\,,
\end{align}
holds with probability at least 
\[
1- C\exp\left(-c\sqrt{n} \log \left(\frac{2\eta_{\min}}{\sqrt{n}}\right) \right)\,.
\]
Using the derived lower bound on $\eta_{\min}$ and the upper bound on $\psi$, we obtain that with high probability $\sigma_{\min}^2(\bJc) \ge d/2$, as long as 
\begin{align}
\frac{L^4 \sigma^4_{\max}\left(\W\right)+1}{\sigma_{\min}^2\left(\Gammab\W\right)}\le c_1 \frac{d}{\sqrt{n}} \,,
\end{align}
which holds true by the assumption given in Equation~\eqref{ass1-1}.

Finally by \eqref{mahdicentered}, we conclude that 
\begin{align}
\sigma_{\min}(\bJ) \ge \frac{d}{2} \sigma_{\min}(\Dv\Gammab\W)\,.
\end{align}

Note that under Assumption~\ref{ass-phi} (b), we skip the whitening step and by following the same proof, we obtain  
\begin{align}
\sigma_{\min}(\bJ) \ge \frac{d}{2} \,.\label{assb-min}
\end{align}

\subsubsection{Bounding maximum singular value}
The argument we used for the minimum eigenvalue
does not apply to the largest eigenvalue.
This is because $\DM$ is a fat matrix
and the maximum eigenvalue of 
$\DM \tJ$ does not provide any information about the maximum eigenvalue of $\tJ$. To see this clearly, consider the case where the column space of $\tJ$ intersects with the null space of $\DM$ e.g.~when maximum eigenvector of $\tJ$ is in the null space of $\DM$.
For bounding $\sigma_{\max}(\bJ)$, instead of whitening and dropping some of the row, we center $\bJ$ in the most natural way. Specifically, in this section we let $\bJc:= \bJ - \E[\bJ]$. Then, 
\[\bJc_{\x} =  \tJ_{\x} - \E[\tJ_{\x}] = \x\otimes \Dv\z - \E[\x\otimes \Dv\z]\,, \]
with $\z = \phi'(\W\x)-\E[\phi'(\W\x)]$. 
Applying Proposition~\ref{pro:subexp2}, we get
\begin{align}
\|\bJc_{\x}\|_{\psi_1} \le {C}\|\vb\|_\infty (L^2\sigma_{\max}^2(\W)+1)\,.
\end{align}
Similar to Section~\ref{pro:mineig}, we can bound the maximum and the minimum norms of the columns of $\bJc$. With probability at least $1-2n^{-1}-2n e^{-(d-1)/32}$,
\begin{align}
\eta_{\min}\ge d/3\,,\quad \eta_{\max}\le 3d\rho(\W)\,.
\end{align}
By employing Proposition~\ref{pro:mineig}, 
\begin{align}
\sigma_{\max}^2(\bJc) &\le C \left (d^2+ \sqrt{n} d\|\vb\|_\infty^2\, (L^4\sigma_{\max}^4(\W)+1) + \sqrt{n}d\rho^2(\W) \right) \nonumber \\
&\le Cd^2\left(1+ c_1\|\vb\|_\infty^2\, \sigma_{\min}^2(\Gammab \W)+c_1\right)\nonumber\\
&\le Cd^2\left(1+ \|\vb\|_\infty^2\, \sigma_{\min}^2(\Gammab \W)\right)\,,\label{sigmaxJc}
\end{align}
where the second inequality follows from our assumption given by Equation~\eqref{ass1-1}. 

We next bound the maximum eigenvalue of $\E[\bJ]$. Invoking Lemma~\ref{lem:center1}, we have
\[
\E[\bJ] = \DA\Q = \vec(\A) \mathbf{1}^T_n\,,
\]
where $\mathbf{1}_n\in \R^n$ is the all-one vector. Therefore, 
\begin{align}
\sigma_{\max}(\E[\bJ]) \le \sqrt{n} \|\A\|_F \le \sqrt{n d}\, \sigma_{\max}(\Dv\Gammab \W)\,.\label{sigmaxQ}
\end{align}
Combining Equations~\eqref{sigmaxJc} and \eqref{sigmaxQ} via the triangular inequality we arrive at:
\begin{align*}
\sigma_{\max}(\bJ) &\le Cd\Big(1+ \|\vb\|_\infty\, \sigma_{\min}(\Gammab \W)\Big) + \sqrt{n d}\, \sigma_{\max}(\Dv\Gammab \W),\\
&\le C\left(d+ \sqrt{nd}\, \sigma_{\max}(\Dv\Gammab \W)\right)\,,
\end{align*}
where the second inequality holds because $n\ge d$ and $\|\vb\|_\infty \sigma_{\min}(\Gammab \W) \le \sigma_{\max}(\Dv\Gammab \W)$ as per Lemma~\ref{Pre4}. 

Note that under Assumption~\ref{ass-phi} (b), matrix $\bJ$ is centered and the whitening step should be skipped. Indeed, by a similar argument for inequality~\eqref{sigmaxJc}, we have
\begin{align}
\sigma_{\max}^2(\bJ) &\le C \left (d^2+ \sqrt{n} d v_{\max}^2 (L^4\sigma_{\max}^4(\W)+1) + \sqrt{n}d\sigma_{\max}^2(\W) \right),\nonumber\\
&\le Cd^2 (1+c_1 v_{\max}^2+c_1),\nonumber\\
&\le C d^2\,. \label{assb-max}
\end{align}

\subsection{Proof of Theorem~\ref{localthm}}\label{pfthmloc}
We begin by stating a few useful lemmas. We defer the proofs of these lemmas to the Appendices.
Throughout, we assume $|\phi'| < B$ and $|\phi''| <L$. We present the proof under Assumption~\ref{ass-phi} (a). The proof under Assumption~\ref{ass-phi} (b) follows by an analogous argument. We begin by a lemma that bound the perturbation of the Jacobian matrix as a function of its inputs. We defer the proof to Appendix~\ref{proof:Jpert}.
\begin{lemma}\label{Jpert} Assume $\vct{x}_i$ are i.i.d.~Gaussian random vectors distributed as $\mathcal{N}(\vct{0},\mtx{I})$. Further, suppose that $n\le c d^2$, and let $\J(\vb,\W) = \Dv\phi'(\W\X)\ast\X$ be the Jacobian matrix associated to weights $\vb$ and $\W$. Then, for any two fixed matrices $\tW,\W\in\R^{k\times d}$ and any two fixed vectors $\vb,\tv \in \R^{k}$, 
\begin{align}\label{eq:Jpert}
\opnorm{J(\tv,\tW)-J(\vb,\W)}\le C'\sqrt{nd}\, \Big(\|\vb\|_\infty \opnorm{\tW-\W} + \|\tv-\vb\|_\infty\Big) ,
\end{align}
holds with probability at least  $1 - ne^{-b_1\sqrt{n}} - n^{-1} - 2ne^{-b_2 d}$. Here, $b_1,b_2>0$ are fixed numerical constants.
\end{lemma} 
Now note that by Proposition~\ref{cor:log-eig} for $\vb^*$ and $\W^*$, we have
\begin{align}
\sigma_{\min}(\J(\vb^*,\W^*)) \ge c\,\sigma_{\min}(\W^*) d\,.\label{sminW*}
\end{align}
Using this value of $c$ we define the radius $$R: = \frac{c}{4C'} \sigma_{\min}(\W^*) \sqrt{\frac{d}{n}}\,,$$
where $C'$ is the same quantity that appears in Equation~\eqref{eq:Jpert}.  Without loss of generality, we can assume $c < 4C' w_{\max}v_{\max}$. Plugging this into the definition of $R$ together with the fact that $d \le n$, allows us to conclude that $R\le v_{\max}$.

Define the set $\Omega\subseteq \R^{k\times d} \times \R^k$ as follows:
\begin{eqnarray}
\Omega := \Big\{(\vb,\W):\, \|\vb-\vb^*\|_\infty \le R,\,\, \fronorm{\W-\W^*} \le R/\|\vb^*\|_\infty\Big\}
\end{eqnarray}
In the next lemma, proven in Appendix~\ref{proof:aux1}, we relate the gradients $\nabla_{(\mtx{W})} \mathcal{L}(\vct{v},\mtx{W})$ and $\nabla_{\vct{v}} \mathcal{L}(\vct{v},\mtx{W})$ to the function value $ \mathcal{L}(\vct{v},\mtx{W})$.
\begin{lemma}\label{lem:aux1}
For $(\vb,\W)\in \Omega$, the following inequalities hold with probability at least $1 - n^{-1} - 2ne^{-b \sqrt{d}}$ for some constant $b>0$.
\begin{align}
\fronorm{\nabla_{\mtx{W}} \mathcal{L}(\vct{v},\mtx{W})}^2\ge \mL \mathcal{L}(\vct{v},\mtx{W})\,\quad \,\,&\mL:= \frac{1}{2} c^2 \sigma_{\min}^2(\W^*) \frac{d^2}{n}\,,\label{eq:aux1}\\
\fronorm{\nabla_{\mtx{W}} \mathcal{L}(\vct{v},\mtx{W})}^2\le \mU \mathcal{L}(\vct{v},\mtx{W})\,\quad \,\,&\mU:= \frac{9}{2} C^2 \sigma_{\max}^2(\W^*) k\,,\label{eq:aux2}\\
\opnorm{\nabla_{\vct{v}} \mathcal{L}(\vct{v},\mtx{W})}_\infty^2\le \tmU \mathcal{L}(\vct{v},\mtx{W})\,\quad \,\,&\tmU:=  \left(4\phi^2(0)+ 128 B^2 + 128B^2 w_{\max}^2\right)\,.\label{eq:aux3}
\end{align}
\end{lemma}
Our next lemma, proven in Appendix~\ref{proof:aux2}, upper bounds the function value $\mathcal{L}(\vct{v},\mtx{W})$ in terms of distance of $(\vct{v},\mtx{W})$ to the planted solution $(\vct{v},\mtx{W})$.
\begin{lemma}\label{lem:aux2}
The following bound holds with probability at least $1 - 2e^{-b_0d}$ for some constant $b_0>0$.
\begin{align}
\mathcal{L}(\vb,\W) \le \twonorm{\vb^*}^2 \fronorm{\W-\W^*}^2 + 2\left( \phi^2(0)+2B^2 w_{\max}^2 \right) k^2 \opnorm{\vb-\vb^*}_\infty^2\,.
\end{align}
\end{lemma} 
Finally, the lemma below, proven in Appendix~\ref{proof:aux3}, controls the second order derivative of the loss function $\cL(\vb,\W)$.
\begin{lemma}\label{lem:aux3}
The function $\cL(\vb,\W)$ is $\beta$-smooth on $\Omega$. Namely, there exists an event of probability at least  $1  - n^{-1} - 2ne^{-b \sqrt{d}}$, such that on this event, for any $(\vb,\W)\in \Omega$, we have
\begin{align}
\nabla^2 \cL(\vb,\W) \le \beta \mtx{I}\,,
\end{align}
where $\nabla^2\cL(\vb,\W) \in \R^{(kd+k)\times (kd+k)}$ denotes the Hessian w.r.t both $\vb$, $\W$. Further, the smoothness parameter $\beta$ 
is given by 
\begin{align}\label{eq:beta}
\beta: = \left( {3C^2}\sigma_{\max}^2(\W^*) + 8 v_{\max}^2 BL+ 4B^2w_{\max}^2+2\phi^2(0)\right) k\,.
\end{align}
\end{lemma}
With this lemmas in place we are now ready to present our local convergence analysis. To this aim note that Lemma~\ref{lem:aux1} (Equation~\eqref{eq:aux1}) implies that for $(\vb,\W)\in \Omega$, the function $\cL(\vb,\W)$ satisfies the Polyak-Lojasiewicz (PL) inequality~\cite{polyak1963gradient}
\begin{align}
\label{polyak}
\fronorm{\nabla_{\mtx{W}} \mathcal{L}(\vct{v},\mtx{W})}^2 + \fronorm{\nabla_{\vb} \mathcal{L}(\vct{v},\mtx{W})}^2 \ge \mL \mathcal{L}(\vct{v},\mtx{W})\,.
\end{align}
We shall now focus on how the loss function value changes in one iteration. Using the $\beta$-smoothness condition for $\s\le 1/\beta$ we have
\begin{align}
&\cL\left(\vct{v}-\s \nabla_{\vct{v}} \cL\vct{v},\mtx{W}-\s\nabla_{\mtx{W}} \cL(\mtx{W})\right),\nonumber\\
&\le \cL(\vb,\mtx{W})-\s\left(\twonorm{\nabla_{\vct{v}} \cL(\vct{v},\mtx{W})}^2+\fronorm{\nabla_{\mtx{W}} \cL(\vct{v},\mtx{W})}^2\right)+\frac{\s^2\beta}{2}\left(\twonorm{\nabla_{\vct{v}} \cL(\vct{v},\mtx{W})}^2+\fronorm{\nabla_{\mtx{W}} \cL(\vct{v},\mtx{W})}^2\right),\nonumber\\
&= \cL(\vb,\mtx{W})-\s\left(1-\frac{\s\beta}{2}\right)\left(\twonorm{\nabla_{\vct{v}} \cL(\vct{v},\mtx{W})}^2 + \fronorm{\nabla_{\mtx{W}} \cL(\vct{v},\mtx{W})}^2\right),\nonumber\\
&\le \left(1-\s\mL\left(1-\frac{\s\beta}{2}\right)\right)\cL(\vct{v},\mtx{W})\le \left(1-\frac{\s}{2}\mL\right)\cL(\vct{v},\mtx{W})\,,\label{eq:aux4}
\end{align}
where in the last inequality we used \eqref{polyak}.

Our assumptions on the initial weights $\vct{v}_0$ and $\mtx{W}_0$ in equations \eqref{eq:initv0} and \eqref{eq:initW0} imply the following identities.
\begin{align}
\fronorm{\W_0-\W^*} &\le \frac{R \mL}{4\sqrt{2 \mU}\, v_{\max} \twonorm{\vb^*}}\,, \label{eq:initW}\\ 
\opnorm{\vb_0-\vb^*}_\infty &\le \frac{R\mL}{4\sqrt{2\mU} v_{\max}\sqrt{2(\phi^2(0)+2B^2 w_{\max}^2)}\,k  }\,. \label{eq:initv}
\end{align}
We next show that starting with $\vb_0,\W_0$ obeying \eqref{eq:initW} and~\eqref{eq:initv}, the entire trajectory of gradient descent remains in the set $\Omega$. To establish this we use induction on $\tau$.
The induction basis $\tau=0$ is trivial. Assuming the induction hypothesis for $0\le t\le \tau-1$, we show that the result continues to hold for $t=\tau$.

By our induction hypothesis ($(\vct{v}_t,\mtx{W}_t)\in\Omega$ for $t=0,1,\ldots,\tau-1$) and \eqref{eq:aux4}, we have
\begin{align*}
\cL(\vct{v}_{\tau},\mtx{W}_{\tau})&= \cL\left(\vct{v}_{\tau-1}-\s\nabla_{\vct{v}} \cL(\vct{v}_{\tau-1},\mtx{W}_{\tau-1}),\mtx{W}_{\tau-1}-\s\nabla_{\mtx{W}} \cL(\mtx{v}_{\tau-1},\mtx{W}_{\tau-1})\right),\\
&\le\left(1-\frac{\s}{2}\mL\right)\cL(\vct{v}_{\tau-1},\mtx{W}_{\tau-1}).
\end{align*}
By iterating the above identity we arrive at
\begin{align}\label{eq:gradient-tau}
\cL(\vct{v}_{\tau},\mtx{W}_{\tau})
&\le\left(1-\frac{\s}{2}\mL\right)^\tau \cL(\vct{v}_0,\mtx{W}_0).
\end{align}
To show that $(\vct{v}_\tau,\mtx{W}_\tau)\in\Omega$ we proceed by quantifying how far the gradient descent trajectory can get from $(\vct{v}_0,\W_0)$.
\begin{align}
\fronorm{\mtx{W}_\tau-\mtx{W}_0}=&\fronorm{\sum_{t=1}^{\tau} \left(\mtx{W}_t-\mtx{W}_{t-1}\right)}
\le\sum_{t=1}^{\tau}\fronorm{\mtx{W}_t-\mtx{W}_{t-1}}\nonumber\\
\le& \s \sum_{t=1}^{\tau}\fronorm{\nabla_{\mtx{W}} \cL(\vct{v}_{t-1},\mtx{W}_{t-1})}
\stackrel{(a)}{\le} \s\sqrt{\mU} \sum_{t=1}^{\tau}\sqrt{\cL(\vct{v}_{t-1},\mtx{W}_{t-1})}\nonumber\\
\stackrel{(b)}{\le}& \s\sqrt{\mU} \sum_{t=1}^{\tau}\sqrt{\left(1-\frac{\s}{2}\mL\right)^{t-1}\cL(\vct{v}_0,\mtx{W}_0)}\nonumber\\
\le& \s\sqrt{\mU \cL(\vct{v}_0,\mtx{W}_0)} \sum_{t=1}^{\tau}\left(1-\frac{\s}{4}\mL\right)^{t-1}\nonumber\\
\le&\frac{\s\sqrt{\mU \cL(\vct{v}_0,\mtx{W}_0)}}{\s\frac{\mL}{4}}\nonumber\\
=&\frac{4}{\mL}{\sqrt{\mU \cL(\vct{v}_0,\mtx{W}_0)}}\,,\label{W-dist}
\end{align}
where $(a)$ follows from $(\vb_{t-1},\W_{t-1})\in \Omega$ and Lemma \ref{lem:aux1} equation \eqref{eq:aux2} and $(b)$ follows from~\eqref{eq:gradient-tau}. 
Likewise, we obtain
\begin{align}
\opnorm{\vct{v}_\tau-\vct{v}_0}_\infty \le \frac{4}{\mL} {\sqrt{\tmU \cL(\vct{v}_0,\mtx{W}_0)}}\label{v-dist}
\end{align}
Using bounds~\eqref{W-dist} and~\eqref{v-dist}, in order to show that $(\vb_\tau,\W_\tau)\in \Omega$, it suffices to show that 
\begin{align}
\cL(\vb_0,\W_0) \le \frac{R^2 \mL^2}{16\max(\tmU , \mU v_{\max}^2 )}\,.\label{eq:LB}
\end{align}
The right-hand sides of Equation~\eqref{eq:LB} depends on $\max(\tmU,\mU v_{\max}^2)$. The dominant term is $\mU v_{\max}^2$ because it is of order at least $k$, while $\mU$ is $O(1)$.
Therefore, the desired bound in~\eqref{eq:LB} is equivalent to 
\begin{align}
\cL(\vb_0,\W_0) \le \frac{R^2 \mL^2}{16\mU v_{\max}^2}\,.\label{eq:LB2}
\end{align}

We can verify Equation~\eqref{eq:LB2} by using Lemma~\ref{lem:aux2} combined with \eqref{eq:initW} and~\eqref{eq:initv}. This completes the induction argument and shows that $(\vb_\tau,\W_\tau)\in \Omega$ for all $\tau\ge 1$. 

Finally, since $(\vb_\tau,\W_\tau)\in \Omega$ for all $\tau\ge 0$, \eqref{eq:gradient-tau} holds for all $\tau\ge 0$. Substituting for $\mL$, we obtain 
\begin{align}
\cL(\vct{v}_{\tau},\mtx{W}_{\tau})
&\le\left(1-\frac{\s c^2}{4} \sigma_{\min}^2( \W^*) \frac{d}{n}\right)^\tau \cL(\vct{v}_0,\mtx{W}_0),\nonumber\\
&\le\left(1-{ c^2} w_{\max}  \frac{\s d}{4n}\right)^\tau \cL(\vct{v}_0,\mtx{W}_0)\,,
\end{align}
where the last step holds because $\sigma_{\min}(\W^*) < w_{\max}$. 

This concludes the proof under Assumption~\ref{ass-phi} (a). The claim under Assumption~\ref{ass-phi} (b) can be proven by a similar argument. The only required adjustment is that the initial radius $R$ and the term $\mL$ should now be defined via $R= \frac{c}{4C'} \sqrt{{d}/{n}}$ and $\mL= c^2 {d^2}/(2n)$.
\subsection*{Acknowledgements}
This work was done in part while M.S. was visiting the Simons Institute for the Theory of Computing. M. Soltanolkotabi is supported by the Packard Fellowship in Science and Engineering, a Sloan Research Fellowship in Mathematics, an NSF-CAREER under award \#1846369, the Air Force Office of Scientific Research Young Investigator Program (AFOSR-YIP) under award \#FA9550-18-1-0078, an NSF-CIF award \#1813877, and a Google faculty research award. A.J.~was partially supported by a Google Faculty Research Award. A.J.~would also like to acknowledge the financial support of the Office of the Provost at the University of Southern California through the Zumberge Fund Individual Grant Program. We would like to thank Marco Mondelli and Simone Bombari for bringing to our attention a mistake in the argument of Proposition \ref{pro:J-Jt} in the previous version of this manuscript which has now been corrected. M.S.~would like to thank Peter Bartlett for discussions related to \cite{zhong2017recovery}. 
\bibliography{Bibfiles,Bibfiles2}
\bibliographystyle{plain}
\newpage
\appendix
\section{Proof of preliminary lemmas}
\label{pfprem}
\subsection{Proof of Lemma~\ref{Pre4}}\label{proof:Pre4}
If $\sigma_{\min}(\A) = 0$ then the claim is obvious. We therefore assume that $\A$ has full column rank. Let $i^* = \arg\max_{1\le i \le k} |v_i|$.
We choose $\ub$ such that $\A\ub = e_{i^*}$. We then have
\begin{align}\label{Eig1}
\|\vb\|_{\ell_\infty} = |v_{i^*}| = \|\Dv \A \ub\| \le \sigma_{\max}(\Dv\A) \|\ub\| \,.
\end{align}
We also have
\begin{align}\label{Eig2}
\sigma_{\min}(\A) \|\ub\| \le \|\A\ub\| = 1\,.
\end{align}
Equations~\eqref{Eig1} and~\eqref{Eig2} together implies the desired result.

\subsection{Proof of Lemma~\ref{Pre3}}\label{prfPre3}
We have
\begin{align}
\E\abs{Y}^p=&\int_0^\infty \mathbb{P}\{\abs{Y}^p\ge u\} \de u\nonumber\\
=&\int_0^\infty \P\{\abs{Y}\ge t\} p t^{p-1}\de t\nonumber\\
=&\int_0^{\frac{A}{B}}e^{-c\frac{t^2}{A}}p t^{p-1}\de t+\int_{\frac{A}{B}}^\infty e^{-c\frac{t}{B}}pt^{p-1}\de t\nonumber\\
\le& p\left(\int_{0}^\infty e^{-c\frac{t^2}{A}}t^{p-1}\de t+\int_{0}^\infty e^{-c\frac{t}{B}}t^{p-1}dt\right)\nonumber\\
=&\left(\frac{A}{c}\right)^{\frac{p}{2}}\Gamma\left(\frac{p}{2}\right)+\left(\frac{B}{c}\right)^p\Gamma(p)\nonumber\\
\le&\left(\frac{A}{2c} p\right)^{\frac{p}{2}}+\left(\frac{B}{c}p\right)^p
\le 2 \max\left(\sqrt{\frac{A}{2c}{p}},\frac{B}{c}p\right)^p\le 2p^p\max\left(\sqrt{\frac{A}{2c}},\frac{B}{c}\right)^p\,.\label{Y1}
\end{align}

Let $\xi =\max(\sqrt{\frac{A}{2c}},\frac{B}{c})$. This implies that
\begin{align*}
\E \exp(|Y|/C) = 1+ \sum_{p=1}^\infty \frac{\E|Y|^p}{C^p p!} \le 1+ 2\sum_{p=1}^\infty \left( \frac{\xi e}{C} \right)^p = 1 +\frac{2\xi e}{C-\xi e}\,.
\end{align*}
The first inequality follows from~\eqref{Y1}; in the second one we use $p!\ge (p/e)^p$. Therefore, we have $\E\exp(|Y|/C)\le 2$, if $C\ge 3e \xi$. This yields $\|Y\|_{\psi_1} \le 9\max(\sqrt{\frac{A}{2c}},\frac{B}{c})$.
\subsection{Proof of Lemma~\ref{XkrX}}\label{app:XkrX}
To prove this result we define the mapping $\mathcal{O}:\R^{d\times n}\mapsto \R^{d^2\times n}$ as $\mathcal{O}(\mtx{X}) = \mtx{X}\ast \mtx{X}$. To show that for almost every $\mtx{X}$, we have $\rank(\mathcal{O}(\mtx{X_0})) = n$ when $n\ge d$ we first construct $\mtx{X_0}\in \R^{d\times n}$, such that $\rank(\mathcal{O}(\mtx{X_0})) = n$. To this aim define the following two sets of vectors in $\R^d$:
\begin{eqnarray*}
\mathcal{S}_1 &\equiv& \{\vct{e_i}:\; 1\le i \le d\}\,,\\
\mathcal{S}_2 &\equiv& \{\vct{e_i}+\vct{e_j}:\; 1\le i<j\le d\}\,,
\end{eqnarray*}
where $\vct{e_i}$ is the $i$-th element of the standard basis with one at the $i$-th position and zero everywhere else. Let $\mathcal{S} = \mathcal{S}_1\cup \mathcal{S}_2$ and note that $|\mathcal{S}| = d(d+1)/2$. We construct $\mtx{X_0}\in\R^{d\times n}$ by choosing $n$ arbitrary vectors in $\mathcal{S}$ as its columns. In order to show that $\mathcal{O}(\mtx{X_0})$ is full rank, we index its rows by pairs $(i,j)$, with $1\le i\le j\le d$. For any $i\neq j$, and note that $(\vct{e_i}+\vct{e_j})\otimes (\vct{e_i}+\vct{e_j})$ is the only vector in $\mathcal{S}^{\otimes}\equiv \{\vct{v}\otimes \vct{v}:\, \vct{v} \in \mathcal{S}\}$ which is nonzero at the $(i,j)$-th coordinate. This implies that the vectors in $\mathcal{S}_2^{\otimes}$ are linearly independent of those in $\mathcal{S}^{\otimes}$. Therefore, it suffices to show that $\mathcal{S}_1^{\otimes}$ is a linearly independent set. To see this note that for any $1\le i\le d$, the vector $\vct{e_i}\otimes \vct{e_i}$ is the only vector in $\mathcal{S}_1^\otimes$ which is nonzero at the $(i,i)$ position. Hence, $\mathcal{S}_1^\otimes$ is also a set of linearly independent vectors. This completes the proof of showing that $\mathcal{O}(\mtx{X_0})$ is full rank.  

We now show that $\mathcal{O}(\mtx{X})$ is full rank for almost every $\mtx{X}$.  To this aim define the polynomial mapping $g: \R^{d^2\times n}\mapsto \R$, with $g(\mtx{A})$ denoting the the sum of the squares of the determinants of all possible different subsets of $n$ rows from $\mtx{A}$. Since $n\le d^2$, $\mathcal{O}(\mtx{X})$ is rank deficient (i.e.~does not have full column rank) if and only if $g(\mathcal{O}(\mtx{X})) = 0$. Moreover, as we showed above, $\mathcal{O}(\mtx{X_0})$ is full rank and hence $g(\mathcal{O}(\mtx{X_0}))\neq 0$. Hence, $g(\mathcal{O}(\mtx{X}))$ is not identically zero and since it is a polynomial mapping, its zeros --which correspond to matrices $\mtx{X}$ such that $\mathcal{O}({\mtx{X}})$ is singular-- are a set of measure zero. This completes the proof.

\subsection{Proof of Lemma \ref{PRlemma}}
\label{PRlemmapf}
We begin by considering the eigenvalue decomposition of the matrix $\mtx{A}$ given by
\begin{align*}
\mtx{A}=\sum_{i=1}^d \lambda_i\vct{v}_i\vct{v}_i^T.
\end{align*}
Note that we have
\begin{align}
\label{tmpWFlem}
\frac{1}{n}\sum_{i=1}^n\left(\vct{x}_i^T\mtx{A}\vct{x}_i\right)\vct{x}_i\vct{x}_i^T-\left(2\mtx{A}+\text{trace}(\mtx{A})\mtx{I}\right)=&\frac{1}{n}\sum_{i=1}^n\left(\vct{x}_i^T\left(\sum_{j=1}^d\lambda_j\vct{v}_j\vct{v}_j^T\right)\vct{x}_i\right)\vct{x}_i\vct{x}_i^T-\left(2\sum_{j=1}^d\lambda_j\vct{v}_j\vct{v}_j^T+\left(\sum_{j=1}^d\lambda_j\right)\mtx{I}\right),\nonumber\\
=&\sum_{j=1}^d\lambda_j\left(\frac{1}{n}\sum_{i=1}^n(\vct{x}_i^T\vct{v}_i)^2\vct{x}_i\vct{x}_i^T-\left(2\vct{v}_i\vct{v}_i^T+\mtx{I}\right)\right).
\end{align}
To continue we state a Lemma due to \cite{WF} (see also \cite{soltanolkotabi2014algorithms, soltanolkotabi2017structured} for closely related results).
\begin{lemma}\label{WFlem}\cite[Lemma 7.4]{WF} Let $\vct{a}\in\R^d$ be a fixed vector and for $i=1,2,\ldots,n$, let $\vct{x}_i$ be distributed i.i.d.~$\mathcal{N}(\vct{0},\mtx{I}_d)$. Then as long as 
\begin{align*}
n\ge c(\delta) d\log d,
\end{align*}
with $c$ a constant depending only on $\delta$. Then
\begin{align*}
\opnorm{\frac{1}{n}\sum_{i=1}^n (\vct{x}_i^T\vct{a})^2\vct{x}_i\vct{x}_i^T-\left(2\vct{a}\vct{a}^T+\twonorm{\vct{a}}^2\mtx{I}\right)}\le \delta \twonorm{\vct{a}}^2,
\end{align*}
holds with probability at least $1-10e^{-\gamma d}-8/d^2$ with $\gamma$ a fixed numerical constant.
\end{lemma}
Applying the union bound to the above lemma we conclude that for the unit norm vectors $\{\vct{v}_j\}_{j=1}^d$
\begin{align*}
\opnorm{\frac{1}{n}\sum_{i=1}^n\abs{\vct{x}_i^T\vct{v}_j}^2\vct{x}_i\vct{x}_i^T-\left(2\vct{v}_j\vct{v}_j^T+\mtx{I}\right)}\le \delta,
\end{align*}
hold simultaneously with probability at least $1-10de^{-\gamma d}-8/d$. Combining the latter inequalities together with \eqref{tmpWFlem} we conclude that
\begin{align*}
\opnorm{\frac{1}{n}\sum_{i=1}^n\left(\vct{x}_i^T\mtx{A}\vct{x}_i\right)\vct{x}_i\vct{x}_i^T-\left(2\mtx{A}+\text{trace}(\mtx{A})\mtx{I}\right)}=&\opnorm{\sum_{j=1}^d\lambda_j\left(\frac{1}{n}\sum_{i=1}^n(\vct{x}_i^T\vct{v}_i)^2\vct{x}_i\vct{x}_i^T-\left(2\vct{v}_i\vct{v}_i^T+\mtx{I}\right)\right)},\\
\le&\sum_{j=1}^d\lambda_j\opnorm{\frac{1}{n}\sum_{i=1}^n\abs{\vct{x}_i^T\vct{v}_j}^2\vct{x}_i\vct{x}_i^T-\left(2\vct{v}_j\vct{v}_j^T+\mtx{I}\right)},\\
\le&\delta\left(\sum_{j=1}^d\lambda_j\right),\\
=&\delta\cdot\text{trace}(\mtx{A}),
\end{align*}
completing the proof.
\section{Proof of Proposition~\ref{pro:J-Jt}}\label{app:J-Jt}
Recalling definition of $\mub$ given by~\eqref{mub}, we first show that $\E[\phi'(\W\x)] = \mub$. Due to rotational invariance of Gaussian distribution, without loss of generality, we assume that $\vct{w_i} = \|\vct{w_i}\| \vct{e_1}$. Therefore,
\begin{align*}
\E[\phi'(\vct{w_i}^T\x)] = \E[\phi'(\|\vct{w_i}\| x_1)] = \mu_i\,.
\end{align*}
Consequently, we have $\E[\phi'(\W\X)]=\mub \vct{1}_n^T$.

We shall try to deduce a bound on the minimum singular value of $\J$ via a lower bound on the minimum singular value of $\tJ = \X\ast\Dv(\phi'(\W\X)-\mub \vct{1}_n^T)$. Define $\Delta = \Dv(\phi'(\W\X)-\mub \vct{1}_n^T)$ and $\widetilde{\vct{\mu}}=\Dv\vct{\mu}$. By Lemma~\ref{Pre5}, we have
\begin{align*}
 \mtx{J}^T\mtx{J} &= \left(\mtx{X} * (\mtx{\Delta} + \widetilde{\vct{\mu}}\vct{1}_n^T)\right)^T \left(\mtx{X} * (\mtx{\Delta} + \widetilde{\vct{\mu}}\vct{1}_n^T)\right)\\
 &= (\mtx{X}^T\mtx{X}) \circ
 \left(\mtx{\Delta}^T\mtx{\Delta} + \mtx{\Delta}^T\widetilde{\vct{\mu}}\vct{1}_n^T+\vct{1}_n\widetilde{\vct{\mu}}^T\mtx{\Delta}+\twonorm{\widetilde{\vct{\mu}}}^2\vct{1}_n\vct{1}_n^T\right)\,.
\end{align*}
Similarly we have
\begin{align*}
\widetilde{\mtx{J}}^T\widetilde{\mtx{J}} &= \left(\mtx{X} * \mtx{\Delta} \right)^T \left(\mtx{X} * \mtx{\Delta} \right)
 = (\mtx{X}^T\mtx{X}) \circ
 \left(\mtx{\Delta}^T\mtx{\Delta}\right)\,.
\end{align*}
Subtracting the above two expressions we obtain
\begin{align*}
\mtx{J}^T\mtx{J}=\widetilde{\mtx{J}}^T\widetilde{\mtx{J}}+\left(\mtx{X}^T\mtx{X}\right)\circ\left(\mtx{\Delta}^T\widetilde{\vct{\mu}}\vct{1}_n^T+\vct{1}_n\widetilde{\vct{\mu}}^T\mtx{\Delta}+\twonorm{\widetilde{\vct{\mu}}}^2\vct{1}_n\vct{1}_n^T\right)\,.
\end{align*}
Now note that the trivial identity
\begin{align*}
\left(\twonorm{\widetilde{\vct{\mu}}}\vct{1}_n+\frac{\mtx{\Delta}^T\widetilde{\vct{\mu}}}{\twonorm{\widetilde{\vct{\mu}}}}\right)\left(\twonorm{\widetilde{\vct{\mu}}}\vct{1}_n+\frac{\mtx{\Delta}^T\widetilde{\vct{\mu}}}{\twonorm{\widetilde{\vct{\mu}}}}\right)^T\succeq \mtx{0}\,,
\end{align*}
along with the fact that the Hadamard (entrywise) product of two positive-semidefinite matrices is positive-semidefinite, implies that
\begin{align*}
\left(\mtx{X}^T\mtx{X}\right)\circ\left(\mtx{\Delta}^T\widetilde{\vct{\mu}}\vct{1}_n^T+\vct{1}_n\widetilde{\vct{\mu}}^T\mtx{\Delta}+\twonorm{\widetilde{\vct{\mu}}}^2\vct{1}_n\vct{1}_n^T\right)\succeq -\left(\mtx{X}^T\mtx{X}\right)\circ\frac{(\mtx{\Delta}^T\widetilde{\vct{\mu}})(\mtx{\Delta}^T\widetilde{\vct{\mu}})^T}{\twonorm{\widetilde{\vct{\mu}}}^2}\,.
\end{align*}
Thus it suffices to show that the matrix
\begin{align*}
\mtx{M}:=\frac{1}{\twonorm{\widetilde{\vct{\mu}}}^2}\left(\mtx{X}^T\mtx{X}\right)\circ \left(\mtx{\Delta}^T\widetilde{\vct{\mu}}\right) \left(\mtx{\Delta}^T\widetilde{\vct{\mu}}\right)^T
\end{align*}
obeys
\begin{align*}
\opnorm{\mtx{M}}\le Cd^2L^2\infnorm{\vct{v}}^2\opnorm{\mtx{W}}^2
\end{align*}
with high probability. To this aim note that the latter matrix can be written in the form $\mtx{Y}^T\mtx{Y}$ where 
\begin{align*}
\mtx{Y}=\mtx{X}\text{diag}(\vct{d})\quad\text{with}\quad\vct{d}=\frac{1}{\twonorm{\widetilde{\vct{\mu}}}}\mtx{\Delta}^T\widetilde{\vct{\mu}}.
\end{align*}
To continue we state the following result from \cite{adamczak2011sharp}.
\begin{theorem}( \cite[Theorem 1]{adamczak2011sharp})\label{lemma69} Let $\vct{y}_1, \vct{y}_2,\ldots,\vct{y}_n\in\R^d$ with $n\ge d$ be independent random vectors satisfying 
\begin{align}
\label{subexp}
\sup_{i=1,2,\ldots,n}\text{ }\|\vct{y}_i\|_{\psi_1}\le \psi\,,
\end{align}
and
\begin{align}
\label{norm}
\mathbb{P}\Bigg\{\max_{i=1,2,\ldots,n}\twonorm{\vct{y}_i} > K\sqrt[4]{nd}\Bigg\}\le e^{-\sqrt{d}}\,,
\end{align}
for some $K\ge 1$. Define $\mtx{Y}=\begin{bmatrix} \vct{y}_1 & \vct{y}_2 & \ldots  &\vct{y}_n\end{bmatrix}$ and $\mtx{\Sigma}:=\E[\vct{y}_1\vct{y}_1^T]$. Then with probability at least $1-2e^{-c\sqrt{d}}$ the following holds:
\begin{align*}
\opnorm{\mtx{Y}\mtx{Y}^T-n\mtx{\Sigma}}\le  C\left(\psi+K\right)^2\sqrt{dn}
\end{align*}
\end{theorem}
To use this theorem we need to verify the assumptions. We first verify the
 sub-exponential assumption. Observe that $\vct{d}_i$ are generated i.i.d. according to the distribution
\begin{align*}
f(\vct{x})=\frac{1}{\twonorm{\widetilde{\vct{\mu}}}}\widetilde{\vct{\mu}}^T\left(\Dv\phi'(\mtx{W}\vct{x})-\widetilde{\vct{\mu}}\right)\,.
\end{align*}
The gradient of this function reads as
\begin{align*}
\nabla_{\vct{x}} f(\vct{x})=\frac{1}{\twonorm{\widetilde{\vct{\mu}}}}\mtx{W}^T\text{diag}\left(\phi''(\mtx{W}\vct{x})\right)\left(\vct{v}\circ\widetilde{\vct{\mu}}\right)\,.
\end{align*}
Thus
\begin{align}\label{eq:grad-norm}
\twonorm{\nabla_{\vct{x}} f(\vct{x})}\le L\infnorm{\vct{v}}\opnorm{\mtx{W}}
\end{align}
This in turn implies that $f(\vct{x})$ is an $L\infnorm{\vct{v}}\opnorm{\mtx{W}}$-Lipshitz function of a Gaussian vector and is thus a sub-Gaussian random variable with sub-Gaussian norm $c L\infnorm{\vct{v}}\opnorm{\mtx{W}}$, for some absolute constant $c>0$. Given that $\vct{x}_i\sim\mathcal{N}(\mtx{0},\mtx{I}_d)$, $\vct{y}_i$ is sub-exponential with sub-exponential norm at most $\psi=c L\infnorm{\vct{v}}\opnorm{\mtx{W}}$.

We next verify the norm bound condition~\eqref{norm}. Note that we have
\begin{align*}
\twonorm{\vct{y}_i}=\twonorm{\vct{x}_i}\abs{\vct{d}_i}
\end{align*}
Since  $\vct{x}_i\sim\mathcal{N}(\mtx{0},\mtx{I}_d)$, we have $\twonorm{\vct{x}_i}\ge 2\sqrt{d}$ with probability at most $e^{-cd}$. Hence,
\begin{align*}
\max_{i\in[n]} \twonorm{\vct{x}_i} \ge 2\sqrt{d},
\end{align*}
holds with probability at most $ne^{-cd}$. The latter is smaller than $\frac{1}{2}e^{-\sqrt{d}}$ as long  as $n\le e^{c'd}$. In addition, as mentioned in the proof of sub-exponential in the previous part we have that $\vct{d}_i$ is a sub-Gaussian random variable with sub-Gaussian norm $c L\infnorm{\vct{v}}\opnorm{\mtx{W}}$. Thus,
\begin{align*}
\mathbb{P}\Bigg\{\max_i \abs{\vct{d}_i}\ge ctL\infnorm{\vct{v}}\opnorm{\mtx{W}}\Bigg\}\le 2ne^{-t^2}.
\end{align*}
Therefore, when $n\le \frac{1}{4}e^{\sqrt{d}}$ using 
\begin{align*}
t=2\sqrt[4]{d}\ge\sqrt{\log (4n)}+\sqrt[4]{d}\quad\Rightarrow \quad t^2\ge \log (4n)+ \sqrt{d}
\end{align*}
we have that 
\begin{align*}
\mathbb{P}\Bigg\{\max_i \abs{\vct{d}_i}\ge ctL\infnorm{\vct{v}}\opnorm{\mtx{W}}\Bigg\}\le 2ne^{-t^2}\le \frac{1}{2}e^{-\sqrt{d}}.
\end{align*}
Putting the latter two together we conclude that
\begin{align*}
\mathbb{P}\Big\{\twonorm{\vct{y}_i}\ge cLd^{\frac{3}{4}}\infnorm{\vct{v}}\opnorm{\mtx{W}}\Big\}\le\frac{1}{2}e^{-\sqrt{d}}+\frac{1}{2}e^{-\sqrt{d}}\le e^{-\sqrt{d}}.
\end{align*}
Thus we can take 
\begin{align*}
K=\frac{c\sqrt{d}}{\sqrt[4]{n}}L\infnorm{\vct{v}}\opnorm{\mtx{W}}
\end{align*}
in \eqref{norm}. Note that this choice of $K$ obeys $K\ge 1$ as long as
\begin{align*}
n\le cd^2 L^4\infnorm{\vct{v}}^4\opnorm{\mtx{W}}^4 \,.
\end{align*}
Thus using Theorem \ref{lemma69} we conclude that as long as $n\le cd^2L^4\infnorm{\vct{v}}^4\opnorm{\mtx{W}}^4 $ we have
\begin{align}
\label{myexp}
\opnorm{\mtx{Y}\mtx{Y}^T-n\mtx{\Sigma}}\le  C\left(d^{\frac{3}{2}} + \sqrt{dn}\right) L^2\infnorm{\vct{v}}^2\opnorm{\mtx{W}}^2
\end{align}
We next upper bound $\opnorm{\mtx{\Sigma}}$. Recall the definition
\begin{align*}
f(\vct{x})=\frac{1}{\twonorm{\widetilde{\vct{\mu}}}}\widetilde{\vct{\mu}}^T\left(\Dv\phi'(\mtx{W}\vct{x})-\widetilde{\vct{\mu}}\right)\,,
\end{align*}
and note that
\begin{align*}
\opnorm{\mtx{\Sigma}} &= \opnorm{\vct{y}_1\vct{y}_1^T}\\
&=\sup_{\vct{u}\in \mathbb{S}^{d-1}}\vct{u}^T\left(\E\big[f^2(\vct{x})\vct{x}\vct{x}^T\big]\right)\vct{u}\\
&=\sup_{\vct{u}\in \mathbb{S}^{d-1}} \E\Big[f^2(\vct{x})\left(\vct{x}^T\vct{u}\right)^2\Big]\\
&\le \sup_{\vct{u}\in \mathbb{S}^{d-1}} \sqrt{\E\big[f^4(\vct{x})\big]}\sqrt{\E\big[\left(\vct{x}^T\vct{u}\right)^4\big]}\,,
\end{align*}
where the last line follows from Cauchy-Schwarz inequality. To continue note that since $f(\vct{x})$ is an $R:=L\infnorm{\vct{v}}\opnorm{\mtx{W}}$-Lipshitz function of a Gaussian vector, using the tail bound inequality (see e.g., \cite[Proposition 5.34]{Vers}), we have
\[
\mathbb{P}(|f(\vct{x})|>t) \le 2e^{-\frac{t^2}{2B^2}}\,.
\]
 Thus,
 \begin{align*}
 \E[f^4(\vct{x})] &= \int_0^{\infty} \mathbb{P}(|f(\vct{x})|\ge t) 4t^3 \de t\\
 &\le \int_0^{\infty} 8e^{-\frac{t^2}{2R^2}} t^3\de t\\
 & = 32R^4\int_0^{\infty} e^{-t^2} t^3\de t \\
 &= 16R^4 = 16 L^4\infnorm{\vct{v}}^4\opnorm{\mtx{W}}^4\,.
 \end{align*}
 We also have $\E\big[\left(\vct{x}^T\vct{u}\right)^4\big] = 3$. Thus,
\begin{align*}
\opnorm{\mtx{\Sigma}}\le 4\sqrt{3} L^2\infnorm{\vct{v}}^2\opnorm{\mtx{W}}^2\,.
\end{align*}
Combining the latter with \eqref{myexp} we can conclude that as long as $n\le cd^2$ we have
\begin{align*}
\opnorm{\mtx{Y}\mtx{Y}^T}&\le \opnorm{\mtx{Y}\mtx{Y}^T-n\mtx{\Sigma}}+n\opnorm{\mtx{\Sigma}}\\
&\le   C\left(d^{\frac{3}{2}} + \sqrt{dn}\right) L^2\infnorm{\vct{v}}^2\opnorm{\mtx{W}}^2+4\sqrt{3} nL^2\infnorm{\vct{v}}^2\opnorm{\mtx{W}}^2
\end{align*}
concluding the proof.

We next prove the bound in~\eqref{eq:J-JtB}. Note that
\[
\sigma_{\max}(\J)\le \sigma_{\max}(\tJ) + \|\Dv\mub\| \sigma_{\max}(\X)\,.
\]
By~\cite[Corollary 5.35]{Vers}, $\sigma_{\max}(\X) \le \sqrt{d}+2\sqrt{n}$, with probability at least $1 - 2e^{-n/2}$. The result follows be recalling that $n\ge d$.
\section{Proof of Proposition~\ref{cor:log-eig}}\label{app:log-eig}
We first prove the claim under Assumption~\ref{ass-phi} (a). We start by bounding the entries of $\Gammab$ and mean vector $\mub$. 

Recall that $\gamma_\phi(\sigma) = \E(\phi''(\sigma g))$, where $g\sim \mathcal{N}(0,1)$. Also not that the function $|\gamma_\phi(\sigma)|$ is continuous and always positive by Assumption~\ref{ass-phi} (a). Therefore it attains its minimum over any compact set. Since $\twonorm{\w_\ell} \in [w_{\min}, w_{\max}]$, for all $1\le \ell\le k$,
there exists a constant $\gamma_{\min} > 0$, such that $|\gamma_{\phi}(\twonorm{\w_\ell})| \ge \gamma_{\min}$, for all $1\le \ell \le k$. Furthermore, $|\phi''| \le L$ implies $|\gamma_\phi(\sigma)| = |\E[\phi''(\sigma g)]| \le L$. Hence, for $1\le \ell \le k$, we have
\begin{align}
0 <\gamma_{\min} < |\Gamma_{\ell\ell}| \le L\,.
\end{align}
 
By a similar argument, we have $|\mu_\ell| > \mu_{\min} > 0$, for some constant $\mu_{\min}$ and for all $1\le \ell\le k$. Furthermore,
\begin{align*}
\twonorm{\mub} &= \twonorm{\E[\phi'(\W\x)]},\\
 &\overset{(a)}{\le} \E[\twonorm{\phi'(\W\x)}],\\
 &\overset{(b)}{=}\E\bigg[\twonorm{\phi'(\vct{0})+\text{diag}\left(\int_{0}^1\phi''(t\mtx{W}\vct{x})\de t\right)\mtx{W}\vct{x}}\bigg],\\
 &\overset{(c)}{\le} \E\bigg[\twonorm{\phi'(0)\mtx{1}_k} + L \sigma_{\max}(\W) \twonorm{x} \bigg],\\
 &\overset{(d)}{\le} \phi'(0) \sqrt{k} + L\sigma_{\max}(\W) \sqrt{d}\,.  
\end{align*}
Here (a) follows from Jenson's inequality, (b) follows from Taylor's theorem, (c) from the triangular inequality together with the fact that $\abs{\phi''}\le L$ and the definition of the maximum eigenvalue, and (d) follows from Jenson's inequality which implies $(\E[\twonorm{\vct{x}}])^2\le\E[\twonorm{\vct{x}}^2]=d$.
In conclusion, we have
\begin{align}\label{eq:muBounds}
\mu_{\min} \sqrt{k} \le \twonorm{\mub} \le \sqrt{k} + L\sigma_{\max}(\W) \sqrt{d}\,.
\end{align} 
We next verify the assumption of Proposition~\ref{pro:J-Jt}.  By our assumptions in the current proposition,
\begin{align}\label{eq:trivialB0}
\frac{n}{d^2}&\le \frac{c_0\sigma_{\min}^4(\mtx{W})}{\sigma_{\max}^8(\mtx{W})} \le \frac{c_0w_{\max}^4}{w_{\min}^8}\,,
\end{align}
since for an arbitrary row $\ell\in [k]$, $\sigma_{\min}(\W)\le \twonorm{\vct{w}_\ell}\le w_{\max}$. Also, $\sigma_{\max}(\W)\ge \twonorm{\vct{w}_\ell} \ge w_{\min}$. For 
\begin{align}
c_0 < c\;\Big(Lv_{\min} \frac{w^{3}_{\min}}{w_{\max}}\Big)^4 \,,
\end{align} 
we have
\[
\frac{n}{d^2}\le \frac{c_0w_{\max}^4}{w_{\min}^8}\le c(Lv_{\min} w_{\min})^4\le c (L \infnorm{\vct{v}} \sigma_{\max}(\W))^4\,.
\]
Therefore, the assumption of Proposition~\ref{pro:J-Jt} holds. Using this proposition, the following bounds hold with probability at least $1-2e^{-c'\sqrt{d}}$:
\begin{align}
\sigma_{\min}^2\left(\mtx{J}\right)\ge \sigma_{\min}^2\left(\tilde{\mtx{J}}\right) - \left(4\sqrt{3} n+ Cd^{\frac{3}{2}} + C\sqrt{dn}\right)L^2 \infnorm{\vct{v}}^2\opnorm{\mtx{W}}^2\,.\label{eq:UB0}
\end{align}
We next lower bound the right-hand side. By our assumption on $n/d^2$, we have
\begin{align}
n \opnorm{\mtx{W}}^2 &\le  \frac{c_0\sigma_{\min}^4(\mtx{W})}{\sigma_{\max}^8(\mtx{W})} \opnorm{\mtx{W}}^2 d^2 \nonumber\\
& \le \frac{c_0\sigma_{\min}^2(\mtx{W})}{\sigma_{\max}^6(\mtx{W})} \sigma_{\min}^2(\W) d^2\nonumber \\
& \le \frac{c_0 w^2_{\max}}{w^6_{\min}} \sigma_{\min}^2(\W) d^2\,.\label{eq: UB1}
\end{align}
Likewise we have
\begin{align}
\sqrt{nd} \opnorm{\mtx{W}}^2 &\le  \frac{\sqrt{c_0}\sigma_{\min}^2(\mtx{W})}{\sigma_{\max}^4(\mtx{W})} \opnorm{\mtx{W}}^2 d^{3/2} \nonumber\\
& \le \frac{\sqrt{c_0}}{\sigma_{\max}^2(\mtx{W})} \sigma_{\min}^2(\W) d^{3/2}\nonumber \\
& \le \frac{\sqrt{c_0} }{w^2_{\min}} \sigma_{\min}^2(\W) d^2\,.\label{eq: UB2}
\end{align}
Further, by invoking assumption~\eqref{ass1-rep}, we have 
\begin{align}
\frac{\sigma_{\max}^8(\W)}{\sigma_{\min}^4(\W)c_0}\le d &\Rightarrow  \frac{\opnorm{\W}^2}{\sqrt{d}} \le \sqrt{c_0}\frac{\sigma_{\min}^2(\W)}{\sigma_{\max}^2(\W)}\nonumber\\
&\Rightarrow \opnorm{\W}^2 d^{3/2}\le \sqrt{c_0}\frac{\sigma_{\min}^2(\W)}{\sigma_{\max}^2(\W)} d^2\le \frac{\sqrt{c_0}}{w_{\min}^2}
\sigma_{\min}^2(\W) d^2\,.\label{eq: UB3}
\end{align}
Putting~\eqref{eq: UB1}, \eqref{eq: UB2} and \eqref{eq: UB3} together we arrive at
\[
\left(4\sqrt{3} n+ Cd^{\frac{3}{2}} + C\sqrt{dn}\right)\opnorm{\mtx{W}}^2 \le \left(4\sqrt{3}\frac{c_0 w^2_{\max}}{w^6_{\min}} + 2C\frac{\sqrt{c_0} }{w^2_{\min}}\right) \sigma_{\min}^2(\W) d^2\,.
\] 
Using the above bound in~\eqref{eq:UB0} we get
\begin{align}
\sigma_{\min}^2\left(\mtx{J}\right)\ge \sigma_{\min}^2\left(\tilde{\mtx{J}}\right) - \left(4\sqrt{3}\frac{c_0 w^2_{\max}}{w^6_{\min}} + 2C\frac{\sqrt{c_0} }{w^2_{\min}}\right)L^2 v_{\max}^2 \sigma_{\min}^2(\W) d^2\,.\label{eq:UB4}
\end{align}

We then move to upper bound $\sigma_{\max}(\J)$.
\begin{align}
\sigma_{\max}(\J) &\le \sigma_{\max}(\bJ) + 3 \sqrt{n} \twonorm{\Dv\mub},\nonumber\\
&\le \sigma_{\max}(\bJ) + 3\sqrt{n} v_{\max} ( \sqrt{k} + L\sigma_{\max}(\W) \sqrt{d}),\nonumber\\
&\le \sigma_{\max}(\bJ) + 3\sqrt{nk} v_{\max} (L +1/w_{\min}) \sigma_{\max}(\W)\,.\label{eq:J-JtB2}
\end{align}
Here, in deriving the upper bound for $\sigma_{\max}(\J)$, we used the bound on $\twonorm{\vct{\mu}}$ from \eqref{eq:muBounds} in the first step. We then used the fact that $k\le d$ combined with $w_{\min} \le \sigma_{\max}(W)$ in the second step.

All that remains is to bound $\sigma_{\min}(\tJ)$ and $\sigma_{\max}(\tJ)$ by appealing to Proposition~\ref{pro:J-eig}. Before applying this result however, we need to show that assumption~\eqref{ass1-1} on the sample size holds. To this aim note that
\begin{align*}
\frac{L^2\sigma_{\max}^4(\W)+1}{\sigma_{\min}^2(\Gammab\W)} &\overset{(a)}{\le} \frac{L^4 \sigma_{\max}^4(\W) +1}{\gamma_{\min}^2 \sigma_{\min}^2(\W)},\\
&\overset{(b)}{\le} \frac{\sigma_{\max}^4(\W)}{\sigma_{\min}^2(\W)}\cdot \frac{1}{\gamma_{\min}^2} \left(L^4+\frac{1}{w_{\min}^4}\right),\\
& \le \frac{d}{\sqrt{n}}\cdot   \frac{\sqrt{c_0}}{\gamma_{\min}^2} \left(L^4+\frac{1}{w_{\min}^4}\right)\,,
\end{align*}
which immediately implies assumption \eqref{ass1-1}. Here, (a) follows from $\sigma_{\min}(\mtx{\Gamma}\mtx{W})\ge \gamma_{\min}\sigma_{\min}(\mtx{W})$, (b) from $\sigma_{\max}(\mtx{W})\ge w_{\min}$, and (c) from  \eqref{ass1-rep}. With assumption \eqref{ass1-1} in place, we can now combine Proposition~\ref{pro:J-eig} with \eqref{eq:UB4}  to conclude that
\begin{align*}
\sigma_{\min}^2(\J) &\ge \sigma_{\min}^2(\bJ) -  \left(4\sqrt{3}\frac{c_0 w^2_{\max}}{w^6_{\min}} + 2C\frac{\sqrt{c_0} }{w^2_{\min}}\right)L^2 v_{\max}^2 \sigma_{\min}^2(\W) d^2\\
&\ge  \frac{v^2_{\min} \gamma^2_{\min}}{4} \sigma_{\min}^2(\W) d^2 -  \left(4\sqrt{3}\frac{c_0 w^2_{\max}}{w^6_{\min}} + 2C\frac{\sqrt{c_0} }{w^2_{\min}}\right)L^2 v_{\max}^2 \sigma_{\min}^2(\W) d^2\\
& = \left\{\frac{v^2_{\min} \gamma^2_{\min}}{4}-  \left(4\sqrt{3}\frac{c_0 w^2_{\max}}{w^6_{\min}} + 2C\frac{\sqrt{c_0} }{w^2_{\min}}\right)L^2 v_{\max}^2 \right\} \sigma_{\min}^2(\W) d^2\\
&\ge c^2 \sigma_{\min}^2(\W) d^2 \,,
\end{align*}
for some constant $c>0$ by choosing $c_0$ sufficiently small. In addition,
\begin{align*}
\sigma_{\max}(\J) &\le \sigma_{\max}(\bJ) + 3\sqrt{nk} v_{\max} (L +1/w_{\min}) \sigma_{\max}(\W)\nonumber\\
&\le C'(d + L v_{\max}\sigma_{\max}(\W)\sqrt{nk}) + 3\sqrt{nk}\, v_{\max} (L +1/w_{\min}) \sigma_{\max}(\W)\nonumber\\
&\le C'\, \sigma_{\max}( \W)\sqrt{nk}\,,
\end{align*}
for a constants $C'>0$ that depends on $v_{\max}$, $w_{\min}$, $L$. The claim under Assumption~\ref{ass-phi} (b) follows in a similar manner by using Remark~\ref{rem:assb} in lieu of Proposition~\ref{pro:J-eig}.
\section{Proof of Proposition~\ref{pro:subexp2}}\label{proof:subexp2}
We prove Proposition~\ref{pro:subexp2} via an asymmetric version of Hanson-Wright inequality. We begin by the definition of the convex concentration property which is the most general condition under which it is known that Hanson-Wright inequality holds.
\begin{definition}[Convex concentration property] Let $\vct{x}$ be a random vector in $\R^d$. We will say that $\vct{x}$ has the convex concentration property with constant $K$ if for every $1$-Lipschitz convex function $\psi:\R^n\rightarrow\R$, we have $\E[\psi(\vct{x})]<\infty$ and for every $t>0$,
\begin{align*}
\mathbb{P}\big\{\abs{\psi(\vct{x})-\E[\psi(\vct{x})]}\ge t\big\}\le 2 e^{-\frac{t^2}{2K^2}}.
\end{align*}
\end{definition}
We will prove the proposition by using a result by~\cite{adamczak2015note} on the Hanson-Wright inequality stated below.
\begin{lemma}\label{lem0} Let $\ub$ be a mean zero random vector in $\R^d$. If $\ub$ has the convex concentration property with constant $K$ then for any matrix $\mtx{A}\in\R^{d\times d}$ and every $t>0$,
\begin{align*}
\mathbb{P}\big\{\abs{\ub^T\mtx{A}\ub-\E[\ub^T\mtx{A}\ub]}\ge t\big\}\le 2\exp\left(-\frac{1}{C}\min\left(\frac{t^2}{2K^4\fronorm{\mtx{A}}^2},\frac{t}{K^2\opnorm{\mtx{A}}}\right)\right).
\end{align*}
\end{lemma}
We wish to apply this result to the vector $\ub = (\x,\z)^T$. Note that $\ub$ has zero mean. To this aim, we first state a lemma about the convex concentration property for the random vector $\ub$ whose proof appears in Section \ref{secpage41} below.
\begin{lemma}
\label{lem1}
Let $\vct{x}\in\R^d$ be a random Gaussian vector distributed as $\mathcal{N}(\vct{0},\mtx{I}_d)$. Also assume that the nonlinear function $\phi:\R\rightarrow\R$ has bounded second derivative, i.e.~$\abs{\phi''}\le L$. Then, the random vector $\ub = (\x,\z)^T$
obeys the convex concentration property with $K= (L^2\sigma_{\max}^2(\mtx{W})+1)^{1/2}$.
\end{lemma}
Lemmas \ref{lem0} and \ref{lem1} combined allows us to conclude that
\begin{align*}
\P\big\{\abs{\ub^T\A \ub - \E[\ub^T \A \ub]} \ge t \big\}
\le 2\exp\left(-\frac{1}{C}\min\left(\frac{t^2}{2K^4\fronorm{\mtx{A}}^2},\frac{t}{K^2\opnorm{\mtx{A}}}\right)\right).
\end{align*}
For $\A=\begin{bmatrix}\mtx{0} & \mtx{U}\\\mtx{U}^T & \mtx{0}\end{bmatrix}$, this is equivalent to 
\begin{align*}
\mathbb{P}\bigg\{\abs{\vct{x}^T\mtx{U}\z-\E[\vct{x}^T\mtx{U}\z]}\ge \frac{t}{2}\bigg\}\le 2\exp\left(-\frac{1}{C}\min\left(\frac{t^2}{2K^4\fronorm{\mtx{U}}^2},\frac{t}{K^2\opnorm{\mtx{U}}}\right)\right).
\end{align*}
The proof is complete by replacing $t$ with $2t$.
\subsection{Proof of Lemma \ref{lem1}}\label{secpage41}
Define $\ub_1= (\x_1,\z_1)^T$ and $\ub_2= (\x_2,\z_2)^T$. Note that for any convex function $\psi:\R^{2d}\rightarrow\R$ that is $1$-Lipschitz we have
\begin{align*}
\abs{\psi\left(\ub_2\right)-\psi\left(\ub_1\right)}^2\le& \twonorm{\ub_2-\ub_1}^2 = \twonorm{\x_1-\z_1}^2 +\twonorm{\x_2-\z_2}^2\\
=&\twonorm{\vct{x}_2-\vct{x}_1}^2+\twonorm{\phi'(\mtx{W}\vct{x}_2)-\phi'(\mtx{W}\vct{x}_1)}^2\\
\le&\twonorm{\vct{x}_2-\vct{x}_1}^2+L^2\twonorm{\mtx{W}(\vct{x}_2-\vct{x}_1)}^2\\
\le& \left(1+L^2\sigma_{\max}^2(\mtx{W})\right)\twonorm{\vct{x}_2-\vct{x}_1}^2.
\end{align*}
Thus $\psi(\ub)$ is a Lipschitz function of $\vct{x}$ with Lipschitz constant $K = \sqrt{1+L^2\sigma_{\max}^2(\mtx{W})}$. The convex concentration property follows from Gaussian isoperimetry~\cite{Led01}.

\section{Proof of Lemma~\ref{lem:center1}}\label{app:center1}
For $g\sim \mathcal{N}(0,1)$ and $r\in \R$, we have
\begin{align*}
\E[\phi'( rg)g]= \frac{1}{r}\E[\phi'( rg)rg]=r\gamma_\phi(r)\,.
\end{align*}
Let us start by calculating $\E[\phi'(\vct{w}_i^T\vct{x})\vct{x}]$. Note that due to symmetry of the Gaussian distribution without loss of generality we can assume that $\vct{w}_i=\twonorm{\vct{w}_i}\vct{e}_1$. In this case we have
\begin{align*}
\E[\phi'(\twonorm{\vct{w}_i}\vct{e}_1^T\vct{x})\vct{x}]=&\E[\phi'( \twonorm{\vct{w}_i} x_1)x_1 \vct{e}_1]+\E[\phi'(\twonorm{\vct{w}_i}x_1)\left(\mtx{I}-\vct{e}_1\vct{e}_1^T\right)\vct{x}]\\
=&\twonorm{\vct{w}_i}\gamma_\phi(\twonorm{\vct{w}_i})\vct{e}_1\,,
\end{align*}
where the last step holds because $x_1$ is independent of $x_i$, for $i\neq 1$.
Replacing $\vct{e}_1$ with $\vct{w}_i/\twonorm{\vct{w}_i}$ the latter calculation immediately implies that
\begin{align*}
\E[\phi'(\vct{w}_i^T\vct{x})\vct{x}]=\gamma_\phi(\twonorm{\vct{w}_i})\vct{w}_i\,.
\end{align*}
This also implies that
\begin{align*}
\E[(\phi'(\vct{w}_i^T\vct{x})-\E[\phi'(\vct{w}_i^T\vct{x})]) \vct{x}] = \gamma_\phi(\twonorm{\vct{w}_i})\vct{w}_i\,.
\end{align*}
Hence,
\begin{align}\label{EJtx}
\E[\tJ_{\x}] = \E[\x\otimes \Dv(\phi'(\W\x) - \E[\phi'(\W\x)])] = \vec(\Dv\Gammab\W)\,,
\end{align}
with $\Gammab = \diag (\gamma_\phi(\twonorm{\vct{w}_1}),\gamma_\phi(\twonorm{\vct{w}_2}),\ldots,\gamma_\phi(\twonorm{\vct{w}_k}) )$.

On the other hand, note that for any vector $\vct{a}$,
$\E[\langle\vct{a},\vct{x}\rangle\vct{x}]=\vct{a}$.

Writing it in matrix form, for any matrix $\A\in \R^{k\times d}$ we have
\begin{align}\label{Akron}
\E[\vct{x}\otimes \A \vct{x}]=\vec(\A)\,.
\end{align}
By comparing~\eqref{EJtx} and~\eqref{Akron}, we get
\begin{align}\label{eq:centertJ}
\E[\tJ_{\x}] = \E[\x\otimes \Dv\Gammab \W \x]\,.
\end{align}
Recalling the definition $\A:=\Dv\Gammab\W$ and using Lemma~\ref{Pre1}, we arrive at
\begin{align*}
\E[\tJ_{\x}] = \DA\E[\x\otimes \x] = \DA\Q\,,
\end{align*}
concluding the proof. 
\section{Proof of Proposition~\ref{pro:mineig}}\label{app:mineig}
Define 
\[
B_n  = \sup_{z\in S^{n-1}} \Big| \sum_{i\neq j} \<z_i \vct{u_i}, z_j \vct{u_j}\> \Big|^{1/2}
\]
We write
\[
 \|\mtx{U}\vct{z}\|^2  = \sum_{i=1}^n z_i^2 \|\vct{u_i}\|^2 + \sum_{i\neq j} \<z_i \vct{u_i}, z_j \vct{u_j}\>
\]
Then, clearly for $\vct{z}\in \mathbb{S}^{n-1}$,
\begin{eqnarray}\label{eq:mineig2}
\|\mtx{U}\vct{z}\|^2 \ge \eta_{\min}^2 - B_n^2\,, \quad \|\mtx{U}\vct{z}\|^2 \le \eta_{\max}^2 + B_n^2.
\end{eqnarray}

The following Lemma is similar to \cite{adamczak2011restricted}[Theorem 3.2]. 
\begin{lemma}\label{lem:Bn}
Let $\vct{u_1}, \dotsc, \vct{u_n}$ be independent sub-exponential random vectors with $\psi = \max_{1\le i\le n} \|\vct{u_i}\|_{\psi_1}$.
Let $\theta\in (0,1/4)$, $K,K'\ge 1$ and assume that 
\[
n\log ^2({2}/{\theta}) \le \theta^2 \eta^2
\]
Then setting $\xi = \psi K + K'$, the inequality
\[
B_n^2\le C\xi^2 \theta \eta^2
\]
holds with probability at least 
\[
1- \exp\Big(-cK\sqrt{n} \log \Big(\frac{2}{\theta}\Big) \Big) - \P \Big(\eta_{\max} \ge K'\eta \Big)
\]
\end{lemma}
We apply Lemma~\ref{lem:Bn} by setting $\eta = \eta_{\min}$ and letting $\theta$ be the solution of $n\log^2(2/\theta) = \theta^2\eta^2$. Clearly, $\theta >\sqrt{n}/\eta_{\min}$. 
and therefore $\theta< (\sqrt{n}/\eta_{\min}) \log(2\eta_{\min}/\sqrt{n})$. This gives
\begin{align}\label{BnUB}
B_n^2\le C\xi^2\eta_{\min} \sqrt{n} \log\Big(\frac{2\eta_{\min}}{\sqrt{n}}\Big)\,.
\end{align}
with probability at least
\[
1- C\exp\Big(-cK\sqrt{n} \log \Big(\frac{2\eta_{\min}}{\sqrt{n}}\Big) \Big) - \P \Big(\eta_{\max} \ge K'\eta_{\min} \Big)
\]
Using bound~\eqref{BnUB} in~\eqref{eq:mineig2} we obtain the desired result.
\bigskip

\begin{proof}
We first recall the following Lemma from Adamczak (this applies to vectors with non-trivial covariance). 
\begin{lemma}\label{lem:B}
Let $\vct{u_1}, \dotsc, \vct{u_n}$ be independent sub-exponential vectors with $\psi = \max_{1\le i\le n} \|\vct{u_i}\|_{\psi_1}$. 
For $\theta\in (0,1/4)$ and $K\ge 1$, one has
\begin{align*}
\P(B_n^2\ge \max\{B^2, \eta_{\max} B, 24\theta \eta_{\max}^2\})\le (1+3\log n) \exp\Big(-2K\sqrt{n} \log({2}/{\theta}) \Big)
\end{align*}
and
\begin{align*}
B= C_0 \psi K \sqrt{n} \log({2}/{\theta})
\end{align*}
\end{lemma}

Fix $K\ge 1$ and define
\[K_1 = \frac{K\theta \eta}{\sqrt{n} \log(2/\theta)} \ge K\]
By Lemma~\ref{lem:B}, we have
\begin{align*}
\P(B_n^2\ge \max\{B^2, \eta_{\max} B, 24\theta \eta_{\max}^2\}) &\le (1+3\log n) \exp\Big(-2K_1\sqrt{n} \log(2/\theta) \Big)\\
&\le \exp \Big(-cK \sqrt{n} \log (2/\theta) \Big)
\end{align*}
where 
$$B =C_0\psi K_1\sqrt{n} \log(2/\theta) = C_0 \psi K \theta \eta$$
Thus if $\eta_{\max}\le K'\eta$ for some $K'$, then
\begin{align*}
\max\{B^2, \eta_{\max} B, 24\theta \eta_{\max}^2\}\le C_1 \theta \eta^2 \max\{\psi^2 K^2, \psi KK', K'^2\} \le C_1\theta \eta^2(\psi K+K')^2\,,
\end{align*}
where $C_1$ is an absolute constant. This completes the proof.
\end{proof}

\section{Proof of Lemma~\ref{twonormconc}}\label{app:twonormconc}
Define the function $f(\vct{x})=\twonorm{\mtx{M}\Dv(\phi'(\mtx{W}\vct{x})-\E[\phi'(\W\x)])}$. We first compute the Lipschitz constant of $f$ as follows:
\begin{align*}
\abs{f(\vct{x})-f(\tilde{\vct{x}})}\le& \twonorm{\mtx{M}\Dv \left(\phi'(\mtx{W}\vct{x})-\phi'(\mtx{W}\tilde{\vct{x}})\right)}\\
\le&\opnorm{\mtx{M}\Dv}\twonorm{\phi'(\mtx{W}\vct{x})-\phi'(\mtx{W}\tilde{\vct{x}})}\\
=&\frac{1}{\sigma_{\min}\left(\Gammab\mtx{W}\right)}\twonorm{\phi'(\mtx{W}\vct{x})-\phi'(\mtx{W}\tilde{\vct{x}})}\\
\le&\frac{L}{\sigma_{\min}\left(\Gammab\mtx{W}\right)}\twonorm{\mtx{W}(\vct{x}-\tilde{\vct{x}})}\\
\le&L\frac{\sigma_{\max}\left(\mtx{W}\right)}{\sigma_{\min}\left(\Gammab\mtx{W}\right)}\twonorm{\vct{x}-\tilde{\vct{x}}} = \rho(\W) \|\x-\tx\|.
\end{align*}
Thus $\twonorm{\vct{z}}$ is a Lipschitz function of $\vct{x}$ with Lipschitz constant $\rho(\mtx{W})$. As a result, 
\begin{align}\label{gauss-con}
\P\left( \Big|\twonorm{\vct{z}} - \sqrt{\E[\twonorm{\vct{z}}^2]}\Big| \ge t\rho(\W) \right) \le 2e^{-t^2/2}\,.
\end{align}
We next bound $\E[\twonorm{\z}^2]$. First note that 
\begin{align}
\E[\twonorm{\z}^2] &= \E\Big[\Big\|\M\Dv\left(\phi'(\W\x)-\mub\right)\Big\|^2\Big],\nonumber\\
& = \E\Big[\Big\|\M\Dv\left(\phi'(\W\x)-\mub - \Gammab\W \x\right) +\x\Big\|^2\Big],\nonumber\\
&\ge \E[\|\x\|^2] = d\,,\label{LB}
\end{align}
where in the last inequality, we used the observation that $\E[\<\x,\M\Dv\left(\phi'(\W\x)-\mub - \Gammab\W \x\right)\>] = 0$ which holds based on Equation~\eqref{eq:centertJ}.

For upper bounding $\E[\twonorm{\z}^2]$ note that
\begin{align}
\E[\twonorm{\vct{z}}^2] &= \E\Big[\twonorm{\M\Dv(\phi'(\W\x)-\mub)}^2\Big]\nonumber\\
&= \|\M\Dv\|^2 \E\Big[\twonorm{\phi'(\W\x)-\mub}^2\Big] \nonumber\\
&= \|\M\Dv\|^2 \E\Big[\twonorm{\phi'(\W\x)-\phi'(0)\mtx{1}_k +\phi'(0)\mtx{1}_k - \mub}^2\Big] \nonumber\\
&= \|\M\Dv\|^2 \left\{\E\left[\twonorm{\phi'(\W\x)-\phi'(0)\mtx{1}_k}^2\right] - \twonorm{\mub-\phi'(0)\mtx{1}_k}^2 \right\}\nonumber\\
&\le \|\M\Dv\|^2  \E\left[\twonorm{\phi'(\W\x)-\phi'(0)\mtx{1}_k}^2\right] \nonumber\\
& \le \frac{L^2\sigma_{\max}^2(\W)}{\sigma_{\min}^2(\Gammab\W)} \E[\|\x\|^2] = \rho^2(\W) d\,. \label{UB}
\end{align}
Here, $\mtx{1}_k\in\R^k$ is the all-one vector.
The result follows by putting together bounds~\eqref{LB} and~\eqref{UB} into~\eqref{gauss-con}.

\section{Proof of Lemma~\ref{Jpert}}\label{proof:Jpert}
Note that
\begin{align*}
\left(J(\tv,\tW)-J(\vb,\W)\right)_i=\frac{1}{n}\vct{x}_i\otimes \left(\mtx{D_{\tv}}\phi'(\tW \vct{x}_i)-\mtx{D_{\vb}}\phi'(\W\vct{x}_i)\right)\,.
\end{align*}
Further,
\begin{align}
&\widetilde{v}_{\ell} \phi'(\tw_\ell^T\vct{x}_i)-v_\ell \phi'(\wb_\ell^T\vct{x}_i)\nonumber\\
&= (\widetilde{v}_{\ell} - v_\ell)  \phi'(\tw_\ell^T\vct{x}_i) + {v}_{\ell} \Big(\phi'(\tw_\ell^T\vct{x}_i)- \phi'(\wb_\ell^T\vct{x}_i)\Big)\nonumber\\
&=  (\widetilde{v}_{\ell} - v_\ell)  \phi'(\tw_\ell^T\vct{x}_i) + {v}_{\ell}  \Big(\int_{0}^1 \phi''\left((t\tw_\ell + (1-t)\wb_\ell)^T\vct{x}_i\right) \de t\Big)  (\tw_\ell- \wb_\ell)^T \vct{x}_i\,.\label{eq:aux6}
\end{align}
To simplify our exposition, define 
$$\lambda_{i\ell}:=  \int_{0}^1 \phi''\left((t\tw_\ell + (1-t)\wb_\ell)^T\vct{x}_i\right) \de t\,\quad \,\, \eta_{i\ell}: = \phi'(\tw_{\ell}^T \x_i)\,.$$
Writing~\eqref{eq:aux6} in terms of $\lambda_{i\ell}$ and $\eta_{i\ell}$, we have
\begin{align*}
\widetilde{v}_{\ell} \phi'(\tw_\ell^T\vct{x}_i)-v_\ell \phi'(\wb_\ell^T\vct{x}_i)
=  \eta_{i\ell}(\widetilde{v}_{\ell} - v_\ell) + \lambda_{i\ell}{v}_{\ell} (\tw_\ell- \wb_\ell)^T \vct{x}_i\,.
\end{align*}
Writing the above identity in matrix form, we get
\begin{align}\label{eq:aux7}
\mtx{D_{\tv}}\phi'(\tW \vct{x}_i)-\mtx{D_{\vb}}\phi'(\W\vct{x}_i) = \diag(\vct{\eta_i}) (\tv-\vb) + \diag(\vct{\lambda_i}) \mtx{D}_{\vb} (\tW-\W)\x_i\,.
\end{align}

By triangle inequality, we have
\begin{align}
\label{main46}
\opnorm{J(\tv,\tW)-J(\vb,\W)}\le \opnorm{J(\vb,\tW)-J(\vb,\W)} +\opnorm{J(\tv,\tW)-J(\vb,\tW)}.
\end{align}
We proceed by bounding the two terms above. For the first term note that by applying \eqref{eq:aux7}, we arrive at
\begin{align*}
\opnorm{J(\vb,\tW)-J(\vb,\W)} 
&= \underset{\vct{u}\in\R^n,\twonorm{\vct{u}}=1}{\sup}\text{ }\fronorm{\frac{1}{n}\sum_{i=1}^n u_i\diag(\vct{\lambda_i}) \Dv (\tW-\W)\x_i\x_i^T}.
\end{align*}
Define a matrix $\mtx{M}\in\R^{d^2\times n}$ with columns given by
\begin{align*}
\mtx{M}_{\vct{x}}= \vct{x}\otimes \vct{x}.
\end{align*}
Invoking Corollary~\ref{cor:ldentity}, 
\begin{align}
\opnorm{\mtx{M}} \le C\sqrt{nd}\,.
\end{align}
holds with probability at least $1 - ne^{-b_1\sqrt{n}} - n^{-1} - 2ne^{-b_2 d}$ for some constants $b_1,b_2>0$. 

Now using the above we can also write
\begin{align}
&\opnorm{J(\vb,\tW)-J(\vb,\W)}\nonumber \\
&=\opnorm{\text{blockdiag}\left(\diag(\vct{\lambda_1})\Dv(\tW-\W),\diag(\vct{\lambda_2})\Dv(\tW-\W),\ldots,\diag(\vct{\lambda_n})\Dv(\tW-\W)\right) (\x\otimes\x)} \nonumber\\
&\le \opnorm{\text{blockdiag}\left(\diag(\vct{\lambda_1})\Dv(\tW-\W),\diag(\vct{\lambda_2})\Dv(\tW-\W),\ldots,\diag(\vct{\lambda_n}) \Dv(\tW-\W) \right)} \opnorm{\x\otimes\x} \nonumber\\
&\le \opnorm{\text{blockdiag}\left(\diag(\vct{\lambda_1}),\diag(\vct{\lambda_2}),\ldots,\diag(\vct{\lambda_n})\right)}\opnorm{\Dv(\tW-\W)}\opnorm{\x\otimes\x} \nonumber\\
&\le L\|\vb\|_{\ell_\infty}\opnorm{\tW-\W}\opnorm{\x\otimes\x} \nonumber\\
&\le C L \|\vb\|_{\ell_\infty} \sqrt{nd}\, \opnorm{\tW-\W},\label{term1-Jpert}
\end{align}
where in the penultimate inequality, we use the fact that $|\lambda_{i\ell}| < L$. This concludes our bound on the first term of \eqref{main46}. To bound the second term in \eqref{main46} note that,
\begin{align*}
\opnorm{J(\tv,\tW)-J(\vb,\tW)} 
=&\underset{\vct{u}\in\R^n,\twonorm{\vct{u}}=1}{\sup}\text{ }\fronorm{\frac{1}{n}\sum_{i=1}^n {u}_i  \diag(\vct{\eta_i}) (\tv-\vb)\vct{x}_i\vct{x}_i^T}.
\end{align*}
Using the fact that $|\eta_{i\ell}|<B$, with an analogous argument to the one we used for bounding the first we arrive at
\begin{align}\label{term2-Jpert}
\opnorm{J(\tv,\tW)-J(\vb,\tW)} 
\le C B\sqrt{nd}\, \|\tv-\vb\|_{\ell_\infty}\,.
\end{align}
Combining inequalities~\eqref{term1-Jpert} and \eqref{term2-Jpert}, we have
\begin{align*}
\opnorm{J(\tv,\tW)-J(\vb,\W)} \le C\sqrt{nd}\, \Big(\|\vb\|_{\ell_\infty} \opnorm{\tW-\W} + \|\tv-\vb\|_{\ell_\infty}\Big)\,,
\end{align*}
where $C$ depends on constants $L$ and $B$.

\section{Proof of Lemma~\ref{lem:aux1}}\label{proof:aux1}
Recall that
\begin{align}
\nabla_{\W} \cL(\vb,\W) = \frac{1}{n} \J(\vb,\W) \vct{r}\,,\quad \cL(\vb,\W) = \frac{1}{2n} \twonorm{\vct{r}}^2\,.
\end{align}
Given that $(\vb,\W)\in \Omega$, we have
\begin{align*}
\sigma_{\min}(\J(\vb,\W)) &\ge \sigma_{\min}(\J(\vb^*,\W^*)) - \opnorm{\J(\vb,\W)-\J(\vb^*,\W^*)}\\
&\ge \sigma_{\min}(\J(\vb^*,\W^*)) - C'\sqrt{nd}\, \Big(\|\vb^*\|_{\ell_\infty} \opnorm{\W-\W^*} + \|\vb-\vb^*\|_{\ell_\infty}\Big)\\
&\ge \frac{c}{2} \sigma_{\min}(\W^*) d\,,
\end{align*}
where the first inequality follows from Lemma~\ref{Jpert} and the second one follows readily from definition of set $\Omega$ and Equation~\eqref{sminW*}.

Likewise,
\begin{align}
\sigma_{\max}(\J(\vb,\W)) &\le \sigma_{\max}(\J(\vb^*,\W^*)) + \opnorm{\J(\vb,\W)-\J(\vb^*,\W^*)}\nonumber\\
&\le \sigma_{\max}(\J(\vb^*,\W^*)) + C'\sqrt{nd}\, \Big(\|\vb^*\|_{\ell_\infty} \opnorm{\W-\W^*} + \|\vb-\vb^*\|_{\ell_\infty}\Big)\nonumber\\
&\le C\sigma_{\max}(\W^*) \sqrt{nk} +  \frac{c}{2} \sigma_{\min}(\W^*) d\,, \nonumber\\
&\le  \frac{3}{2}C\sigma_{\max}(\W^*) \sqrt{nk}\,.\label{smaxJ}
\end{align}

Therefore,
\begin{align*}
\fronorm{\nabla_{\W} \cL(\vb,\W)}^2  &= \frac{1}{n^2} \twonorm{\J(\vb,\W) \vct{r}}^2\\
&\ge \frac{2}{n} \sigma_{\min}^2(\J(\vb,\W)) \cL(\vb,\W)\\
&\ge \frac{c^2}{2} \sigma_{\min}^2(\W^*) \frac{d^2}{n}\, \cL(\vb,\W)\,,
\end{align*}
proving Equation~\eqref{eq:aux1}. 

Similarly, we have
\begin{align*}
\fronorm{\nabla_{\W} \cL(\vb,\W)}^2  &= \frac{1}{n^2} \twonorm{\J(\vb,\W) \vct{r}}^2\\
&\le \frac{2}{n} \sigma_{\max}^2(\J(\vb,\W)) \cL(\vb,\W)\\
&\le \frac{9}{2} C^2 \sigma_{\max}^2(\W^*) k\, \cL(\vb,\W)\,,
\end{align*}
proving Equation~\eqref{eq:aux2}. 

To prove the last bound (i.e.~\eqref{eq:aux3}), we recall Equation~\eqref{gradv}
\begin{align*}
\nabla_{\vb} \cL(\vb,\W) = \frac{1}{n} \phi(\W\X) \vct{r}\,.
\end{align*}
We write
\begin{align}\label{eq:aux9}
\opnorm{\nabla_{\vb} \cL(\vb,\W)}_{\ell_\infty} &\le  \opnorm{\sqrt{\frac{2}{n}} \phi(\W\X) \frac{1}{\sqrt{2n}} \vct{r} }_{\ell_\infty}\\
&\le \sqrt{\frac{2}{n}} \opnorm{\phi(\W\X)}_{2,\infty} \sqrt{\cL(\vb,\W)}\,,
\end{align}
where for a matrix $\A$, $\opnorm{\A}_{2,\infty}$ denotes the maximum $\ell_2$ norm of its rows. Hence,
\begin{align}
\opnorm{\phi(\W\X)}_{2,\infty} &= \max_{\ell\in[k]} \Big(\sum_{i=1}^n \phi(\wb_\ell^T \x_i)^2 \Big)^{1/2} \nonumber\\
&\le \max_{\ell\in[k]} \Big(\sum_{i=1}^n \left(|\phi(0)|+ B |\wb_\ell^T \x_i|\right)^2 \Big)^{1/2}\nonumber\\
&\le \max_{\ell\in[k]} \Big(\sum_{i=1}^n \left(2 \phi^2(0)+ 2 B^2 |\wb_\ell^T \x_i|^2\right) \Big)^{1/2}\nonumber\\
&= \max_{\ell\in[k]} \Big(2\phi^2(0)n+ 2B^2\twonorm{\wb_\ell^T \X}^2 \Big)^{1/2}\nonumber\\
&\le \max_{\ell\in[k]} \Big(2\phi^2(0)n+ 2B^2\twonorm{\wb_\ell}^2 \opnorm{\X}^2 \Big)^{1/2}\label{eq:aux10}
\end{align}
Here, we used the assumption that $\phi(z)$ is $B$-Lipschitz.

Next, by the Bai-Yin law~\cite{Guionnet}, we have $\opnorm{\X} \le 4\sqrt{n}$, with probability at least $1-2e^{-2 n}$. 
Continuing with Equation~\eqref{eq:aux10}, we have
\begin{align}
\opnorm{\phi'(\W\X)}_{2,\infty} &\le  \max_{\ell\in[k]} \Big(2\phi^2(0)n+ 32B^2 n\twonorm{\wb_\ell}^2 \Big)^{1/2} \nonumber\\
&\le  \max_{\ell\in[k]} \Big(2\phi^2(0)n+ 32B^2 n\left(\twonorm{\wb^*_\ell}+1\right)^2 \Big)^{1/2}\nonumber\\
&\le \max_{\ell\in[k]} \Big(2\phi^2(0) n+ 32B^2 n\left( 2\twonorm{\wb^*_\ell}^2+2\right) \Big)^{1/2}\nonumber\\
&\le \max_{\ell\in[k]} \Big((2\phi^2(0)+ 64 B^2) n+ 64B^2 n \twonorm{\wb^*_\ell}^2 \Big)^{1/2}\nonumber\\
&\le \left(2\phi^2(0)+ 64 B^2 + 64B^2 w_{\max}^2\right)^{1/2} \sqrt{n} \,,\label{eq:aux11}
\end{align}
where in the second inequality, we use the fact that $(\vb,\W)\in \Omega$ and hence $\fronorm{\W-\W^*} \le 1$ (since $R<v_{\max}$).
Using bound~\eqref{eq:aux11} in~\eqref{eq:aux9}, we obtain
\begin{align}
\opnorm{\nabla_{\vb} \cL(\vb,\W)}_{\ell_\infty}^2 \le \left(4\phi^2(0)+ 128 B^2 + 128B^2 w_{\max}^2\right) \, \cL(\vb,\W)\,.
\end{align} 
concluding the proof of Equation~\eqref{eq:aux3}. 

\section{Proof of Lemma~\ref{lem:aux2}}\label{proof:aux2}
We have
\begin{align}
\label{secineq}
\cL(\vct{v},\mtx{W})=&\frac{1}{2n}\sum_{i=1}^n\left(\vct{v}^T\phi\left(\mtx{W}\vct{x}_i\right)-\vct{v}^{*T}\phi\left(\mtx{W}^*\vct{x}_i\right)\right)^2,\nonumber\\
\le& \underbrace{\frac{1}{n} \twonorm{\vct{v}}^2 \sum_{i=1}^n \twonorm{\phi(\mtx{W}\vct{x}_i)-\phi(\mtx{W}^*\vct{x}_i)}^2}_{T_1} +
\underbrace{ \frac{1}{n}\sum_{i=1}^n \Big( (\vct{v}-\vct{v}^*)^T\phi(\mtx{W}^* \vct{x_i})\Big)^2}_{T_2}.\nonumber\\
\end{align}

By standard concentration of sample covariance matrices we conclude that for any $\delta>0$, there exist constants $c, b_0>0$, such that for $n\ge \frac{c^2}{\delta^2} d$, we have
\begin{align}\label{eq:aux12}
\opnorm{\frac{1}{n}\sum_{i=1}^n \x_i\x_i^T - \mtx{I}} \le \delta\,,
\end{align}
with probability at least $1- 2 e^{-b_0 d}$. Using this bound, with $\delta =1$, along with the $1$-Lipschitz property of the activation $\phi(z)$, we arrive at 
\begin{align}
T_1 &\le \frac{\|\vct{v}\|^2}{2n}\sum_{i=1}^n\twonorm{\phi(\mtx{W}\vct{x}_i)-\phi(\mtx{W}^*\vct{x}_i)}^2,\nonumber\\
&{\le}\frac{B\|\vct{v}\|^2}{2n}\sum_{i=1}^n\twonorm{\left(\mtx{W}-\mtx{W}^*\right)\vct{x}_i}^2,\nonumber\\
&{=}\frac{B\|\vct{v}\|^2}{2}\langle \left(\mtx{W}-\mtx{W}^*\right)^T\left(\mtx{W}-\mtx{W}^*\right) , \frac{1}{n}\sum_{i=1}^n\vct{x}_i\vct{x}_i^T\rangle,\nonumber\\
&\le B \|\vct{v}\|^2 \fronorm{\mtx{W}-\mtx{W}^*}^2.
\end{align}
To bound the second term note that we have
\begin{align*}
T_2&\le \frac{1}{n}\sum_{i=1}^n \Big( (\vb-\vb^*)^T\phi(\mtx{W}^* \vct{x_i})\Big)^2,\\
&\le\frac{1}{n}{\opnorm{\vb-\vb^*}_{\ell_\infty} ^2 } \sum_{i=1}^n \opnorm{\phi(\W^* \x_i)}_{\ell_1} ^2,\\
&\le \frac{1}{n}{\opnorm{\vb-\vb^*}_{\ell_\infty} ^2 } \sum_{i=1}^n \left( \sum_{\ell=1}^k. \phi(\wb_\ell ^{*T} \x_i) \right)^2
\end{align*}
Note that for any $z\in\R$, we have $|\phi(z)| < |\phi(0)|+ B |z|$. Continuing with the above inequality, we write
\begin{align*}
T_2&\le  \frac{1}{n}{\opnorm{\vb-\vb^*}_{\ell_\infty} ^2 } \sum_{i=1}^n \left( \sum_{\ell=1}^k  \left(|\phi(0)|+ B |\wb_\ell ^{*T} \x_i|\right) \right)^2,\\
&\le \frac{k}{n}{\opnorm{\vb-\vb^*}_{\ell_\infty} ^2 } \sum_{i=1}^n  \sum_{\ell=1}^k\left(|\phi(0)|+ B|\wb_\ell ^{*T} \x_i| \right)^2,\\
&\le \frac{k}{n}{\opnorm{\vb-\vb^*}_{\ell_\infty} ^2 } \sum_{i=1}^n  \left(2|\phi(0)|k+ 2B \sum_{\ell=1}^k (\wb_\ell ^{*T} \x_i)^2 \right),\\
&= \frac{2k}{n}{\opnorm{\vb-\vb^*}_{\ell_\infty} ^2 } \left(\phi^2(0) kn+ B^2\fronorm{\W^*\X}^2 \right),\\
&\le \frac{2k}{n}{\opnorm{\vb-\vb^*}_{\ell_\infty} ^2 } \left(\phi^2(0) kn+ B^2 \fronorm{\W^*}^2\opnorm{\X}^2 \right),\\
&\le {2k^2}{\opnorm{\vb-\vb^*}_{\ell_\infty} ^2 } \left(\phi^2(0)+ 2B^2 w_{\max}^2 \right)\,,
\end{align*}
where in the last step, we used~\eqref{eq:aux12} and the fact that $\fronorm{\W^*}^2 \le k w_{\max}^2$. Combining the bounds on $T_1$ and $T_2$, completes the proof.

\section{Proof of Lemma~\ref{lem:aux3}}\label{proof:aux3}
Note that the spectral norm of a matrix is bounded above by the sum of the spectral norms of the diagonal blocks of that matrix. Hence,
\begin{align}
\opnorm{\nabla^2\cL(\vb,\W)} \le \opnorm{\nabla_{\W}^2\cL(\vb,\W)} + \opnorm{\nabla_{\vb}^2\cL(\vb,\W)}\,.
\end{align}
We proceed by bounding each of these two terms. Using Identity~\eqref{partHess2}, we have
\begin{multline}
\nabla^2_{\W} \cL(\vb,\W) =\nonumber\\
 \frac{1}{n} \J \J^T+ \frac{1}{n} \text{blockdiag} \left( \sum_{i=1}^n v_1\phi''(\wb_1^T\x_i) r_i \x_i\x_i^T, \sum_{i=1}^n v_2\phi''(\wb_2^T\x_i) r_i \x_i\x_i^T, \dotsc, \sum_{i=1}^n v_k\phi''(\wb_k^T\x_i) r_i \x_i\x_i^T\right).
\end{multline}
Since $(\vb,\W)\in \Omega$, we have 
$$\opnorm{\vb}_{\ell_\infty} \le \opnorm{\vb^*}_{\ell_\infty} + \opnorm{\vb-\vb^*}_{\ell_\infty} \le R+v_{\max} <2v_{\max}\,.$$
Furthermore using \eqref{eq:aux12} together with $|\phi''|<L$ we conclude that
\begin{align}
\opnorm{\nabla^2_{\W} \cL(\vb,\W) - \frac{1}{n}\J\J^T} &\le 2Lv_{\max} \opnorm{\frac{1}{n}\sum_{i=1}^n r_i \x_i \x_i^T},\nonumber\\
&\le 2L v_{\max} \opnorm{\vct{r}}_{\ell_\infty} \opnorm{\frac{1}{n}\sum_{i=1}^n \x_i\x_i^T}, \nonumber\\
&\le 4L v_{\max} \opnorm{\vct{r}}_{\ell_\infty}\,,\label{eq:H1-1}
\end{align}
holds with probability at least $1- 2e^{-b_0d}$. Moreover, using the B-Lipschitz property of the activation, we have 
\begin{align}
|r_i| &= |\vb^T\phi(\W\x_i) - \vb^T \phi(\W^*\x_i)|\nonumber\\
&\le \twonorm{\vb} \twonorm{\phi(\W\x_i) - \vb^T \phi(\W^*\x_i)}\nonumber\\
&\le Bv_{\max}\sqrt{k} \twonorm{(\W-\W^*)\x_i}\nonumber\\
&\le 2B v_{\max}\sqrt{kd} \opnorm{\W-\W^*}\,,\label{eq:rB}
\end{align}
where in the last step we used the fact that $\twonorm{\x_i}^2\le d$, which holds with probability at least $1- e^{-d/8}$ by standard concentration bound for $\chi^2$- random variables. Using~\eqref{eq:rB} in~\eqref{eq:H1-1}, we obtain
\begin{align}
\opnorm{\nabla^2_{\W} \cL(\vb,\W) - \frac{1}{n}\J\J^T} 
&\le 8  BL v_{\max}^2 \sqrt{kd} \opnorm{\W-\W^*},\nonumber\\
&\overset{(a)}{\le} 8  BL v_{\max}BL \sqrt{kd} R,\nonumber\\
&\overset{(b)}{\le} 8  BL v_{\max}^2 BL k \,,  \label{eq:H1-2}
\end{align}
where (a) follows from the definition of the set $\Omega$ and (b) follows from $R\le v_{\max}$ and $k\ge d$. Note that while we can plug in for $R$ to obtain a tighter bound, we use a looser bound on $R$ as this term will be dominated by the other term (i.e.~$\|\J\J^T/n\|$) and thus a more accurate bound is unnecessary.

Restating the bound on $\sigma_{\max}(\J(\vb,\W))$ for $(\vb,\W)\in \Omega$, given by~\eqref{smaxJ}, we have
\begin{align}
\opnorm{\frac{1}{n}\J\J^T} 
\le \frac{9}{4} {C^2} \sigma_{\max}^2(\W^*) k  \,.\label{eq:H1-3}
\end{align}
 
Combining~\eqref{eq:H1-2} and~\eqref{eq:H1-3}, we arrive at
\begin{align}
\opnorm{\nabla^2_{\W} \cL(\vb,\W)} \le \left({3C^2}\sigma_{\max}^2(\W^*) + 8 v_{\max}^2 BL \right) k\,.\label{eq:H1-4}
\end{align} 

We next bound $\opnorm{\nabla_{\vb}^2\cL(\vb,\W)}$. Starting with Equation~\eqref{Hessianv}, we have
\begin{align}
\opnorm{\nabla_{\vb}^2\cL(\vb,\W)} &= \frac{1}{n} \opnorm{\phi(\W\X)\phi(\W\X)^T}\nonumber\\
&\le \frac{1}{n}\opnorm{\sum_{i=1}^n \phi(\W\x_i)\phi(\W\x_i)^T}\nonumber\\
&\le \frac{1}{n}\sum_{i=1}^n \opnorm{\phi(\W\x_i)\phi(\W\x_i)^T}\nonumber\\
&=  \frac{1}{n}\sum_{i=1}^n \twonorm{\phi(\W\x_i)}^2\nonumber\\
&\stackrel{(a)}{\le} \frac{1}{n} \sum_{i=1}^n  \sum_{\ell=1}^k\left(|\phi(0)|+ B |\wb_\ell ^{*T} \x_i| \right)^2\nonumber\\
&= \frac{1}{n} \sum_{i=1}^n  \left(2 \phi(0)^2 k+ 2B^2 \sum_{\ell=1}^k (\wb_\ell ^{*T} \x_i)^2 \right)\nonumber\\
&= \frac{2}{n} \left(\phi^2(0) kn+ B^2\fronorm{\W^*\X}^2 \right)\nonumber\\
&\le \frac{2}{n}\left(\phi^2(0) kn+ B^2\fronorm{\W^*}^2\opnorm{\X}^2 \right)\nonumber\\
&\stackrel{(b)}{\le} {2k}\left(\phi^2(0)+ 2B^2 w_{\max}^2 \right)\,.\label{eq:H2}
\end{align}
Here, $(a)$ holds because $\phi$ is $B$-Lipschitz; $(b)$ holds because $\fronorm{\W}^2\le k w_{\max}^2$ and $\opnorm{\X}\le \sqrt{2n}$, with probability at least $1-2e^{-2 n}$.

Using bounds~\eqref{eq:H1-4} and~\eqref{eq:H2}, we obtain
\begin{align}
\opnorm{\nabla_{\W}^2\cL(\vb,\W)} + \opnorm{\nabla_{\vb}^2\cL(\vb,\W)} \le \left( {3C^2}\sigma_{\max}^2(\W^*) + 8 v_{\max}^2 BL+ 4B^2w_{\max}^2+2\phi^2(0)\right) k\,.
\end{align}
\end{document}

%% file: Num2.tex
\section{Numerical experiments}
\begin{figure}[!t]
    \centering
   
        \begin{subfigure}[b]{.43\textwidth}
        \includegraphics[width=\linewidth]{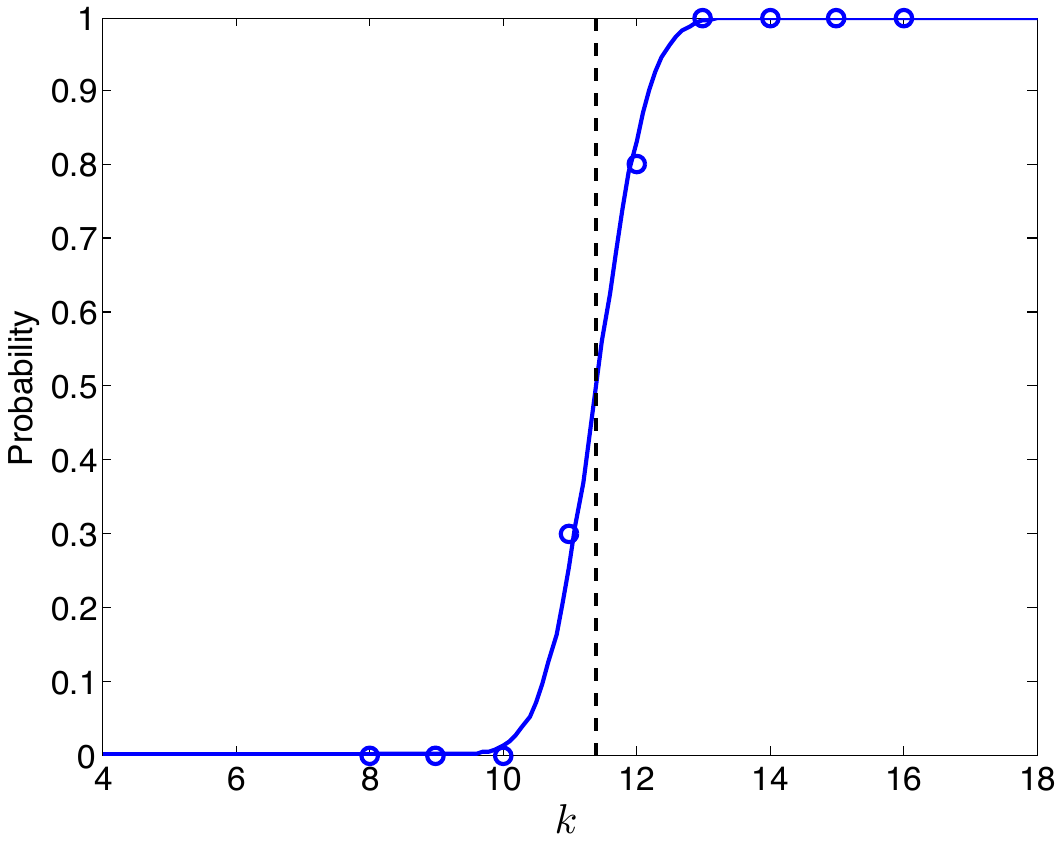}
           \hspace{1cm}
        \caption{$d =10$}\label{SoftPlus-varyK}
         \end{subfigure}    
         \hspace{1cm} 
   \begin{subfigure}[b]{.43\textwidth}
        \includegraphics[width=\linewidth]{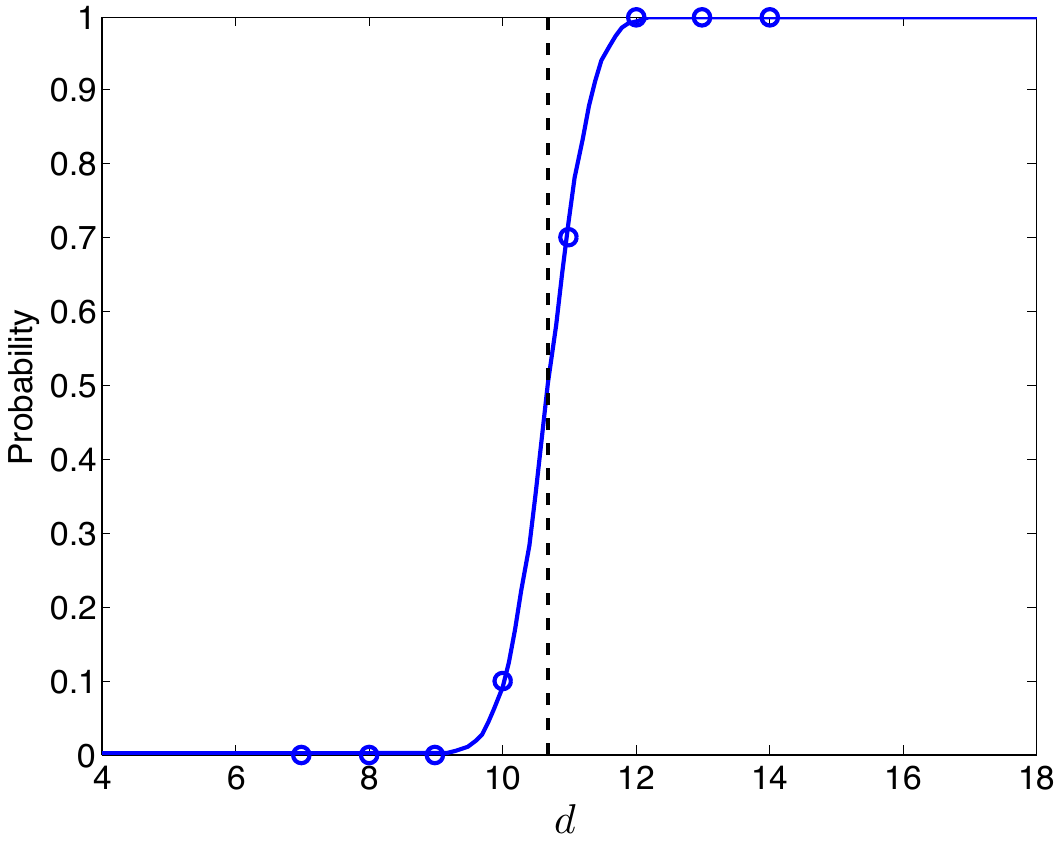}
        \hspace{1cm}
        \caption{$k=10$}\label{SoftPlus-varyD}
         \end{subfigure}

          \caption{Empirical probability of having no spurious local minimizers for a one-hidden layer neural network with softplus activation. Dotted points depict the empirical probabilities, the solid line is obtained by fitting a logistic model to the dotted points, and the dashed line depicts the point where the probability is $1/2$.  The number of samples are equal to $n=100$. In (a) we fix $d=10$ and vary $k$ and in (b) we fix $k=10$ and vary $d$. In both experiments as the number of parameters $(kd)$ increases beyond the sample size $n$, the probability of having no spurious local minimizer increases.}
               \label{fig:softplus}

   \end{figure} 
  {
In this section, we numerically investigate whether gradient descent finds the global optimum for various configurations of $n,k,d$. In our first experiment, the data is generated from a planted Gaussian model of the form
\begin{align*}
y_i=\langle\vct{v}^*,\phi(\mtx{W}^*\vct{x}_i)\rangle\quad\text{with}\quad\vct{x}_i \sim \mathcal{N}(\vct{0},\mtx{I}_d),
\end{align*}
and $\vct{v}^*$ the all-one vector.  We consider two activations, namely the softplus $\phi(z)=\frac{1}{10}\log\left(1+e^{10z}\right)$ and the quadratic activation $\phi(z) =z^2$.}

{In Figure \ref{fig:softplus}, we show the results for the softplus activation with $n=100$. In Figure~\ref{SoftPlus-varyK}, we fix the input dimension at $d = 10$ and vary the number of hidden nodes $k$. For each value of ($n$, $k$, $d$) we carryout $10$ experiments. In each experiment, we generate $\mtx{W}^*$ at random with i.i.d.~$\mathcal{N}(0,1)/\sqrt{d}$ entries. For each value of  $\mtx{W}^*$ we carry out $10$ trials where in each trial we run gradient descent starting from a random initialization pair $(\vct{v}_0,\mtx{W}_0)$ generated at random with $\vct{v}_0$ consisting of i.i.d.~Rademacher $\pm 1$ entries and $\mtx{W}_0$ with i.i.d.~$\mathcal{N}(0,1)/\sqrt{d}$ entries. If a global minimum was obtained for every single initialization, we declare that the loss function has no spurious local minima for the corresponding $\mtx{W}^*$. In Figure \ref{fig:softplus} we plot the empirical probability that the loss function has no spurious local optima for the softplus activation. Dots correspond to the simulation results and the solid curve is obtained by fitting a logistic model to the results. In Figure \ref{SoftPlus-varyK} we focus on different number of hidden units, with the input dimension fixed at $d=10$. The dashed line indicates the value of $k$ ($k= 11.40$) at which the fitted logistic model crosses the probability $1/2$. As $k$ grows we enter a region that the gradient descent frequently converges to global minimizers, and for $k\ge 13$ every single local minimizer found by gradient descent is a global minimizer. Note that in this regime the number of parameters $(kd)$ exceeds the sample size $n$. This suggests that a modest amount of over-parameterization is sufficient for the no spurious local optima conclusion to hold and the global convergence to occur. In Figure~\ref{SoftPlus-varyD}, we fix the number of hidden nodes at $k = 10$ and vary the input dimension $d$. We observe a similar trend: As $d$ grows we enter a region where every single local minimizer found by gradient descent is a global minimizer.  Here, the dashed line is located at $d = 10.70$, corresponding to a $1/2$ probability of converging to a global minima. For $d \ge12 $, gradient descent always finds a global minimizer.} 

{We also repeat a similar experiment with quadratic activations and report the results in Figure~\ref{fig:quad}. In Figure \ref{Quad activation-varyK}, we fix the input dimension at $d = 20$ and the number of samples at $n=100$ and vary $k$. The dashed line here is located at $k= 5.5$. In Figure \ref{Quad activation-varyD}, we fix the number of hidden units at $k = 5$ and vary $d$. The dashed line in this experiment is located at $d= 25.40$. These experiments further corroborate our theory that over-parameterization helps gradient descent find a global minimizer with quadratic activations.}

   \begin{figure}[!t]
    \centering
   
        \begin{subfigure}[b]{.43\textwidth}
        \includegraphics[width=\linewidth]{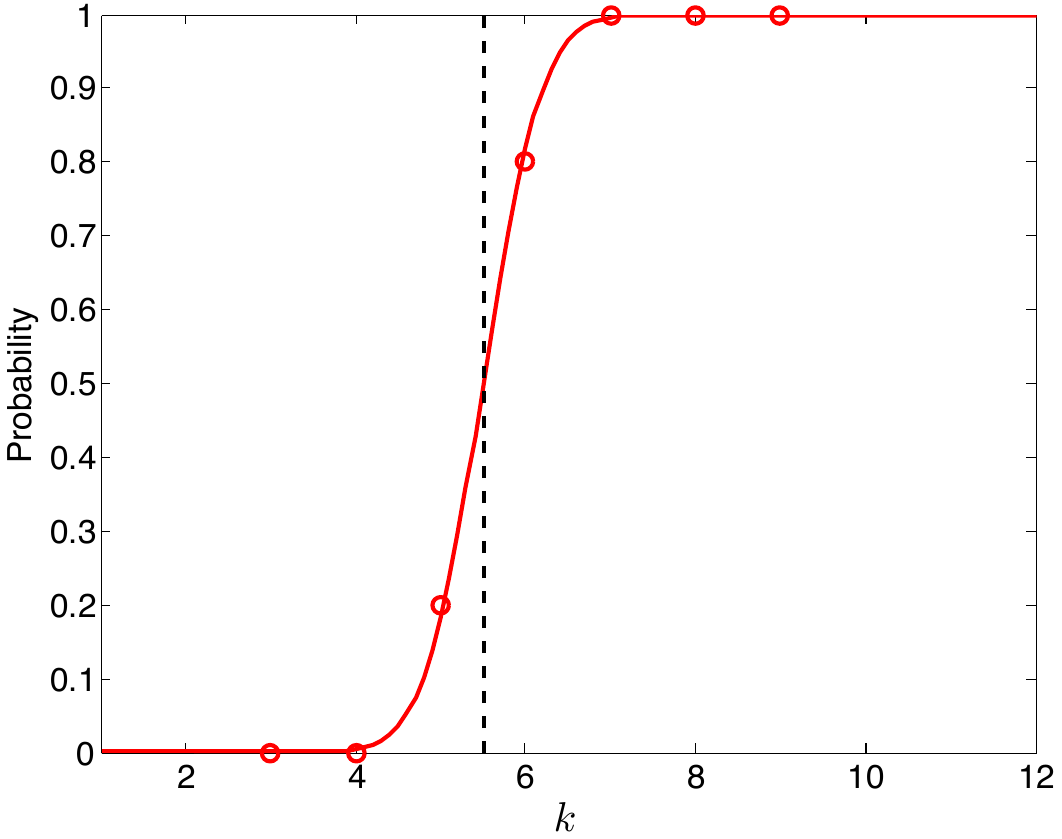}
           \hspace{1cm}
        \caption{$d = 20$}\label{Quad activation-varyK}
         \end{subfigure}     
         \hspace{1cm}
   \begin{subfigure}[b]{.43\textwidth}
        \includegraphics[width=\linewidth]{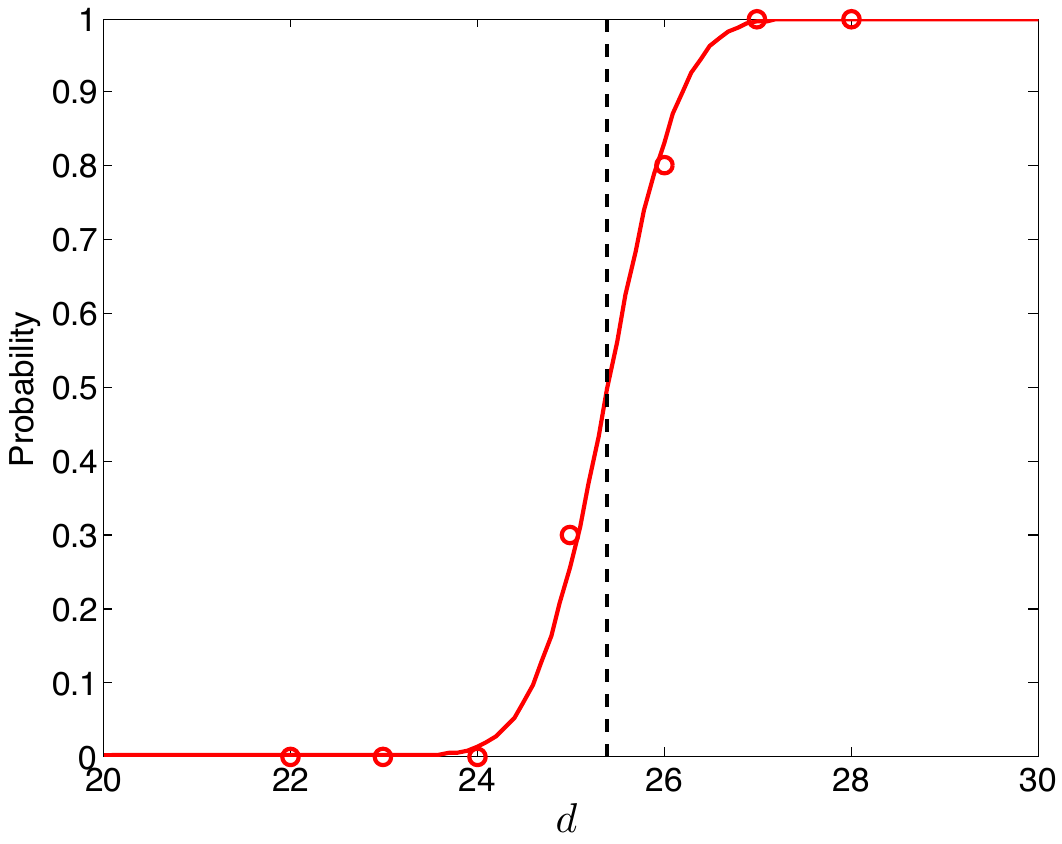}
        \hspace{1cm}
        \caption{$k=5$}\label{Quad activation-varyD}
         \end{subfigure}     
         
     \caption{Empirical probability of having no spurious local minimizers for a one-hidden layer neural network with quadratic activations. Dotted points depict the empirical probabilities, the solid line is obtained by fitting a logistic model to the dotted points, and the dashed line depicts the point where the probability is $1/2$. The number of samples is $n=100$. In (a) we fix $d=10$ and vary $k$ and in (b) we fix $k=10$ and vary $d$. In both experiments as the number of parameters $(kd)$ increases beyond the sample size $n$, the probability of having no spurious local minimizer increases.}
               \label{fig:quad}  
   \end{figure}
{In the next set of experiments, we generate the labels and features i.i.d.~at random
\begin{align*}
\vct{x}_i \sim \mathcal{N}(\vct{0},\mtx{I}_d)\quad\text{and}\quad y_i\sim \mathcal{N}( 0, 1).
\end{align*}
We fit a softplus one-hidden layer network of the form $\vct{x}\rightarrow\vct{1}^T\phi\left(\mtx{W}\vct{x}\right)$ to this data with varying number of hidden units $k$. In Figure \ref{fig:softplus-quad-rand-data} we plot the square root of the objective function $\mathcal{L}$ as a function of iterations. Interestingly, we observe that even for randomly labeled data when $kd>1.5n$, gradient descent is able to find a global minimizer. These experiments show that simple gradient descent is almost always able to find a global minimizer in the sufficiently over-parametrized regime, regardless of whether the data is generated from a planted model.}
\begin{figure}[!t]
    \centering
        \begin{subfigure}[b]{.49\textwidth}
        \includegraphics[width=\linewidth]{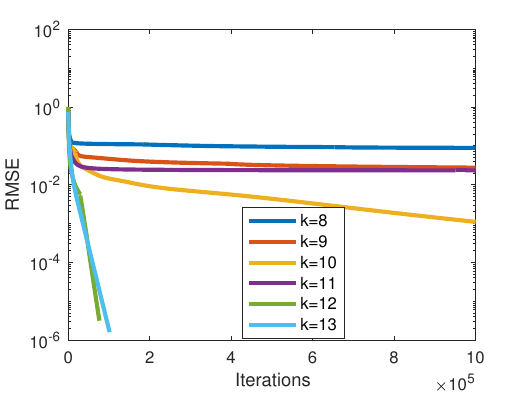}
           \hspace{1cm}
        \caption{Softplus with random data $(d,n)=(10,100)$.}
         \end{subfigure}     
   \begin{subfigure}[b]{.49\textwidth}
        \includegraphics[width=\linewidth]{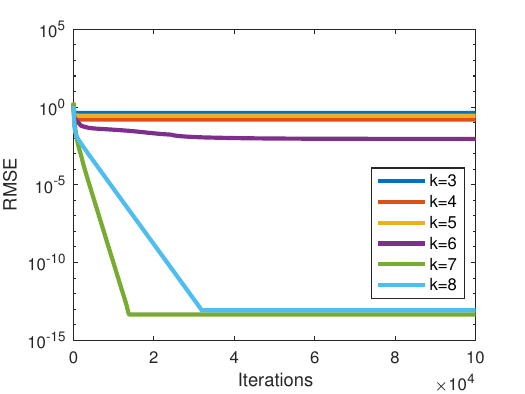}
        \hspace{1cm}
        \caption{Quadratic with random data $(d,n)=(20,100)$.}
         \end{subfigure}     
          \caption{Plot of RMSE=$\sqrt{\mathcal{L}(\vct{v},\mtx{W})}$ vs. the number of gradient descent iterations. As the number of hidden units $k$ increases, gradient descent converges to a solution with zero RMSE. }
               \label{fig:softplus-quad-rand-data}          
   \end{figure}